\documentclass{article}
\usepackage{graphicx} 
\usepackage[letterpaper,top=2cm,bottom=2cm,left=3cm,right=3cm,marginparwidth=1.75cm]{geometry}
\usepackage[noadjust]{cite}
\usepackage{amsmath, amsfonts, amsthm,amssymb}
\usepackage{graphicx,soul}
\usepackage{bm}
\usepackage{authblk}
\usepackage{float}
\usepackage[colorlinks=true, allcolors=blue]{hyperref}
\newtheorem{theorem}{Theorem}[section]
\newtheorem{lemma}{Lemma}[section]
\newtheorem{corollary}{Corollary}[section]
\newtheorem{assumption}{Assumption}[section]

\usepackage{amsfonts}
\usepackage{color,xcolor}
\theoremstyle{definition}
\newtheorem{definition}[theorem]{Definition}

\usepackage{multirow}
\usepackage{soul}
\theoremstyle{remark}
\newtheorem{remark}[theorem]{Remark}

\usepackage{pdfpages}
\usepackage{subfigure}
\usepackage{ulem}
\usepackage[thicklines]{cancel}
\usepackage{dsfont}
\usepackage{booktabs} 
\usepackage{multirow} 

\usepackage[utf8]{inputenc}
\newcommand{\bM}{\mathbf{M}}
\newcommand{\cB}{\mathcal{B}}
\newcommand{\bV}{\mathbf{V}}

\newcommand{\bP}{\mathbf{P}}
\newcommand{\bTheta}{\mathbf{\Theta}}

\newcommand{\bSigma}{\boldsymbol{\Sigma}}

\newcommand{\cC}{\mathcal{C}}

\newcommand{\cG}{\mathcal{G}}
\newcommand{\bS}{\mathbf{S}}

\newcommand{\bx}{\mathbf{x}}

\newcommand{\cP}{\mathcal{P}}
\newcommand{\cK}{\mathcal{K}}
\newcommand{\cN}{\mathcal{N}}

\newcommand{\cE}{\mathcal{E}}

\newcommand{\bZ}{\mathbf{Z}}
\newcommand{\bW}{\mathbf{W}}
\newcommand{\bs}{\mathbf{\Sigma}}

\usepackage{marvosym}
\newcommand{\mail}{\textsuperscript{\Letter}}

\title{On Interpolation Formulas  Describing Neural Network Generalization}
\author[1]{Jin Guo}
\author[1]{Roy Y. He}
\author[2,\mail]{Jean-Michel Morel}
\affil[1]{Department of Mathematics, City University of Hong Kong, Hong Kong}
\affil[2,\mail]{Division of Industrial Data Science, Lingnan University, Hong Kong }

\date{}

\begin{document}
\maketitle
\begin{abstract}
In 2020 Domingos introduced an interpolation formula valid for ‘‘every model trained by gradient descent". He concluded that such models behave approximately as kernel machines. In this work, we extend the Domingos formula to stochastic training. We introduce a stochastic gradient kernel that extends the deterministic version via a continuous-time diffusion approximation. We prove stochastic Domingos theorems and show that the expected network output admits a kernel-machine representation with optimizer-specific weighting. It reveals that training samples contribute through loss-dependent weights and gradient alignment along the training trajectory. We then link the generalization error to the null space of the integral operator induced by the stochastic gradient kernel. The same path-kernel viewpoint provides a unified interpretation of diffusion models and GANs: diffusion induces stage-wise, noise-localized corrections, whereas GANs induce distribution-guided corrections shaped by discriminator geometry. We visualize the evolution of implicit kernels during optimization and quantify out-of-distribution behaviors through a series of numerical experiments. Our results support a feature-space memory view of learning: training stores data-dependent information in an evolving tangent feature geometry, and predictions at test time arise from kernel-weighted retrieval and aggregation of these stored features, with generalization governed by alignment between test points and the learned feature memory.
\end{abstract}


\section{Introduction}
Despite the remarkable empirical success of deep learning in diverse domains \cite{lecun2015deep,schmidhuber2015deep},  we still lack a unified mathematical framework  that explains its generalization capacity.  Developing such a theory requires careful empirical observation of the deep learning process, with the goal of identifying variables governed by correlated dynamical laws \cite{chaudhari2018stochastic,orvieto2019continuous}. A substantial body of research has examined intriguing phenomena in deep learning within over-parameterized regimes \cite{arora2019exact,aiudi2025local}.

Deeper mechanistic insights have emerged from recent studies linking infinite-width neural networks to kernel methods \cite{jacot2018neural}, offering new perspectives on training dynamics, generalization, and expressivity. Central to this line of work is the neural tangent kernel (NTK) \cite{jacot2018neural,novak2019neural}, which provides a theoretical framework for characterizing the behavior of deep neural networks in the infinite-width limit under specific initialization scales and learning rates. In this regime, network parameters evolve only slightly in parameter space, allowing the learning process to be interpreted as linear dynamics in the tangent space of the function manifold at initialization. More precisely, Lee et al. \cite{lee2019wide} demonstrate that the gradient flow in infinitely wide networks induces linear parameter evolution, so that the training dynamics can be described by kernel regression using a fixed kernel derived from the tangent space at initialization. These kernel models also reproduce key behaviors of over-parameterized networks, including interpolation of training data \cite{zhang2016understanding,belkin2018understand}.

However, the equivalence between neural networks and kernel methods is no longer valid for finite-width networks. Empirically, state-of-the-art neural networks consistently outperform their random kernel counterparts by substantial margins~\cite{fort2020deep}, indicating that practical deep learning systems explore function-space regions far from their initialization, with the tangent space undergoing significant reconfiguration during training. This reconfiguration implies that the effective kernel governing learning evolves dynamically and becomes intrinsically data-dependent. However, the dynamics of this function-space exploration, including the trajectory and scale of evolution, the magnitude of tangent-space distortion, and the precise role of the data in shaping the learned kernel, remain poorly characterized.

Beyond the kernel perspective, other theoretical approaches offer partial explanations. Early work on the representational power of neural networks through variants of Universal Approximation Theorems (UATs)~\cite{cybenko1989approximation,hornik1989multilayer} showed that sufficiently wide architectures can approximate any continuous function, and subsequent studies extended these results to broader network classes~\cite{cohen2016expressive,hanin2017approximating}. However, these analyses address only what neural networks \textit{can} represent, not why they generalize from finite training data, a question central to their statistical performance.

This gap has made generalization in neural networks an important topic of investigation, with a longstanding goal of clarifying how architecture, optimization, and task structure jointly determine test performance. Classical capacity measures in statistical learning theory, such as the Vapnik–Chervonenkis (VC) dimension~\cite{vapnik1991necessary} and Rademacher complexity~\cite{bartlett2002rademacher}, successfully characterize many traditional algorithms, yet do not adequately explain the generalization behavior observed in modern deep networks. These models are often heavily over-parameterized, with far more parameters than training samples, a setting in which classical theory would suggest severe overfitting~\cite{zhang2016understanding}. In contrast, a substantial body of empirical evidence~\cite{zhang2016understanding,belkin2019does,belkin2021fit,44f6fc04-5e81-3369-9330-afe10d4a9993} shows that even interpolating, over-parameterized networks can generalize well to unseen data. This phenomenon, commonly termed \textit{benign overfitting}~\cite{bartlett2020benign,frei2022benign,mallinar2022benign}, has motivated a wide range of theoretical efforts aimed at identifying when and why over-parameterized models can still generalize effectively~\cite{montanari2022interpolation,jacot2018neural,oymak2019overparameterized}.
Among these efforts, several works~\cite{schmidt2020nonparametric,suzuki2021deep,44f6fc04-5e81-3369-9330-afe10d4a9993} explain the generalization of neural networks from the perspective of statistical decision theory using carefully designed nonparametric regression frameworks. These results provide sharp statistical guaranties in carefully specified regimes, but typically abstract away the concrete optimization dynamics of modern training, motivating complementary, dynamics-based explanations that make explicit how training trajectories shape generalization in over-parameterized networks.

One such dynamics-based perspective is offered by Domingos \cite{domingos2020every}, who argued that gradient-based training models, including neural networks, induce path kernels that depend on training trajectories. Unlike conventional kernel methods, the coefficients of the path kernels are determined by both the training and the input data. Path kernels effectively store gradient information at the training points and match them to the query. While a well-established approach for interpreting a network's outputs involves searching for training instances close to the query in Euclidean or other canonical spaces \cite{ribeiro2016should}, path kernels clarify the precise space in which these comparisons should occur and how that space relates to the model's predictions. Building on this view, Courtois et al.~\cite{courtois2023can} examine the extrapolation capabilities of trained networks, finding that in simple cases networks interpolate training data and extrapolate perfectly, but for more complex tasks, this ability breaks down and networks fail to extrapolate to samples that differ significantly from the training set. This perspective challenges the view that deep learning autonomously discovers novel representations \cite{bengio2013representation}; rather, gradient descent iteratively selects features from a space fixed by the network architecture to construct an effective kernel.

Although Domingos' theorem was originally developed for gradient descent, contemporary neural networks are almost never trained with full-batch deterministic updates. In practice, training is driven by stochastic optimizers such as SGD~\cite{d8d62392-9a37-31e7-ad3b-37a6f6ee8ef6}, SGDM~\cite{sutskever2013importance}, RMSprop~\cite{tieleman2012rmsprop}, and Adam~\cite{kingma2014adam}, where random mini-batches, momentum, and adaptive rescaling fundamentally reshape the parameter trajectories. This gap motivates our work: what are neural networks actually learning under stochastic training procedures, and does the kernel-based interpretation suggested by Domingos' theorem extend to networks with stochastic optimizers used in practice?

To answer this question, we establish a stochastic extension of Domingos' theory for data-driven models trained with random mini-batches and provide a unified perspective on test-time predictions across various optimizers. The central object in our framework is the \textit{stochastic gradient kernel} (Definition~\ref{def:stoch_kernel}), which replaces the deterministic gradient kernel of Domingos' theory with a stochastic counterpart evaluated along a continuous-time diffusion process that weakly approximates the discrete mini-batch iterates. This weak approximation, which controls expectations of smooth observables uniformly over all iterates, allows us to transfer the path-kernel aggregation structure from gradient flow to stochastic optimization while incurring only an $O(\eta)$ remainder, where $\eta>0$ is the learning rate. Using this kernel, we show that, across SGD, SGDM, RMSprop, and Adam, the expected network output associated with a test input admits a Domingos-type path-kernel representation with optimizer-specific weighting mechanisms. This representation formalizes the network's interpolation  mechanism in a gradient feature space: the prediction at a test input is obtained by aggregating training residual signals using similarities defined by path kernels between test and training tangent features. In this sense, the network implements a form of memory in which information from training samples is stored in the parameters and later retrieved through learned similarities.

Our results on the optimizer-specific structure of this memory are revealing. Momentum introduces an exponential temporal weighting that carries past gradient information forward with a tunable decay, while adaptive methods such as RMSprop and Adam reshape the interpolation geometry through a \textit{weighted stochastic gradient kernel} (Definition~\ref{def:weighted_kernel}), in which a time-varying preconditioner encoding running second-moment statistics reweights gradient feature directions. Building on this representation, we further develop an adaptive RKHS framework to analyze generalization through the evolution of the induced tangent kernel and its null-space structure.
Beyond supervised learning, we use the same path-kernel perspective to extend this interpretation to generative modeling, showing how diffusion models and GANs can be understood within a unified correction and aggregation mechanism despite their different training objectives.

Our contributions and results are summarized as follows:
\begin{itemize}
    \item We introduce a stochastic gradient kernel, defined by evaluating gradient inner products along a continuous-time diffusion process that weakly approximates the discrete mini-batch iterates. Using this kernel, we establish stochastic, non-asymptotic extensions of Domingos' theorem for SGD, SGDM, RMSprop, and Adam, showing that the expected network output admits an equivalent kernel machine representation. For adaptive optimizers, we further introduce a weighted stochastic gradient kernel in which a time-varying preconditioner modulates the feature-space geometry. Unlike the original Domingos formulation, our results hold at finite learning rate, eliminating the restrictive requirement of a vanishing step size.
    \item We introduce a trajectory-indexed adaptive tangent kernel and a dynamic reproducing kernel Hilbert space (RKHS), whose eigendecomposition separates learnable modes from null modes. We then show that, after convergence, the terminal residual concentrates near the null space, establishing that the generalization error is primarily driven by target components that lie outside the learnable RKHS modes.
    \item We leverage the Domingos path-kernel representation to place diffusion models and GANs within a unified kernel-weighted correction framework. Diffusion produces stage-wise corrections that are localized by noise level through the time or noise embedding, while GANs produce distribution-guided corrections driven by discriminator-implied density-ratio geometry. This contrast offers a geometric explanation for their different refinement and stability behaviors, with diffusion showing progressive denoising and GANs exhibiting more globally coupled dynamics that can offer a principled lens on mode collapse.
    \item Numerical experiments on simple examples  corroborate the theoretical findings: we visualize the evolution of the gradient kernel toward class-aware structure, verify linear separability of the tangent feature space, demonstrate that tangent-space rank gaps diagnose out-of-distribution extrapolation failure consistent with the null-space characterization, and show that centered kernel alignment between the tangent-feature kernel and the label kernel tracks training progress across architectures.
\end{itemize}

This paper is organized as follows. Section~\ref{prelimina} introduces the preliminaries, including Domingos’ path-kernel viewpoint, and the dynamical formulation of neural network training. Section~\ref{main_result} presents our main theoretical results. We first describe the weak-approximation framework that links discrete mini-batch iterates to continuous-time stochastic processes, then establish stochastic Domingos representations for various stochastic optimizers, and finally develop the kernel space projection viewpoint together with a trajectory-indexed adaptive RKHS formalism connecting generalization error to kernel null-space components. In Section~\ref{genertive_unify}, we leverage the path-kernel mechanism to reinterpret key deep learning paradigms, providing a unified perspective on diffusion models and GANs through sample-guided and distribution-guided correction dynamics. Section~\ref{numerical} reports numerical experiments that visualize implicit kernel evolution and illustrate the geometry of tangent feature spaces, including input similarity, linear separability of the feature space, tangent-space rank gaps and extrapolation failure, and central kernel alignment. Section~\ref{conclude} concludes the paper.

\section{Preliminaries}\label{prelimina}
Consider an objective function \(f: \mathbb{R}^p \times \mathbb{R}^d \to \mathbb{R}^m\) parameterized by \(\bTheta \in \mathbb{R}^d\). Let the training data \(\{(x_n,y_n^*)\}_{n=1}^N \subseteq \mathbb{R}^p \times \mathbb{R}^m\) be drawn independently and identically from an unknown distribution, denote by \(\mathcal{P}(x)\) the marginal distribution of \(x\). For an input \(x\), the model output is \(y(x) = f(x, \bTheta) \in \mathbb{R}^m\), and the loss function is \(\ell(\cdot,\cdot): \mathbb{R}^m \times \mathbb{R}^m \to \mathbb{R}\). Training a neural network corresponds to finding an optimal parameter \(\bTheta\) that minimizes the expected risk
\[
    \int \ell\bigl(f(x,\bTheta),\, y(x)\bigr) \, d\mathcal{P}(x).
\]
Because the true distribution of data is unknown, this expected risk cannot be directly minimized. Instead, one minimizes the empirical risk
\[
    L(\bTheta) = \frac{1}{N} \sum_{n=1}^N \ell\bigl(f(x_n,\bTheta),\, y_n^*\bigr).
\]
The standard gradient descent of the parameters \(\bTheta\) yields the iterative update
\[
    \bTheta_k = \bTheta_{k-1} - \eta \, \nabla L(\bTheta_{k-1}),
\]
where \(k\) is the iteration index, \(\eta > 0\) is the learning rate, and \(\bTheta_0\) is the vector of initial parameters. The most straightforward continuous-time approximation of gradient descent is given by letting the learning rate tend to zero, yielding the gradient flow
\begin{equation}\label{eq: gradient-flow}
    \dot{\bTheta}_t = - \nabla L(\bTheta_t),
\end{equation}
when learning rate $\eta\to 0$. This formulation provides the foundation for Domingos' path kernel theory, which we review next.

\subsection{Gradient kernel and Domingos' theory}

Remarkably, Domingos~\cite{domingos2020every} proved that gradient descent imposes a path kernel structure on any learned model. We now describe this formalism and its key results.

\begin{definition}[Gradient kernel and path kernel]
Let $\bTheta_t$, $t \in [0,T]$ denote the parameter trajectory of the 
gradient flow~\eqref{eq: gradient-flow}. The \textit{gradient kernel} 
between two inputs $(x, x')\in \mathbb{R}^{p}\times\mathbb{R}^p$ at 
time $t$ is
\begin{equation}\label{eq_gradient_kernel}
K_t(x,x') := \bigl\langle \nabla_{\bTheta} f(x,\bTheta_t),\, 
\nabla_{\bTheta} f(x',\bTheta_t) \bigr\rangle.
\end{equation}
The \textit{path kernel} is its time integral over the training trajectory:
\begin{equation}\label{eq_path_kernel}
K_{\mathrm{path}}(x,x') := \int_0^T K_t(x,x') \, dt.
\end{equation}
\end{definition}
Intuitively, the gradient kernel~\eqref{eq_gradient_kernel} measures the instantaneous similarity 
between the model's gradients at two inputs, and the path kernel~\eqref{eq_path_kernel} 
accumulates this similarity over the entire training trajectory. 
The following theorem shows that the output of any model trained by 
gradient descent can be expressed as a weighted sum of path kernels 
evaluated between the test input and each training point.

\begin{theorem}[Domingos' Theorem~\cite{domingos2020every}]\label{thm_domingo}
    Let \(f(\cdot,\bTheta)\) be a differentiable model trained on 
    \(\{(x_n,y_n^*)\}_{n=1}^N\) via gradient flow~\eqref{eq: gradient-flow} 
    on the empirical loss
    \(L(\bTheta) = \frac{1}{N}\sum_{n=1}^N \ell(f(x_n,\bTheta), y_n^*)\). 
    Then for any input \(x \in \mathbb{R}^p\), the model output at time 
    \(T\) satisfies
    \begin{equation}\label{eq:determ-domingos}
        f(x,\bTheta_T)
        = - \frac{1}{N}\sum_{n=1}^N 
\int_0^T \, K_t(x,x_n) \frac{\partial \ell}{\partial f}\bigl(f(x_n,\bTheta_t),y_n^*\bigr)\, dt + f(x,\bTheta_0),
    \end{equation}
\end{theorem}
\begin{proof}
    Let $\bTheta_t=(\Theta^1_t,\cdots,\Theta^d_t)\in \mathbb{R}^d$ denote the parameter of the neural network $f(\cdot,\bTheta)$. By the chain rule and the gradient flow~\eqref{eq: gradient-flow}, we have 
    \begin{align*}
        \frac{\mathrm{d}f}{\mathrm{d}t}(x,\bTheta_t)&=-\sum_{i=1}^d \frac{\partial f}{\partial \Theta^i_t}(x,\bTheta_t)\frac{\partial L(\bTheta_t)}{\partial \Theta^i_t}\\
        &=-\sum_{i=1}^d \frac{\partial f}{\partial \Theta^i_t}(x,\bTheta_t)\left(\frac{1}{N}\sum_{n=1}^N\frac{\partial f}{\partial \Theta^i_t}(x_n,\bTheta_t)^\top \frac{\partial \ell}{\partial f}(f(x_n,\bTheta_t),y_n^*)\right)\\
        &=-\frac{1}{N}\sum_{n=1}^N\sum_{i=1}^d \frac{\partial f}{\partial \Theta^i_t}(x,\bTheta_t)\frac{\partial f}{\partial \Theta^i_t}(x_n,\bTheta_t)^\top \frac{\partial \ell}{\partial f}(f(x_n,\bTheta_t),y_n^*)\\
        &=-\frac{1}{N}\sum_{n=1}^N\langle \nabla_\bTheta f(x,\bTheta_t),\nabla_\bTheta f(x_n,\bTheta_t)\rangle \frac{\partial \ell}{\partial f}(f(x_n,\bTheta_t),y_n^*) .
    \end{align*}
Integrating both sides from \(0\) to \(T\) gives~\eqref{eq:determ-domingos}.  
\end{proof}
Notably, Domingos' theorem extends directly to time-varying learning rates: if the gradient flow is replaced by $\dot{\bTheta}_t = -w(t)\nabla L(\bTheta_t)$ for a schedule $w(t)>0$, the representation~\eqref{eq:determ-domingos} acquires an additional scalar weight $w(t)$ inside the integral, but the path-kernel structure is otherwise unchanged; see Appendix~\ref{sec_path_kernel_scheduler}. For the common case of a quadratic loss, \(\partial\ell/\partial f = 2(f - y^*)\), \eqref{eq:determ-domingos} reduces to
\begin{equation}\label{simpleDomingo}
y(T) - y(0) = -\frac{2}{N}\sum_{n=1}^N
\int_0^T\bigl(f(x_n,\bTheta_t) - y_n^*\bigr)\,
K_t(x, x_n) \, dt,
\end{equation}
where \(y(T):=f(x,\bTheta_T)\) and \(y(0):=f(x,\bTheta_0)\). The change in the model's prediction is thus a weighted sum of training-point residuals \(f(x_n,\bTheta_t)-y_n^*\), with weights given by the gradient kernel integrated along the training trajectory.

A key feature of this representation is that similarity between data points is defined not by proximity in input space but by alignment of the model's gradients. Consequently, points that are far apart in Euclidean distance may be close in gradient space, and vice versa.

\subsection{Stochastic optimization at finite learning rate}
Domingos' theorem relies on the gradient flow limit \(\eta \to 0\), yet in practice neural networks are trained with finite learning rates and stochastic mini-batch updates. Stochastic gradient descent (SGD), originally proposed by Robbins and Monro~\cite{d8d62392-9a37-31e7-ad3b-37a6f6ee8ef6}, remains the workhorse of modern deep learning due to its efficiency, stability, and strong generalization performance~\cite{hardt2016train}. Classical results~\cite{pmlr-v40-Frostig15,doi:10.1137/070704277} demonstrate that SGD is nearly optimal for empirical risk minimization of convex losses, and subsequent work has established tighter bounds and stronger probabilistic guarantees~\cite{10.1007/11776420_37,hazan2014beyond,rakhlin2012making}. We now formalize the stochastic optimization at finite learning rate and with finite batches.

Assume \(N\) training points \(\{(x_i,y^*_i)\}_{i=1}^N\) are available, and let \(B \le N\) denote the batch size. Let \(\Gamma\) be the set of all subsets of \(\{1,\dots,N\}\) of size \(B\), and let \(\mathcal{B}\) denote a uniformly random element of \(\Gamma\). The mini-batch loss is
\[
L_{\mathcal{B}}(\bTheta) = \frac{1}{B}\sum_{i\in \mathcal{B}} \ell\bigl(f(x_i,\bTheta),y_i^*\bigr).
\]

\begin{assumption}\label{assumption}
For every \(i\), the per-sample loss \(\ell(f(x_i,\bTheta),y_i^*)\) is continuously differentiable in \(\bTheta\), and its gradient satisfies a uniform local bound: for each \(R > 0\) there exists a constant \(M_R > 0\) such that
\[
\max_{\|\bTheta\| \le R} \|\nabla_{\bTheta}\, \ell(f(x_i,\bTheta),y_i^*)\| \le M_R \quad \text{for all } i=1,\dots,N.
\]
\end{assumption}
This assumption is valid for every smooth network and loss, and for every bounded domain. Nevertheless, it excludes  NNs with less smooth activation functions such as ReLUs and less smooth losses such as $L^1$. 

Since \(\Gamma\) is finite and \(\mathcal{B}\) is uniformly distributed, the expected loss \(L(\bTheta) := \mathbb{E}_{\mathcal{B}}[L_{\mathcal{B}}(\bTheta)]\) is well defined, and interchanging the gradient with the finite sum gives
\[
\mathbb{E}_{\mathcal{B}}\bigl[\nabla_{\bTheta}L_{\mathcal{B}}(\bTheta)\bigr] 
= \nabla_{\bTheta}\,\mathbb{E}_{\mathcal{B}}\bigl[L_{\mathcal{B}}(\bTheta)\bigr] 
= \nabla_{\bTheta}L(\bTheta),
\]
so the mini-batch gradient is an unbiased estimator of the full-batch gradient. Let \(\{\mathcal{B}_k : k = 0,1,2,\dots\}\) be a sequence of i.i.d.\ uniform random batches. Starting from an initial point \(\bTheta_0\in\mathbb{R}^d\), the \(k\)-th iteration of SGD is
\begin{align}
    \bTheta_{k+1} = \bTheta_k - \eta \, \nabla L_{\mathcal{B}_k}(\bTheta_k),\label{sgd}
\end{align}
where \(\eta > 0\) is the learning rate.

With the SGD iterates~\eqref{sgd}, the random mini-batch sampling is of discrete nature. In contrast, Domingos' theorem (Theorem~\ref{thm_domingo}) relies on a continuous gradient flow~\eqref{eq: gradient-flow} and does not directly extend to the stochastic, finite learning rate setting. To bridge this gap, we adopt continuous-time SDE approximations of stochastic gradient algorithms~\cite{li2019stochastic,hu2017diffusion,mandt2016variational} as our starting point, with the formal development presented in Section~\ref{main_result}.

\section{Stochastic Domingos formulas}\label{main_result}
 In this section, we introduce the stochastic gradient kernel (Definition~\ref{def:stoch_kernel}) that replaces the deterministic gradient kernel of Domingos' theory with a stochastic counterpart evaluated along a continuous-time weak approximation of the mini-batch iterates. We then establish stochastic Domingos formulas for SGD (Theorem~\ref{thm:sgd_domingos}), SGDM (Theorem~\ref{thm:sgdm_domingos}), RMSprop (Theorem~\ref{thm:rmsprop_domingos}), and Adam (Theorem~\ref{thm:adam_domingos}), showing that the expected predictor retains a path-kernel aggregation structure with optimizer-dependent weighting. We also develop a kernel-space projection viewpoint that links the post-convergence generalization error to the null-space components of the integral operator induced by the stochastic gradient kernel (Theorem~\ref{thm:rkhs_decompose}).

\subsection{Stochastic gradient kernel via weak approximation}

The deterministic gradient kernel $K_t$ in Definition~\ref{eq_gradient_kernel} is evaluated along the gradient flow trajectory $\bTheta_t$, which is not available under stochastic training. To define a stochastic counterpart, we need a continuous-time process that faithfully approximates the discrete SGD iterates $\{\bTheta_k\}$. Rather than seeking a path-wise result analogous to Domingos' original theorem, we approximate $\{\bTheta_k\}_{k=0}^{\infty}$ with a continuous-time SDE in the sense of weak approximation, which controls expectations of smooth observables uniformly over all iterates.

Let $T > 0$, $\eta \in (0, 1 \wedge T)$ (where $a \wedge b := \min\{a, b\}$), and let $\alpha \geq 1$ be an integer. Set $K = \lfloor T / \eta \rfloor$ (where $\lfloor \cdot \rfloor$ denotes the floor function). Let $\cG$ be the set of continuous functions $g \colon \mathbb{R}^d \to \mathbb{R}$ satisfying
\[
    |g(z)| \leq \kappa_1 \bigl( 1 + |z|^{2\kappa_2} \bigr),
    \quad z \in \mathbb{R}^d,
\]
for some $\kappa_1, \kappa_2 > 0$. For $\alpha \geq 1$, define
\[
    \cG^\alpha := \bigl\{ g \in \cC^\alpha(\mathbb{R}^d) :
    \partial^\beta g \in \cG \text{ for all } |\beta| \leq \alpha \bigr\},
\]
where $\beta$ is a multi-index. Thus $\cG^\alpha$ is the subset of $\cC^{\alpha}(\mathbb{R}^d)$, the space of $\alpha$-times continuously differentiable functions, whose derivatives up to order $\alpha$ have at most polynomial growth.

\begin{definition}[$\alpha$-th order weak approximation~\cite{li2019stochastic}]\label{def:approx}
    We say that a continuous-time stochastic process $\{\bZ_t : t \in [0, T]\}$ in $\mathbb{R}^d$ is an $\alpha$-th order weak approximation of a discrete stochastic process $\{\bTheta_k : k = 0, \dots, K\}$ in $\mathbb{R}^d$ if, for every $g \in \cG^{\alpha+1}$, there exists a positive constant $C$, independent of $\eta$, such that
    \begin{equation*}
        \max_{k = 0, \dots, K} \bigl| \mathbb{E}[g(\bTheta_k)] - \mathbb{E}[g(\bZ_{k\eta})] \bigr| \leq C\,\eta^\alpha.
    \end{equation*}
\end{definition}

In Definition~\ref{def:approx}, the elements of $\cG^{\alpha+1}$ serve as test functions. This space collects sufficiently smooth functions with at most polynomial growth---a condition mild enough to include the observables of interest, yet strong enough to guarantee integrability and to justify the It\^{o}--Taylor expansion rigorously. This framework is widely used to analyse discretisation errors of SDEs when modelling SGD by continuous-time approximations (\cite{kloeden2013numerical,feng2019uniform,li2019stochastic}). Crucially, the resulting error bounds hold uniformly over all SGD iterations up to $K = \lfloor T/\eta\rfloor$ steps, not just at the final iterate.

To connect the discrete-time SGD dynamics with a continuous-time description, we adopted the diffusion-approximation viewpoint. For a given learning rate~$\eta$, the SGD iterates admit a first-order weak approximation by an It\^{o} SDE whose drift matches the full gradient $-\nabla L(\bTheta)$ and whose diffusion coefficient encodes the mini-batch gradient-noise covariance. The following lemma formalizes this approximation and establishes an error bound $O(\eta)$  on the expectation mismatch between the discrete iterates and the continuous process.

\begin{lemma}[First-order weak approximation of SGD~\cite{li2019stochastic}]\label{lem:1order}
    Let $T > 0$, $\eta \in (0, 1 \wedge T)$, and set $K = \lfloor T/\eta \rfloor$. Let $\{\bTheta_k\}_{k \geq 0}$ be the SGD iterates defined in~\eqref{sgd}. Assume that $L \in \cG^3$ and that the test function satisfies $g \in \cG^2$.
    Define $\{\bZ_t : t \in [0, T]\}$ as the solution of the SDE
    \begin{equation}\label{eq:1order}
        \mathrm{d}\bZ_t = -\nabla L(\bZ_t)\,\mathrm{d}t
        + \sqrt{\eta}\,\bSigma_{|\cB|}^{1/2}(\bZ_t)\,\mathrm{d}\mathbf{W}_t,
        \qquad \bZ_0 = \bTheta_0,
    \end{equation}
    where
    \[
        \bSigma_{|\cB|}(\bZ)
        = \mathbb{E}_{\cB}\!\bigl[
            \bigl(\nabla L_{\cB}(\bZ) - \nabla L(\bZ)\bigr)
            \bigl(\nabla L_{\cB}(\bZ) - \nabla L(\bZ)\bigr)^{\!\top}
        \bigr].
    \]
    Then $\{\bZ_t\}$ is a first-order weak approximation of the SGD iterates; that is, there exists a constant $C > 0$, independent of $\eta$, such that
    \begin{equation}\label{eq:1order_bound}
        \max_{k = 0, \dots, K}
        \bigl|\mathbb{E}[g(\bTheta_k)] - \mathbb{E}[g(\bZ_{k\eta})]\bigr|
        \leq C\,\eta.
    \end{equation}
\end{lemma}
Lemma~\ref{lem:1order} ensures that, for any sufficiently smooth test function, the continuous-time process~\eqref{eq:1order} faithfully tracks the expected behavior of SGD uniformly over all $K$ iterates, with an approximation error that vanishes linearly in the step size. This continuous-time representation will serve as the foundation for deriving our main results.

Now that we have a continuous-time approximation $\{\bZ_t\}$ of the SGD iterates, we introduce the notation for the gradient inner products that appear in our main results.

\begin{definition}[Stochastic gradient kernel]\label{def:stoch_kernel}
For two inputs $(x,x') \in \mathbb{R}^p \times \mathbb{R}^p$ and times $(t,s) \in [0,T]^2$, define
\begin{equation}\label{eq:stoch_gradient_kernel}
    \widetilde{K}_{t,s}(x,x') := \bigl\langle \nabla_{\!\bTheta} f(x,\bZ_t),\,
                   \nabla_{\!\bTheta} f(x',\bZ_s)\bigr\rangle,
\end{equation}
where $\{\bZ_t\}$ is the first-order weak approximation of the SGD iterates defined in Lemma~\ref{lem:1order}. When both gradients are evaluated at the same time, we write $\widetilde{K}_t(x,x') := \widetilde{K}_{t,t}(x,x')$.
\end{definition}

The two-time form $\widetilde{K}_{t,s}$ arises naturally in the presence of momentum, where a gradient signal produced at time $s$ interacts with the tangent feature at a later time $t$ through an exponential memory kernel (see Section~\ref{sec:sgdm}). Since $\bZ_t$ is a random process, $\widetilde{K}_{t,s}(x,x')$ is a random variable for each fixed $(x,x',t,s)$.

Moreover, adaptive optimizers such as RMSprop and Adam rescale gradient directions through a time-varying preconditioner, which modifies the feature space where gradient similarities are measured.

\begin{definition}[Weighted stochastic gradient kernel]\label{def:weighted_kernel}
Given a symmetric positive definite matrix $\bP$, the $\bP^{-1}$-weighted stochastic gradient kernel is
\begin{equation}\label{eq:weighted_gradient_kernel}
    \widetilde{K}_{t,s}^{\,\bP}(x,x') := \bigl\langle \nabla_{\!\bTheta} f(x,\bZ_t),\,
                   \nabla_{\!\bTheta} f(x',\bZ_s)\bigr\rangle_{\bP^{-1}},
\end{equation}
and we set $\widetilde{K}_t^{\,\bP}(x,x') := \widetilde{K}_{t,t}^{\,\bP}(x,x')$.
\end{definition}

When $\bP = \mathbf{I}_d$, the weighted kernel reduces to $\widetilde{K}_{t,s}$ of Definition~\ref{def:stoch_kernel}.

\subsection{Stochastic Domingos theorem for SGD}
The following theorem shows that the path-kernel representation established by Domingos' Theorem~\ref{thm_domingo} extends to SGD: the expected output retains the same integral structure, up to an $O(\eta)$ remainder.

\begin{theorem}[Stochastic Domingos' Theorem (SGD)]\label{thm:sgd_domingos}
    Consider a learning model $y = f(x,\bTheta)$. Assume that $\bTheta$ is learned from a dataset $\{(x_n,y_n^*)\}_{n=1}^N$ by SGD with learning rate $\eta$, and that the conditions of Lemma~\ref{lem:1order} hold with $g(\cdot) = f(x,\cdot)$. Then
    \begin{align}
        \mathbb{E}\bigl[f(x,\bTheta_{k})\bigr]
        &= f(x,\bTheta_0)
           - \mathbb{E}\!\left[
             \int_0^{k\eta} \frac{1}{N}\sum_{n=1}^N
             \widetilde{K}_s(x,x_n)\,
             \frac{\partial \ell}{\partial f}\bigl(f(x_n,\bZ_s),y_n^*\bigr)
             \,\mathrm{d}s
           \right] \notag\\
        &\quad + O(\eta). \label{eq:sgd_result}
    \end{align}
\end{theorem}

\begin{proof}
    For a given input $x$, the expected output at iteration $k$ can be decomposed as
    \begin{align}
        \mathbb{E}[f(x,\bTheta_{k})]
        &= \mathbb{E}[f(x,\bZ_{k\eta})]
           + \mathbb{E}[f(x,\bTheta_{k})] - \mathbb{E}[f(x,\bZ_{k\eta})] \notag\\
        &= \mathbb{E}[f(x,\bZ_{k\eta})] + O(\eta), \label{eq:expectation}
    \end{align}
    where the second equality follows from Lemma~\ref{lem:1order}. Since $\bZ_t$ is a continuous It\^{o} process, the It\^{o} formula gives
    \begin{align}\label{eq:ito1}
        f(x,\bZ_{k\eta})
        = f(x,\bZ_0)
          + \sum_{j=1}^d \int_0^{k\eta}
            \frac{\partial f}{\partial \Theta^j}(x,\bZ_s)\,\mathrm{d}Z_s^j
          + \frac{1}{2}\sum_{j,l=1}^d \int_0^{k\eta}
            \frac{\partial^2 f}{\partial \Theta^j \partial \Theta^l}(x,\bZ_s)
            \,\mathrm{d}\langle Z^j, Z^l \rangle_s,
    \end{align}
    where $Z_t^j$ denotes the $j$-th coordinate of $\bZ_t$. Using~\eqref{eq:1order}, we obtain the drift and quadratic variations
    \begin{equation*}
        \mathrm{d}Z_t^j
        = -\frac{\partial L}{\partial \Theta^j}(\bZ_t)\,\mathrm{d}t
          + \sqrt{\eta}\sum_{h=1}^d a_{jh}(\bZ_t)\,\mathrm{d}W_t^h,
        \qquad
        \mathrm{d}\langle Z^j, Z^l \rangle_s
        = \eta \sum_{h=1}^d a_{jh}(\bZ_s)\,a_{lh}(\bZ_s)\,\mathrm{d}s,
    \end{equation*}
    where $\bSigma_{|\cB|}^{1/2}(\bZ_t) = \bigl(a_{jl}(\bZ_t)\bigr)_{j,l=1}^{d}$, $W_t^h$ is the $h$-th component of $d$-dimesional Brownial motion $\bW_t$. Substituting into~\eqref{eq:ito1} yields
    \begin{align}\label{eq:f_simplified}
        f(x,\bZ_{k\eta})
        &= f(x,\bTheta_0)
           - \int_0^{k\eta}
             \bigl(\nabla_{\!\bTheta} f(x,\bZ_s)\bigr)^{\!\top}
             \nabla_{\!\bTheta} L(\bZ_s)\,\mathrm{d}s
           + \sqrt{\eta}\int_0^{k\eta}
             \bigl(\nabla_{\!\bTheta} f(x,\bZ_s)\bigr)^{\!\top}
             \bSigma_{|\cB|}^{1/2}(\bZ_s)\,\mathrm{d}\mathbf{W}_s \notag\\
        &\quad
           + \frac{\eta}{2}\int_0^{k\eta}
             \mathrm{Tr}\!\left(
               \nabla_{\!\bTheta}^2 f(x,\bZ_s)\,
               \bSigma_{|\cB|}(\bZ_s)
             \right)\mathrm{d}s.
    \end{align}
    The last term is $O(\eta)$, since $f \in \cG^2$ and $\bSigma_{|\cB|}$ is locally bounded by Assumption~\ref{assumption}. Expanding $\nabla_{\!\bTheta} L = \frac{1}{N}\sum_{n=1}^N \nabla_{\!\bTheta} f(x_n,\bZ_s)\,\frac{\partial \ell}{\partial f}(f(x_n,\bZ_s),y_n^*)\,$ and taking expectations,
    \begin{align}
        \mathbb{E}[f(x,\bTheta_{k})]
        &= f(x,\bTheta_0)
           - \mathbb{E}\!\left[\int_0^{k\eta}
             \frac{1}{N}\sum_{n=1}^N
             \widetilde{K}_s(x,x_n)\,
             \frac{\partial \ell}{\partial f}\bigl(f(x_n,\bZ_s),y_n^*\bigr)
             \,\mathrm{d}s\right] \notag\\
        &\quad
           + \sqrt{\eta}\,\mathbb{E}\!\left[\int_0^{k\eta}
             \bigl(\nabla_{\!\bTheta} f(x,\bZ_s)\bigr)^{\!\top}
             \bSigma_{|\cB|}^{1/2}(\bZ_s)\,\mathrm{d}\mathbf{W}_s\right]
           + O(\eta). \label{eq:before_martingale}
    \end{align}
    It remains to show that the stochastic integral term vanishes in expectation. The polynomial growth of $\nabla_{\!\bTheta} f$ (from $f \in \cG^2$) together with the moment bounds on $\bZ_t$ guarantee that the integrand is square-integrable:
    \begin{equation*}
        \mathbb{E}\!\left[\int_0^{k\eta}
          \bigl\|(\nabla_{\!\bTheta} f(x,\bZ_s))^{\!\top}
          \bSigma_{|\cB|}^{1/2}(\bZ_s)\bigr\|^2\,\mathrm{d}s\right] < \infty.
    \end{equation*}
    Hence the It\^{o} integral is a martingale~\cite[Theorem~3.2.1]{oksendal2013stochastic}, and its expectation vanishes. Substituting into~\eqref{eq:before_martingale} yields the desired result. \qedhere
\end{proof}

Theorem~\ref{thm:sgd_domingos} shows that, under SGD, the expected output admits a path-kernel representation to first order with respect to the learning rate. The correction to the initial prediction $f(x,\bTheta_0)$ is a weighted superposition of training gradient features $\nabla_{\!\bTheta} f(x_n,\bZ_s)$ accumulated along the learning trajectory, with weights given by the loss sensitivity $\frac{\partial \ell}{\partial f}(f(x_n,\bZ_s),y_n^*)$ and similarities measured by the stochastic gradient kernel $\widetilde{K}_s(x,x_n)$. Test data enter only through these kernel similarities, so the prediction can be interpreted as kernel-weighted retrieval of training information stored along the stochastic trajectory. This establishes that Domingos' interpolation structure persists under mini-batch noise, in expectation and up to $O(\eta)$.

Notably, Theorem~\ref{thm:sgd_domingos} is stated under the regularity conditions of Assumption~\ref{assumption} and Lemma~\ref{lem:1order}, which require smoothness of both the loss function and the network output $f$ with respect to $\bTheta$. The commonly used ReLU networks satisfy these conditions only piecewise: the parameter space is partitioned into regions with fixed activation patterns, within each of which $\bTheta \mapsto f(x,\bTheta)$ is continuously differentiable and Theorem~\ref{thm:sgd_domingos} applies. However, SGD may repeatedly cross activation-switching boundaries during training, and a rigorous extension to the general piecewise $\cG^3$ case would require quantitative control on these boundary visits, which is difficult to guarantee without additional assumptions on the distribution of pre-activations near the switching thresholds. We shall leave the investigations to future work.

\subsection{Extensions to momentum and adaptive methods}\label{sec:sgdm}
The path-kernel representation of Theorem~\ref{thm:sgd_domingos} extends to SGDM~\cite{sutskever2013importance}, RMSprop~\cite{tieleman2012rmsprop}, and Adam~\cite{kingma2014adam}, with each optimizer introducing a distinct modification to the kernel weighting structure. Momentum adds an exponential temporal decay that carries past gradient signals forward; adaptive methods reshape the interpolation geometry through time-varying preconditioners that reweight gradient feature directions. In all cases, the expected output retains a Domingos-type aggregation structure to first order in the learning rate. The algorithmic details and first-order SDE approximations for each optimizer are provided in Appendices~\ref{appendix sgdm},~\ref{appendix rmsprop}, and~\ref{appendix adam}, respectively.

We begin with SGDM. The momentum mechanism introduces an exponential memory kernel $e^{-(t-s)\mu}$ that filters past gradient contributions, so that the path kernel acquires a temporal weighting absent in the SGD case. Unlike the SGD proof, which applies the It\^{o}'s formula directly to the parameter process, the SGDM proof requires solving for the momentum state via an integrating factor before recovering the parameter dynamics.

\begin{theorem}[Stochastic Domingos' Theorem (SGDM)]\label{thm:sgdm_domingos}
    Consider a learning model $y = f(x,\bTheta)$. Assume that $\bTheta$ is learned from the dataset $\{(x_n,y_n^*)\}_{n=1}^N$ by SGDM with learning rate $\eta$ and momentum weight $\beta \in [0,1)$, and that the conditions of Lemma~\ref{lem:1order} hold with $g(\cdot) = f(x,\cdot)$. Let $\mu = (1-\beta)/\eta$. Then
    \begin{align}
        \mathbb{E}[f(x,\bTheta_{k})]
        &= f(x,\bTheta_0)
           - \mathbb{E}\!\left[\int_0^{k\eta}\!\int_0^t
             e^{-(t-s)\mu}\,
             \frac{1}{N}\sum_{n=1}^N
             \widetilde{K}_{t,s}(x,x_n)\,
             \frac{\partial \ell}{\partial f}\bigl(f(x_n,\bZ_s),y_n^*\bigr)
             \,\mathrm{d}s\,\mathrm{d}t\right] \notag\\
        &\quad + O(\eta). \label{eq:sgdm_domingos}
    \end{align}
\end{theorem}

\begin{proof}
    As in the proof of Theorem~\ref{thm:sgd_domingos}, the weak approximation of Lemma~\ref{lem:1order} gives
    \begin{align}\label{eq:sgdm_decomp}
        \mathbb{E}[f(x,\bTheta_{k})]
        = \mathbb{E}[f(x,\bZ_{k\eta})] + O(\eta).
    \end{align}
    To evaluate $f(x,\bZ_{k\eta})$, we first solve for the momentum state.
    Let $\{\bM_t,\bZ_t:t\in [0, k\eta]\}$ denote the continuous analogue of momentum and parameters, defined by the SDEs (details are provided in Appendix~\ref{appendix sgdm})
    \begin{align}
        & \mathrm{d}\bM_t = -[\mu \bM_t+\nabla_\bTheta L(\bZ_t)]\mathrm{d}t+\sqrt{\eta}\bs_{|\mathcal{B}|}^\frac{1}{2}(\bZ_t)\mathrm{d}\bW_t, \bM_0=m_0\label{sde momentum}\\
        & \mathrm{d}\bZ_t = \bM_t \mathrm{d}t,\bZ_0=\bTheta_0.
    \end{align}
    Applying the integrating factor $e^{\mu t}$ to the momentum SDE~\eqref{sde momentum} and using $\bM_0 = 0$,
    \begin{align*}
        \bM_t
        = -\int_0^t e^{-(t-s)\mu}\nabla_{\!\bTheta} L(\bZ_s)\,\mathrm{d}s
          + \sqrt{\eta}\int_0^t e^{-(t-s)\mu}\bSigma_{|\cB|}^{1/2}(\bZ_s)\,\mathrm{d}\mathbf{W}_s.
    \end{align*}
    Since $\mathrm{d}\bZ_t = \bM_t\,\mathrm{d}t$ contains no diffusion term, the  chain rule gives
    \begin{align*}
        \frac{\mathrm{d}}{\mathrm{d}t}\bigl[f(x,\bZ_t)\bigr]
        = \bigl(\nabla_{\!\bTheta} f(x,\bZ_t)\bigr)^{\!\top}\bM_t.
    \end{align*}
    Substituting the expression for $\bM_t$, integrating over $[0, k\eta]$, and expanding $\nabla_{\!\bTheta} L$,
    \begin{align*}
        f(x,\bZ_{k\eta})
        &= f(x,\bTheta_0)
           - \int_0^{k\eta}\!\int_0^t
             e^{-(t-s)\mu}\,
             \frac{1}{N}\sum_{n=1}^N
             \widetilde{K}_{t,s}(x,x_n)\,
             \frac{\partial \ell}{\partial f}\bigl(f(x_n,\bZ_s),y_n^*\bigr)
             \,\mathrm{d}s\,\mathrm{d}t \\
        &\quad
           + \sqrt{\eta}\int_0^{k\eta}
             \bigl(\nabla_{\!\bTheta} f(x,\bZ_t)\bigr)^{\!\top}
             \int_0^t e^{-(t-s)\mu}
             \bSigma_{|\cB|}^{1/2}(\bZ_s)\,\mathrm{d}\mathbf{W}_s\,\mathrm{d}t.
    \end{align*}
    It remains to show that the stochastic term is $O(\eta)$ after taking expectations. Define $$\mathbf{J}_t = \int_0^t e^{-(t-s)\mu}\bSigma_{|\cB|}^{1/2}(\bZ_s)\,\mathrm{d}\mathbf{W}_s.$$ By the It\^{o} isometry and the bound $\sup_{t \leq T} \mathbb{E}[\mathrm{Tr}(\bSigma_{|\cB|}(\bZ_t))] \leq M_{\bSigma}$,
    \begin{align*}
        \mathbb{E}\bigl[\|\mathbf{J}_t\|^2\bigr]
        \leq M_{\bSigma}\int_0^t e^{-2(t-s)\mu}\,\mathrm{d}s
        \leq \frac{M_{\bSigma}}{2\mu}
        = \frac{M_{\bSigma}\,\eta}{2(1-\beta)}.
    \end{align*}
    By the Cauchy--Schwarz inequality and $\sup_{t \leq T} \mathbb{E}[\|\nabla_{\!\bTheta} f(x,\bZ_t)\|^2] \leq M_f$,
    \begin{align*}
        \left|\mathbb{E}\!\left[
          \bigl(\nabla_{\!\bTheta} f(x,\bZ_t)\bigr)^{\!\top}\mathbf{J}_t
        \right]\right|
        \leq \sqrt{M_f}\,\sqrt{\frac{M_{\bSigma}\,\eta}{2(1-\beta)}}.
    \end{align*}
    Multiplying by $\sqrt{\eta}$ and integrating over $[0, k\eta]$ gives an $O(\eta)$ contribution. Combining with~\eqref{eq:sgdm_decomp} yields the desired result. \qedhere
\end{proof}

Theorem~\ref{thm:sgdm_domingos} demonstrates that, under SGDM, the expected output again admits a path-kernel representation to first order. Training examples with higher loss sensitivity exert stronger influence, but their contributions are filtered through the exponential decay $e^{-\mu(t-s)}$: a gradient signal produced at time $s$ influences $f(x)$ through all past tangent directions $\nabla_{\!\bTheta} f(x,\bZ_t)$ for $t \geq s$ with exponentially decreasing strength. When $\mu$ is large (weak momentum), the memory kernel collapses and the result reduces to Theorem~\ref{thm:sgd_domingos}.

Whereas momentum modifies the temporal weighting of past gradients in SGDM, RMSprop applies a coordinate-wise adaptive rescaling based on an exponential moving average of past squared gradients, downweighting directions with large fluctuations and amplifying stable ones. This reshapes how training examples influence the trajectory through an evolving preconditioner. We show that the path-kernel aggregation structure is preserved, with the kernel weights now modulated by the RMSprop preconditioner.

\begin{theorem}[Stochastic Domingos' Theorem (RMSprop)]\label{thm:rmsprop_domingos}
    Consider a learning model $y = f(x,\bTheta)$. Assume that $\bTheta$ is learned from the dataset $\{(x_n,y_n^*)\}_{n=1}^N$ by RMSprop with learning rate $\eta$, second-moment decay rate $\beta \in [0,1)$, and smoothing constant $\varepsilon > 0$. Let $\mu = (1-\beta)/\eta$ and suppose that $\mu = O(1)$ as $\eta \to 0$. Assume that the conditions of Lemma~\ref{lem:1order} hold with $g(\cdot) = f(x,\cdot)$. Then
    \begin{align}\label{eq:rmsprop_domingos}
        \mathbb{E}[f(x,\bTheta_{k})]
        &= f(x,\bTheta_0)
           - \mathbb{E}\!\left[\int_0^{k\eta}
             \frac{1}{N}\sum_{n=1}^N
             \widetilde{K}_t^{\,\bP_t}(x,x_n)\,
             \frac{\partial \ell}{\partial f}\bigl(f(x_n,\bZ_t),y_n^*\bigr)
             \,\mathrm{d}t\right] \notag\\
        &\quad + O(\eta),
    \end{align}
    where
    \[
        \bP_t
        = \mathrm{diag}\!\left(
            \mu\int_0^t e^{-\mu(t-s)}
            \bigl[(\nabla L(\bZ_s))^2 + \mathrm{diag}\bigl(\bSigma(\bZ_s)\bigr)\bigr]
            \,\mathrm{d}s
          \right)
          + \varepsilon\,\mathbf{I}_d.
    \]
\end{theorem}
\begin{proof}
    For any input $x$, the expectation of output in $k$-th iteration can be written as
    \begin{align}\label{rms expec}
        \mathbb{E}[f(x,\bTheta_k)]&=\mathbb{E}[f(x,\bZ_{k\eta})]+\mathbb{E}[f(x,\bTheta_k)]-\mathbb{E}[f(x,\bZ_{k\eta})]\notag\\
        &=\mathbb{E}[f(x,\bZ_{k\eta})]+O(\eta).
    \end{align}

    Let $\{\bZ_t, \bS_t : t \in [0, k\eta]\}$ denote the continuous process of parameters and second-moment, defined by the following SDEs (see Appendix~\ref{appendix rmsprop})
    \begin{align}
        &\mathrm{d}\bZ_t=-\bP_t^{-1}(\nabla L(\bZ_t)dt+\sqrt{\eta}\bs^\frac{1}{2}_{|\mathcal{B}|}(\bZ_t)\mathrm{d}\bW_t)\\
        &\mathrm{d}\bS_t=\mu((\nabla L(\bZ_t))^2+\mathrm{diag}(\bs(\bZ_t))-\bS_t)\mathrm{d}t.\label{sde_s}
    \end{align}
    Since
    \begin{equation*}
        \mathrm{d}(e^{\mu t}\bS_t)=\mu e^{\mu t}\bS_t\mathrm{d}t+e^{\mu t}\mathrm{d}\bS_t,
    \end{equation*}
    by Eq.(\ref{sde_s}) we have
    \begin{align*}
        \mathrm{d}(e^{\mu t}\bS_t)&=\mu e^{\mu t}\bS_t\mathrm{d}t+e^{\mu t}\mu((\nabla L(\bZ_t))^2+\mathrm{diag}(\bs(\bZ_t))-\bS_t)\mathrm{d}t\\
        &=\mu e^{\mu t}((\nabla L(\bZ_t))^2+\mathrm{diag}(\bs(\bZ_t)))\mathrm{d}t.
    \end{align*}
    Integrating on both sides, since $\bS_0=0$, we have
    \begin{align*}
        \bS_t=\mu \int_0^t e^{-\mu(t-s)}((\nabla L(\bZ_s))^2+\mathrm{diag}(\bs(\bZ_s)))\mathrm{d}s,
    \end{align*}
    and
    \begin{align*}
        \bP_t=\mathrm{diag}(\bS_t)^\frac{1}{2}+\varepsilon \mathbf{I}_d=\mathrm{diag}\left(\mu \int_0^t e^{-\mu(t-s)}((\nabla L(\bZ_s))^2+\mathrm{diag}(\bs(\bZ_s)))\mathrm{d}s\right)+\varepsilon \mathbf{I}_d.
    \end{align*}
    From Eq. (\ref{sde rmsp}), we have
    \begin{align*}
        \mathrm{d}\bZ_t=\bP^{-1}_t(\nabla L(\bZ_t)\mathrm{d}t+\sqrt{\eta}\bs^\frac{1}{2}(\bZ_t)\mathrm{d}\bW_t).
    \end{align*}
By It\'o formula, we obtain
    \begin{align*}
f(x,\bZ_{k\eta})&=f(x,\bZ_0)-\int_0^{k\eta} \nabla f(x,\bZ_t)\mathrm{d}\bZ_t+\frac{1}{2}\int_0^{k\eta} \eta\mathrm{Tr}( \bP_t^{-1}\bs_{|\mathcal{B}|}(\bZ_t)\bP_t\nabla^2_\bTheta f(x,\bZ_t))\mathrm{d}t\\
&=f(x,\bZ_0)-\int_0^{k\eta} \frac{1}{N}\sum_{n=1}^N\nabla f(x,\bZ_t)\bP^{-1}_t (\nabla f(x,\bZ_t))^\top \frac{\partial \ell}{\partial f}(f(x_n,\bZ_t),y_n^*)\mathrm{d}t\\&-\sqrt{\eta}\int_0^{k\eta} \nabla f(x,\bZ_t) \bP^{-1}_t\bs_{|\mathcal{B}|}^\frac{1}{2}(\bZ_t)\mathrm{d}\bW_t+\frac{\eta}{2}\int_0^{k\eta} \mathrm{Tr}( \bP_t^{-1}\bs_{|\mathcal{B}|}(\bZ_t)\bP_t\nabla^2_\bTheta f(x,\bZ_t))\mathrm{d}t.
    \end{align*}
Taking expectation to both sides, since
\begin{align*}
    \mathbb{E}\bigg[\int_0^{k\eta}\left\|\nabla f(x,\bZ_t) \bigg(\mathrm{diag}\bigg(\mu \int_0^t e^{-\mu(t-s)}((\nabla L(\bZ_s))^2+\mathrm{diag}(\bs(\bZ_s)))ds\bigg)+\varepsilon \mathbf{I}_d\bigg)^{-1} \bs_{|\mathcal{B}|}^\frac{1}{2}(\bZ_t)\right\|^2dt\bigg]<\infty,
\end{align*}
we obtain
\begin{align*}
    &\mathbb{E}\left[\int_0^{k\eta} \nabla f(x,\bZ_t) \bP^{-1}_t\bs_{|\mathcal{B}|}^\frac{1}{2}(\bZ_t)\mathrm{d}\bW_t\right]\\&=\mathbb{E}\bigg[\int_0^{k\eta}\nabla f(x,\bZ_t) \bigg(\mathrm{diag}\bigg(\mu \int_0^t e^{-\mu(t-s)}((\nabla L(\bZ_s))^2+\mathrm{diag}(\bs(\bZ_s)))\mathrm{d}s\bigg)+\varepsilon \mathbf{I}_d\bigg)^{-1} \bs_{|\mathcal{B}|}^\frac{1}{2}(\bZ_t)\mathrm{d}\bW_t\bigg]\\&=0.
\end{align*}
Then, by chain rule, we obtain
\begin{align*}
    \mathbb{E}[f(x,\bZ_{k\eta})]&=f(x,\bZ_{0})-\mathbb{E}\left[\int_0^{k\eta} \frac{1}{N}\sum_{n=1}^N\widetilde{K}_t(x,x_n) \frac{\partial \ell}{\partial f}(f(x_n,\bZ_t),y_n^*)\mathrm{d}t\right]\\
    &\quad+\frac{\eta}{2}\mathbb{E}\left[\int_0^{k\eta} \mathrm{Tr}( \bP_t^{-1}\bs_{|\mathcal{B}|}(\bZ_t)\bP_t\nabla^2_\bTheta f(x,\bZ_t))\mathrm{d}t\right].
\end{align*}
Plugging this last equation into Eq.(\ref{rms expec}) yields the result.
\end{proof}

Theorem~\ref{thm:rmsprop_domingos} shows that, under RMSprop, the expected output still admits a Domingos-type path-kernel representation, but the similarity weights are shaped by an adaptive, variance-dependent preconditioner. The stochastic gradient kernel is no longer the plain Euclidean gradient alignment: it is measured through the $\bP_t^{-1}$-weighted kernel $\widetilde{K}_t^{\,\bP}(x,x_n)$, where $\bP_t$ encodes the running second-moment statistics. Consequently, parameter directions with large historical gradient fluctuations are suppressed, and training samples contribute primarily through the stable directions emphasised by $\bP_t^{-1}$.

Adam combines both mechanisms: the temporal memory of SGDM and the coordinate-wise adaptive rescaling of RMSprop. This makes the approximation more involved, as the update rule is governed by both an exponential first-moment state and a variance-dependent preconditioner. Nevertheless, the path-kernel aggregation structure is preserved.

\begin{theorem}[Stochastic Domingos' Theorem (Adam)]\label{thm:adam_domingos}
    Consider a learning model $y = f(x,\bTheta)$. Assume that $\bTheta$ is learned from the dataset $\{(x_n,y_n^*)\}_{n=1}^N$ by Adam with learning rate $\eta$, first-moment coefficient $\beta_1 \in [0,1)$, second-moment coefficient $\beta_2 \in [0,1)$, and smoothing constant $\varepsilon > 0$. Let $c_1 = (1-\beta_1)/\eta$, $c_2 = (1-\beta_2)/\eta$, and suppose that $c_1 = O(\eta^{-\xi})$ for some $\xi \in (0,1)$ and $c_2 = O(1)$ as $\eta \to 0$. Assume that the conditions of Lemma~\ref{lem:1order} hold with $g(\cdot) = f(x,\cdot)$. Then
    \begin{align}\label{eq:adam_domingos}
        &\mathbb{E}[f(x,\bTheta_{k})]
        = f(x,\bTheta_0) \notag\\
        &\quad - c_1\beta_1\,\mathbb{E}\!\left[\int_0^{k\eta}
          \frac{\sqrt{1-e^{-c_2 t}}}{1-e^{-c_1 t}}
          \int_0^t e^{-c_1(t-s)}\,
          \frac{1}{N}\sum_{n=1}^N
          \widetilde{K}_{t,s}^{\,\bP_t}(x,x_n)\,
          \frac{\partial \ell}{\partial f}\bigl(f(x_n,\bZ_s),y_n^*\bigr)
          \,\mathrm{d}s\,\mathrm{d}t\right] \notag\\
        &\quad - c_1(1-\beta_1)\,\mathbb{E}\!\left[\int_0^{k\eta}
          \frac{\sqrt{1-e^{-c_2 t}}}{1-e^{-c_1 t}}\,
          \frac{1}{N}\sum_{n=1}^N
          \widetilde{K}_t^{\,\bP_t}(x,x_n)\,
          \frac{\partial \ell}{\partial f}\bigl(f(x_n,\bZ_t),y_n^*\bigr)
          \,\mathrm{d}t\right] \notag\\
        &\quad + O(\sqrt{\eta}),
    \end{align}
    where
    \[
        \bP_t
        = c_2^{1/2}\,\mathrm{diag}\!\left(
            \int_0^t e^{-c_2(t-s)}
            \bigl[\mathrm{diag}\bigl(\bSigma(\bZ_s)\bigr) + (\nabla L(\bZ_s))^2\bigr]
            \,\mathrm{d}s
          \right)^{\!1/2}
          + \varepsilon\sqrt{1 - e^{-c_2 t}}\,\mathbf{I}_d.
    \]
\end{theorem}
The proof is given in Appendix, section~\ref{proofofAdam}. A notable difference from the preceding results is the weaker remainder $O(\sqrt{\eta})$ in place of $O(\eta)$. This arises because Adam injects gradient noise into the first-moment state and feeds it back into the parameter dynamics through exponentially weighted memory. Although this memory is smoothed by the exponential kernel with amplitude on the order of $1/\sqrt{c_1}$, it re-enters the parameter dynamics with a prefactor $c_1\sqrt{\eta}$, leaving an unavoidable $c_1^{1/2}$ scaling loss. In practice, $\beta_1$ is a fixed constant (e.g., $0.9$) that does not scale with the step size, so $c_1 = (1-\beta_1)/\eta \to \infty$ as $\eta \to 0$, and the stochastic term can no longer be absorbed into an $O(\eta)$ remainder as in SGD or RMSprop.

\subsection{Generalization error and the kernel null space}\label{sec:adaptive_rkhs}

The stochastic path kernel representations in Theorems~\ref{thm:sgd_domingos}--\ref{thm:adam_domingos} show that the prediction correction at each step is confined to the space spanned by training tangent features. A test tangent feature $\nabla_{\!\bTheta} f(x,\bTheta_k)$ may contain components orthogonal to this span that cannot be represented by gradient-based updates. To make this precise at the population level, we replace the finite-dimensional training tangent span by the function space generated by the stochastic gradient kernel and characterise the unlearnable directions through the null space of the associated integral operator.

Recall the stochastic gradient kernel $\widetilde{K}_t(x,x')$ from Definition~\ref{def:stoch_kernel}. The training data $\{(x_n,y_n^*)\}_{n=1}^N$ are drawn from a joint distribution on $\mathbb{R}^p \times \mathbb{R}$, and we denote by $\mathcal P$ the marginal distribution of $x$ on $\mathbb R^p$. Let $\cK_t$ be the integral operator induced by $\widetilde{K}_t$, defined by
\begin{equation}\label{eq:integral_operator}
    [\cK_t g](x)
    = \int_{\mathbb{R}^p} \widetilde{K}_t(x,x')\,g(x')\,\mathrm{d}\mathcal{P}(x').
\end{equation}
By the spectral theorem for compact self-adjoint operators, $\cK_t$ admits an orthonormal eigenbasis $\{\phi_i\}_{i=1}^{\infty}$ of $L^2(\mathbb{R}^p, p)$ with non-negative eigenvalues $\{\lambda_i\}_{i=1}^{\infty}$, and Mercer's theorem~\cite{berlinet2011reproducing} gives $\widetilde{K}_t(x,x') = \sum_{i=1}^{\infty} \lambda_i\,\phi_i(x)\,\phi_i(x')$. We claim that the eigenspaces with $\lambda_i > 0$ span the learnable modes, while $\ker(\cK_t) = \mathrm{span}\{\phi_i : \lambda_i = 0\}$ collects the unlearnable directions. The following assumption gathers the conditions needed in the proof.

\begin{assumption}\label{assump:main}
Let $\{\bTheta_k\}_{k=0}^K$ be the SGD iterates with step size $\eta \in (0, 1 \wedge T)$ and $K = \lfloor T/\eta \rfloor$. Let $\{\bZ_t : t \in [0,T]\}$ be the continuous-time diffusion approximation of $\{\bTheta_k\}$ defined in Lemma~\ref{lem:1order}. Define the residual $\Delta_t(x) := f^*(x) - f(x,\bZ_t)$, where $x \in\mathcal{X}\subset \mathbb{R}^p$.

\smallskip
\noindent\textbf{(A1) First-order weak approximation.}
The conditions of Lemma~\ref{lem:1order} hold with $g(\cdot) = f(x,\cdot)$.

\smallskip
\noindent\textbf{(A2) It\^{o} and Fubini-Tonelli admissibility.}
Assume the input domain $\mathcal X\subset\mathbb R^p$ is compact and $x\sim \cP(x)$ is supported on $\mathcal X$.
Assume $f(\cdot,\bTheta)$ is continuous in $x$ and $C^2$ in $\bTheta$.

\smallskip
\noindent\textbf{(A3) Stationarity at time $T$ (up to $O(\eta)$).}
The dynamics are stationary up to $O(\eta)$ at time $T$:
\[
    \left.\frac{\mathrm{d}}{\mathrm{d}t}\,
    \mathbb{E}\bigl[L(\bZ_t)\bigr]\right|_{t=T} = O(\eta),
    \qquad (\eta \to 0).
\]
\end{assumption}

\begin{theorem}\label{thm:rkhs_decompose}
    Under Assumption~\ref{assump:main}, let $f^*$ be the target function and suppose the loss is the mean-squared error. Define the terminal residual $\Delta_T(x) := f^*(x) - f(x,\bZ_T)$ and the generalisation error
    \[
        \cE_T
        := \|\Delta_T\|_{L^2(\mathbb{R}^p,\,\cP)}^2
        = \int_{\mathbb{R}^p} \bigl(f^*(x) - f(x,\bZ_T)\bigr)^2\,\mathrm{d}\cP(x).
    \]
    Then, as $\eta \to 0$,
    \[
        \forall\,\varepsilon > 0, \qquad
        \mathbb{P}\,\!\bigl(
          \Delta_T \in \cN_{\varepsilon}\!\bigl(\ker(\cK_T)\bigr)
        \bigr) \to 1,
    \]
    where $\cN_{\varepsilon}(\cdot)$ denotes the $\varepsilon$-neighbourhood in $L^2(\mathbb{R}^p, \cP)$. In particular, the generalisation error $\cE_T$ is governed by the component of the residual lying in $\ker(\cK_T)$.
\end{theorem}
\begin{proof}
Denote $\Delta_t(x)=f^*(x)-f(x,\bZ_t)$, where $\bZ_t$ is the continuous process of parameters $\bTheta_t$. Let $\mathcal{K}_t:L^2(\mathcal{X},\mathcal{P}(x))\to L^2(\mathcal{X},\mathcal{P}(x))$ denote the integral operator associated with $K_t$
      \begin{align*}
          (\mathcal{K}_t f)(x)=\int \widetilde{K}_t(x,x')f(x')\mathrm{d}\mathcal{P}(x'),
      \end{align*}
where $\widetilde{K}_t(x,x')=\langle \nabla_{\bTheta} f(x,\bZ_t),\nabla_{\bTheta} f(x',\bZ_t)\rangle$.
We show that $\Delta_t(x)$ falls into the orthogonal complement space of $\mathcal{K}_t$ by proving  that $(\mathcal{K}_t \Delta_t)(x)=0,\forall x.$ Consider
\begin{align}\label{delta expectation}
    \mathbb{E}[(\Delta_t(x))^2]&=\mathbb{E}[(f^*(x))^2+(f(x,\bZ_t))^2-2f^*(x)f(x,\bZ_t)]\notag\\&=(f^*(x))^2+\mathbb{E}[(f(x,\bZ_t))^2]-2f^*(x)\mathbb{E}[(f(x,\bZ_t))],
\end{align}
we have
\begin{align}\label{delta_difference}
    \mathrm{d}(\mathbb{E}[(\Delta_t(x))^2])=\mathrm{d}(\mathbb{E}[(f(x,\bZ_t))^2])-2f^*(x)\mathrm{d}(\mathbb{E}[(f(x,\bZ_t))]).
\end{align}
 By Assumption~\ref{assump:main} (A1) and Theorem~\ref{thm:sgd_domingos}, we obtain
\begin{align*}
    \mathrm{d}f(x,\bZ_t)&=-\left(\nabla_{\bTheta} f(x,\bZ_t)\right)^\top\nabla_{\bTheta} L(\bZ_t)\mathrm{d}t+\sqrt{\eta}\left(\nabla_{\bTheta} f(x,\bZ_t)\right)^\top\bs^\frac{1}{2}_{|\mathcal{B}|}(\bZ_t)\mathrm{d}\bW_t\\
    &+\frac{\eta}{2}\operatorname{Tr}(\bs(\bZ_t)\nabla^2f(x,\bZ_t))\mathrm{d}t.
\end{align*}
By Itô’s multiplication rule,
\begin{align*}
    \mathrm{d} (f(x,\bZ_t))^2=2f(x,\bZ_t)\mathrm{d}f(x,\bZ_t)+(\mathrm{d}f(x,\bZ_t))^2,
\end{align*}
we have
\begin{align*}
    \mathrm{d} (f(x,\bZ_t))^2&=-2f(x,\bZ_t)\left(\nabla_{\bTheta} f(x,\bZ_t)\right)^\top\nabla_{\bTheta} L(\bZ_t)\mathrm{d}t\\&+2\sqrt{\eta}f(x,\bZ_t)\left(\nabla_{\bTheta} f(x,\bZ_t)\right)^\top\bs^\frac{1}{2}_{|\mathcal{B}|}(\bZ_t)\mathrm{d}\bW_t\\&+{\eta}f(x,\bZ_t)\operatorname{Tr}(\bs(\bZ_t)\nabla^2f(x,\bZ_t))\mathrm{d}t+\eta\left(\nabla_{\bTheta} f(x,\bZ_t)\right)^\top\bs_{|\mathcal{B}|}(\bZ_t) \nabla_{\bTheta} f(x,\bZ_t)\mathrm{d}t.
\end{align*}
 Integrating  $d (f(x,\bZ_t))^2$ and taking the expectation we have, by Assumption~\ref{assump:main} (A2),
\begin{align*}
    \mathbb{E}[(f(x,\bZ_t))^2]&=(f(x,\bZ_0))^2-2\mathbb{E}\left[\int_0^t f(x,\bZ_s)\left(\nabla_{\bTheta} f(x,\bZ_s)\right)^\top\nabla_{\bTheta} L(\bZ_s)\mathrm{d}s\right]\\&+{\eta}\mathbb{E}\left[\int_0^t f(x,\bZ_s)\operatorname{Tr}(\bs(\bZ_s)\nabla^2f(x,\bZ_s))\mathrm{d}s\right]\\&+\eta\mathbb{E}\left[\int_0^t \left(\nabla_{\bTheta} f(x,\bZ_s)\right)^\top\bs_{|\mathcal{B}|}(\bZ_s) \nabla_{\bTheta} f(x,\bZ_s)\mathrm{d}s\right].
\end{align*}
By Assumption~\ref{assump:main} (A2), we may apply Fubini theorem to the above equation, yielding
\begin{align}\label{eq:squre_f_expec}
    \mathbb{E}[(f(x,\bZ_t))^2]&=(f(x,\bZ_0))^2-2\int_0^t\mathbb{E}\left[ f(x,\bZ_s)\left(\nabla_{\bTheta} f(x,\bZ_s)\right)^\top\nabla_{\bTheta} L(\bZ_s)\right]\mathrm{d}s\notag\\&+{\eta}\int_0^t \mathbb{E}\left[f(x,\bZ_s)\operatorname{Tr}(\bs(\bZ_s)\nabla^2f(x,\bZ_s))\right]\mathrm{d}s\notag\\&+\eta\int_0^t\mathbb{E}\left[ \left(\nabla_{\bTheta} f(x,\bZ_s)\right)^\top\bs_{|\mathcal{B}|}(\bZ_s) \nabla_{\bTheta} f(x,\bZ_s)\right]\mathrm{d}s.
\end{align}
From Eq.(\ref{eq:sgd_result}) and Eq.(\ref{eq:squre_f_expec}), the Eq.(\ref{delta_difference}) becomes
\begin{align*}
 \mathrm{d}(\mathbb{E}[(\Delta_t(x))^2])&=\mathrm{d}(\mathbb{E}[(f(x,\bZ_t))^2])-2f^*(x)\mathrm{d}(\mathbb{E}[(f(x,\bZ_t))])\\
&=-2\mathbb{E}\left[f(x,\bZ_t)\left(\nabla_{\bTheta} f(x,\bZ_t)\right)^\top\nabla_{\bTheta} L(\bZ_t)\right]\mathrm{d}t+\eta\mathbb{E}\left[f(x,\bZ_t)\operatorname{Tr}(\bs(\bZ_t)\nabla^2f(x,\bZ_t))\right]\mathrm{d}t\\&+\eta\mathbb{E}\left[\left(\nabla_{\bTheta} f(x,\bZ_t)\right)^\top\bs_{|\mathcal{B}|}(\bZ_t) \nabla_{\bTheta} f(x,\bZ_t)\right]\mathrm{d}t+2 f^*(x)\mathbb{E}\left[\left(\nabla_{\bTheta} f(x,\bZ_t)\right)^\top\nabla_{\bTheta} L(\bZ_t)\right]\mathrm{d}t\\&-\eta f^*(x)\mathbb{E}\left[\operatorname{Tr}(\bs(\bZ_t)\nabla^2f(x,\bZ_t)\right]\mathrm{d}t
\end{align*}
Taking the derivative to both sides, we obtain
\begin{align*}
    \frac{\mathrm{d}\mathbb{E}[(\Delta_t(x))^2]}{\mathrm{d}t}&=-2\mathbb{E}\left[f(x,\bZ_t)\left(\nabla_{\bTheta} f(x,\bZ_t)\right)^\top\nabla_{\bTheta} L(\bZ_t)\right]+\eta\mathbb{E}\left[f(x,\bZ_t)\operatorname{Tr}(\bs(\bZ_t)\nabla^2f(x,\bZ_t))\right]\\&+\eta\mathbb{E}\left[\left(\nabla_{\bTheta} f(x,\bZ_t)\right)^\top\bs_{|\mathcal{B}|}(\bZ_t) \nabla_{\bTheta} f(x,\bZ_t)\right]+2 f^*(x)\mathbb{E}\left[\left(\nabla_{\bTheta} f(x,\bZ_t)\right)^\top\nabla_{\bTheta} L(\bZ_t)\right]\\&-\eta f^*(x)\mathbb{E}\left[\operatorname{Tr}(\bs(\bZ_t)\nabla^2f(x,\bZ_t)\right]\\
    &=-2\mathbb{E}\left[(f^*(x)-f(x,\bZ_t))\left(\nabla_{\bTheta} f(x,\bZ_t)\right)^\top\nabla_{\bTheta} L(\bZ_t)\right]+O(\eta)\\
    &=4\mathbb{E}\left[\mathbb{E}_{x'\sim \mathcal{P}(x')}\left[\Delta_t(x) \left(\nabla_{\bTheta} f(x,\bZ_t)\right)^\top \nabla_{\bTheta} f(x',\bZ_t) \Delta_t(x')\right]\right]+O(\eta).
\end{align*}
The last equation comes from loss $L(\bZ_t)=\int_{\mathcal{X}} (f^*(x')-f(x',\bZ_t))^2\mathrm{d}\mathcal{P}(x')$ and chain rule
\begin{align*}
    \nabla_\bTheta L(\bZ_t)=-2\int_\mathcal{X}\nabla_\bTheta f(x',\bZ_t)(f^*(x')-f(x',\bZ)d\mathcal{P}(x')=-2\int_\mathcal{X}\nabla_\bTheta f(x',\bZ_t)\Delta_t(x')d\mathcal{P}(x').
\end{align*}
By Assumption~\ref{assump:main}(A2), we may interchange the time derivative with the $x$ integration, yielding
\begin{align*}
\frac{\mathrm{d}\mathbb{E}\left[\mathbb{E}_{x\sim \mathcal{P}(x)}[(\Delta_t(x))^2]\right]}{\mathrm{d}t}=&\mathbb{E}_{x\sim \mathcal{P}(x)}\left[\frac{\mathrm{d}\mathbb{E}[(\Delta_t(x))^2]}{\mathrm{d}t}\right]\\=&4\mathbb{E}_{x\sim \mathcal{P}(x),x'\sim \mathcal{P}(x')}\left[\mathbb{E}\left[\Delta_t(x) \left(\nabla_{\bTheta} f(x,\bZ_t)\right)^\top \nabla_{\bTheta} f(x',\bZ_t) \Delta_t(x')\right]\right]\notag\\&+O(\eta).
\end{align*}
By Assumption~\ref{assump:main}(A3) we assume that the dynamics has reached a stationary regime at time $T$,  so we obtain $\frac{\mathrm{d} }{\mathrm{d}t}\mathbb{E}\left[\int_{\mathcal{X}}(\Delta_t(x))^2d\mathcal{P}(x)\right]\bigg|_{t=T}=0$. Thus
\begin{align*}
    4\lim_{\eta\to 0}\mathbb{E}\|\mathcal{K}^{1/2}_T\Delta_T\|^2_{L^2(\mathcal{X},\mathcal{P}(x))}=0.
\end{align*}
From Markov inequality, we have for any $\varepsilon>0$,
\begin{align*}
\mathbb{P}\left(\|\mathcal{K}^{1/2}_T\Delta_T\|_{L^2(\mathcal{X},\mathcal{P}(x))}>\varepsilon\right)\leq \frac{\mathbb{E}\left[\|\mathcal{K}^{1/2}_T\Delta_T\|^2_{L^2(\mathcal{X},\mathcal{P}(x))}\right]}{\varepsilon^2}\to 0\,\,\text{as}\,\,\eta\to 0,
\end{align*}
where $\mathcal N_\varepsilon(\cdot)$ denotes the $\varepsilon$-neighborhood in $L^2(\mathcal X,\mathcal P)$.

That implies that $\Delta_T$ becomes asymptotically close to the kernel of RKHS $\mathcal{K}_T$ as $\mathcal{K}_T=\{\mathcal{K}_T^{1/2}g:g\in L^2(\mathcal{X},\mathcal{P}(x))\}$,
      i.e.
      $$\forall \varepsilon>0,\,\mathbb{P}\left(\Delta_T\in \mathcal{N}_\varepsilon(\mbox{Ker} (\mathcal{K}_T))\right)\to 1\,\,\text{as}\,\,\eta\to 0.$$
      Thus the terminal residual becomes orthogonal to the RKHS range generated by $K_T$ with probability tending to one.
\end{proof}

Theorem~\ref{thm:rkhs_decompose} states that, after convergence, the terminal residual concentrates near the null space of the kernel operator $\cK_T$. Any component of the target $f^*$ supported on the learnable modes (eigenspaces with $\lambda_i > 0$) is progressively reduced during training, while components in $\ker(\cK_T)$ persist as an irreducible residual. The generalisation error is therefore driven by the energy of $f^*$ in the null modes.

\begin{remark}[Connection to NTK and lazy training]
    In the lazy-training regime, where $\widetilde{K}_t \approx \widetilde{K}_0$ throughout optimisation~\cite{jacot2018neural,lee2019wide,arora2019exact}, the operator $\cK_t$ reduces to the classical NTK operator $\cK \equiv \cK_0$. The predictor then converges to the projection of $f^*$ onto $\mathrm{Im}(\cK)$, and any component in $\ker(\cK)$ persists as an irreducible residual. This is consistent with learning-curve analyses that quantify mode-wise learning rates in terms of the kernel eigenvalues~\cite{bordelon2020spectrum}.
\end{remark}

\section{Path kernel correction for generative models}\label{genertive_unify}
Building on the stochastic Domingos theorem~\ref{thm:sgd_domingos}, Eq.~(\ref{eq:sgd_result}) shows that gradient based learning can be written as a time ordered accumulation of functional corrections, where each correction is weighted by a path kernel evaluated along the parameter trajectory. In this section, we use the same path kernel mechanism to reinterpret diffusion models and GANs. For diffusion, the path kernel weights are localized in time and in noise level, so learning takes the form of stage-wise refinement, with corrections concentrated around nearby noise scales through the time or noise embedding. For GANs, the correction is guided by distribution level information through the discriminator, which encodes a density ratio geometry. This yields more globally coupled update trajectories, helps explain why collapse arises as a distribution level phenomenon rather than a purely local fitting issue.

\subsection{Diffusion models through path kernels: sample-guided corrections}

Diffusion models~\cite{ho2020denoising,song2019generative,song2020score} learn to reverse a forward corruption process by training a neural network $\varepsilon_\bTheta$ to predict the noise $\varepsilon$ injected at each timestep, minimizing $\mathbb{E}_{\mathbf{x}_t}\|\varepsilon_\bTheta(\mathbf{x}_t,\gamma(t))-\varepsilon\|^2$. The network takes as input the pair $(\bx_t, \gamma(t))$, where $\gamma(t)$ is a  positional encoding, e.g., the sinusoidal embedding~\cite{ho2020denoising}. Applying Theorem~\ref{thm:sgd_domingos} with $f = \varepsilon_\bTheta$ and input $(\bx, \gamma(\tau))$ yields the following.
\begin{corollary}
\label{cor:diffusion_path_kernel}
Under the assumptions of Theorem~\ref{thm:sgd_domingos}, let $\{(\bx_{t,n},\gamma(t)),\varepsilon_{t}\}_{t=1,n=1}^{T,N}$ denote the training set, where $T$ is the total number of timesteps and $N$ is the number of training samples. Let $\{\bZ_s\}$ denote the weak first-order approximation of the SGD iterates. Then, for any query $(\bx,\gamma(\tau))$, the expected noise prediction satisfies
\begin{align*}
\mathbb{E}[\varepsilon_\bTheta((\bx,\gamma(\tau)),{\bTheta_{k}})]
=&\;\varepsilon_\bTheta((\bx,\gamma(\tau)),\bTheta_{0})\\
&-\mathbb{E}\left[\frac{2}{N\cdot T}\int_0^{k\eta} \sum_{t=1}^T\sum_{n=1}^{N}\|\varepsilon_\bTheta((\bx_{t,n},\gamma(t)),\bZ_s)-\varepsilon_t\|\;\widetilde{K}_s((\bx,\gamma(\tau)),(\bx_{t,n},\gamma(t)))\, ds\right]\\
&+O(\eta),
\end{align*}
where $\widetilde{K}_s$ is the stochastic gradient kernel (Definition~\ref{def:stoch_kernel}) of the noise predictor $\varepsilon_\bTheta$.
\end{corollary}
The corollary reveals that the expected noise prediction for the query noise level $\tau$ is a kernel-weighted aggregation of training residuals across all timesteps and samples. The stochastic gradient kernel $\widetilde{K}_s((\bx,\gamma(\tau)),(\bx_{t,n},\gamma(t)))$ controls how strongly each training pair $(\bx_{t,n}, \gamma(t))$ influences the prediction at $(\bx, \gamma(\tau))$. Crucially, the kernel depends on the timestep only through the positional encoding $\gamma$: the network never sees the raw index $t$, so all temporal structure in the kernel is mediated by the design of $\gamma$. When $\tau$ is close to $t$, the encodings $\gamma(\tau)$ and $\gamma(t)$ are similar, which ensures that $\nabla_{\bTheta} \varepsilon_\bTheta((\bx,\gamma(\tau)),\bZ_s)$ and $\nabla_{\bTheta} \varepsilon_\bTheta((\bx_{t,n},\gamma(t)),\bZ_s)$ are closely aligned, producing a large kernel value. In contrast, when $\tau$ and $t$ are far apart, the encodings differ substantially and the kernel similarity is suppressed. This gives the correction a structure localized in time: the noise prediction at level $\tau$ is mainly shaped by training samples at nearby noise levels, with distant levels only weakly contributing. In this sense, diffusion training induces a stage-wise refinement, where corrections at each noise level are concentrated around that level rather than being globally coupled across the entire diffusion trajectory. This mechanism is consistent with the general principle of Domingos' theorem that gradient-aligned samples exert the strongest influence on prediction corrections.

\subsection{GANs through path kernels: distribution-guided corrections}

Generative adversarial networks (GANs)~\cite{goodfellow2020generative} learn a generator $G(\cdot,\cdot):\mathcal{Z}\times\mathbb{R}^p\to \mathcal{X}$ and a discriminator $D(\cdot):\mathcal{X}\to (0,1)$ via the minimax objective
\[
\min_{\bTheta}\max_D\Big\{
\mathbb{E}_{x\sim p_{\text{data}}}\log D(x)
+\mathbb{E}_{z\sim p(z)}\log\big(1-D(G(z,\bTheta))\big)
\Big\}.
\]
For a fixed generator, the optimal discriminator is $D^*(x)=p_{\mathrm{data}}(x)/(p_{\mathrm{data}}(x)+p_G(x))$ for $(P_{\mathrm{data}}+P_G)$-a.e.\ $x$~\cite{goodfellow2020generative}. Under this idealized optimal-discriminator setting, we can apply the stochastic Domingos representation to the generator dynamics and obtain a path-kernel aggregation formula for the expected generator output.

\begin{corollary}
\label{cor:gan_generator_sgd}
Under the assumptions of Theorem~\ref{thm:sgd_domingos}, consider a generator $G_\bTheta:=G(\cdot,\bTheta)$ trained by SGD with learning rate $\eta$ on latent samples $z'\sim p(z)$. Assume that the discriminator is fully optimized at each iteration and treating $D^*$ as fixed when differentiating with respect to $\bTheta$. Let $\{\bZ_t\}$ denote the weak first-order approximation of the SGD iterates. Then, when $\nabla_y \log p_{\text{data}}$ and $\nabla_y \log p_{G_t}$ are well-defined, for any query $z$, the expected generator output satisfies
\begin{align}
&\mathbb E\!\left[G_{\bTheta_k}(z)\right]
=\;
G_{\bTheta_0}(z)
-
\mathbb E\!\left[
\int_{0}^{k\eta}\mathbb E_{z'\sim p(z)}
\Bigl[
\frac{\partial}{\partial G}\,\ell\,\!\bigl(D^*(G_{\bZ_t}(z'))\bigr)\,
\widetilde{K}_t(z,z')
\Bigr]\,\mathrm{d}t
\right]
+O(\eta)\notag\\
=&\;G_{\bTheta_0}(z)
+
\mathbb E\!\left[
\int_{0}^{k\eta}\mathbb E_{z'\sim p(z)}
\Bigl[
D^\ast\!\bigl(G_{\bZ_t}(z')\bigr)\,
\bigl(\nabla_y\log p_{\mathrm{data}}(y)-\nabla_y\log p_{G_t}(y)\bigr)\big|_{y=G_{\bZ_t}(z')}\,
\widetilde{K}_t(z,z')
\Bigr]\,\mathrm{d}t
\right]\notag\\
&+O(\eta),
\label{eq:gan_sgd_cor_score}
\end{align}
where $\widetilde{K}_t$ is the stochastic gradient kernel (Definition~\ref{def:stoch_kernel}) of the generator $G_\bTheta$, and $\ell(D(G_\bZ (z)))=\log (1-D(G_\bZ(z)))$.
\end{corollary}

The score difference $(\nabla_y \log p_{\text{data}}(y)-\nabla_y \log p_{G_t}(y))|_{y=G_{\bTheta_t}(z')}$ admits a natural attraction--repulsion decomposition. The first component $\nabla_y\log p_{\mathrm{data}}(y)$ pushes a generated point $y$ towards directions of higher real-data likelihood, while the second component $\nabla_y\log p_{G_t}(y)$ acts as a repulsion away from regions where the generator distribution is already concentrated. Together, these define a vector field on the data space that encodes where the probability mass should flow to reduce the mismatch between $p_{\mathrm{data}}$ and $p_{G_t}$. However, this vector field is not applied directly to the generated samples; it is evaluated along the generator trajectory $y = G_{\bZ_t}(z')$ and then transported through the parameter space to update the output at a new latents $z$ via the stochastic gradient kernel $\widetilde{K}_t(z,z')$. The kernel thus determines which latents $z$ are most affected by the correction signal at $z'$: latents whose generator Jacobians are well-aligned receive similar updates, while misaligned latents are largely unaffected. In this sense, the discriminator shapes the direction of correction, while the stochastic gradient kernel controls its spatial propagation across the latent space.

This contrast with the diffusion setting is instructive: diffusion corrections are localized by the timestep embedding to nearby noise levels, whereas GAN corrections are globally coupled through the discriminator's density-ratio geometry, which helps explain why GANs are more prone to instabilities such as saturation, oscillation, and mode preference.

\begin{remark}[Mode collapse through the path-kernel lens]
The Domingos view offers a mechanism for understanding mode collapse~\cite{che2016mode}: dominant trajectories are amplified by the stochastic gradient kernel. If, during training, a subset of latents $z'$ accounts for the main contribution to the kernel-weighted integral, then the update of the shared parameter propagates this contribution broadly across $z$ via $\widetilde{K}_t(z,z')$. As a result, many different test latents receive updates aligned with the directions favored by that cluster, so their generated outputs move coherently toward the same region in the data space. In particular, once the training dynamics has already collapsed around a few modes in $\mathcal X$, new test latents $z$ inherit this preference and are drawn toward the same collapsed locations.

However, the Domingos formula alone does not determine why such domination or coherence arises; it just makes explicit an amplification effect. In GANs, the structure of $\frac{\partial}{\partial G}\ell(D(G_{\bTheta_t}(z'))$ is shaped by the adversarial objective and discriminator dynamics (e.g. saturation and support mismatch under minimax training, discriminator overfitting, or nonconvex optimization effects). Thus, while the stochastic gradient kernel provides an interpretable mechanism for how mode preference can spread from a subset of trajectories to many latents, explaining the onset of mode collapse requires further analysis~\cite{arjovsky2017towards}. GAN-specific vanishing or uninformative gradients and weak penalties for missing modes explain the origin of mode preference, while the Domingos view explains its amplification and spread across latents $z$.
\end{remark}

\subsection{Discussion}
The Domingos representations above not only unify diffusion models and GANs under a common path kernel mechanism, but also clarify why their training behaviors can exhibit markedly different empirical properties. In both paradigms, the final output equals an initialization prior plus a time integrated accumulation of kernel gated corrections, and the dependence on a test query enters exclusively through a path kernel measuring Jacobian alignment in parameter space. The key difference lies in the driving weight attached to each historical training state.

In diffusion, the weight is a supervised residual $2\|\varepsilon((\bx_{t,n},\gamma(t)),\bZ_s)-\varepsilon_t\|$ indexed by a time-labeled state, the time embedding $\gamma(t)$ makes the kernel sharply prefer interactions between nearby noise levels, so the correction is forced to be local in the noise coordinate. This built-in locality yields a natural multi-stage influence allocation across time steps: early noisy states mainly borrow from early noisy states, later denoising states mainly borrow from later denoising states. The states at a given noise level are influenced primarily by nearby levels, yielding incremental refinement along the diffusion trajectory rather than globally coupled updates.

In GANs, the weight is instead an output space vector field induced by the discriminator and evaluated along the generator trajectory $y=G_{\bTheta_t}(z')$. The score difference decomposes the update into an attraction towards regions of high real density and a repulsion away from regions where the generator density is already high. So, the update is explicitly guided by the distribution  under the idealized optimal discriminator and smooth density assumptions.

\section{Numerical Experiments}\label{numerical}

\begin{figure}[ht]
\centering 
\subfigure[]{
\label{visualization}
\includegraphics[width=0.32\textwidth]{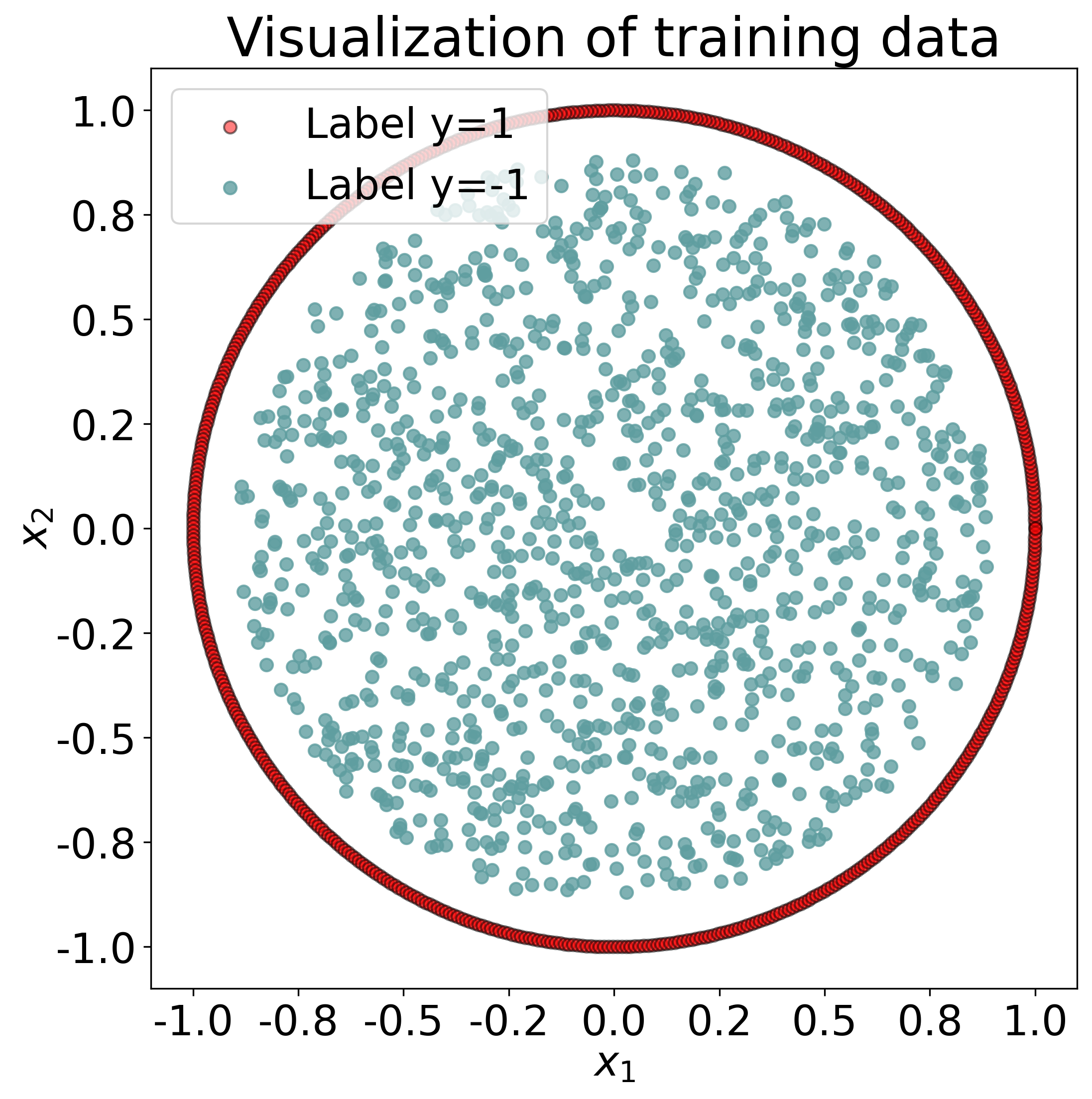}}
\subfigure[]{
\label{Classification SGD}
\includegraphics[width=0.32\textwidth]{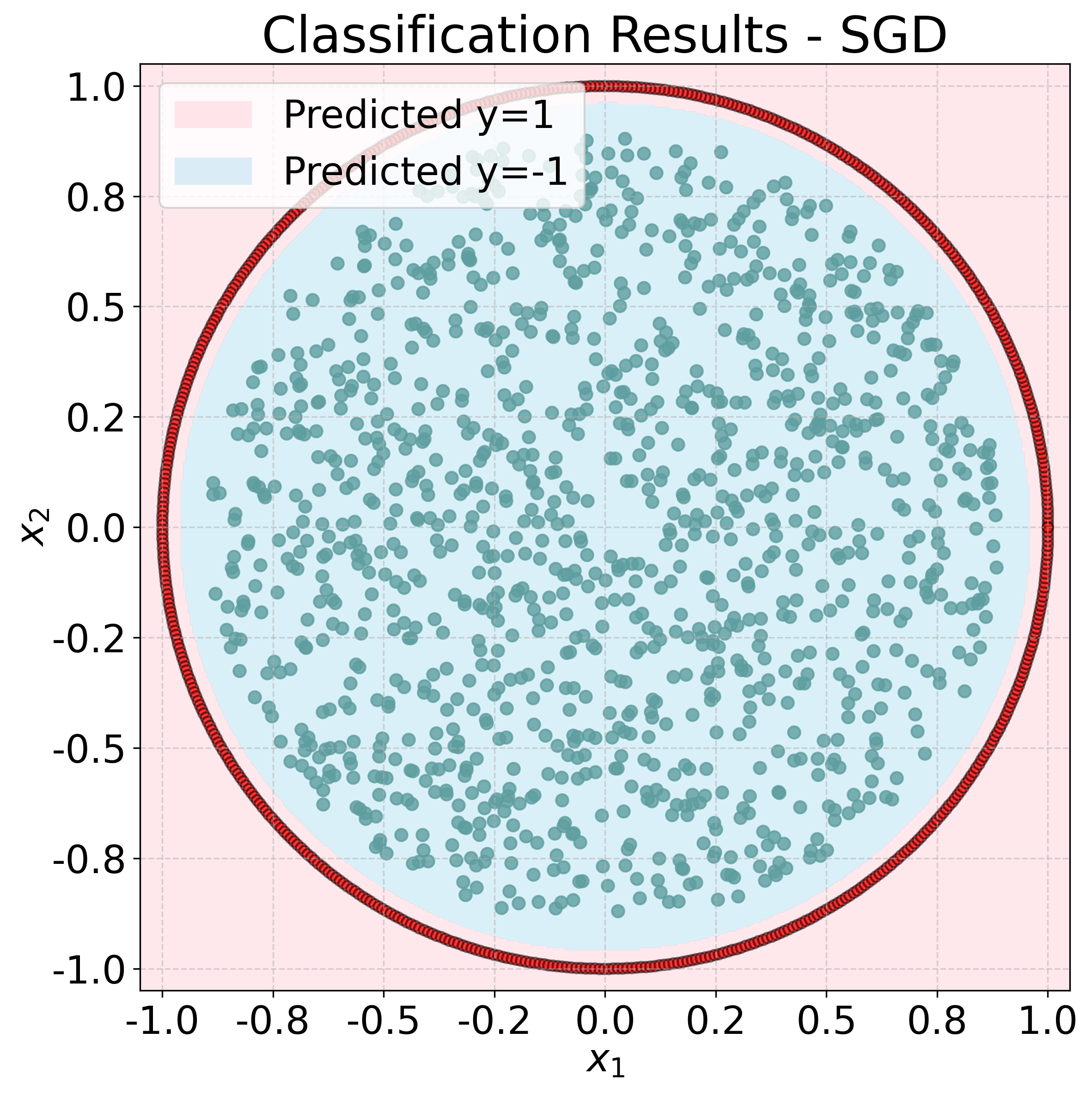}}
\subfigure[]{
\label{Nearest Euclidean}
\includegraphics[width=0.32\textwidth]{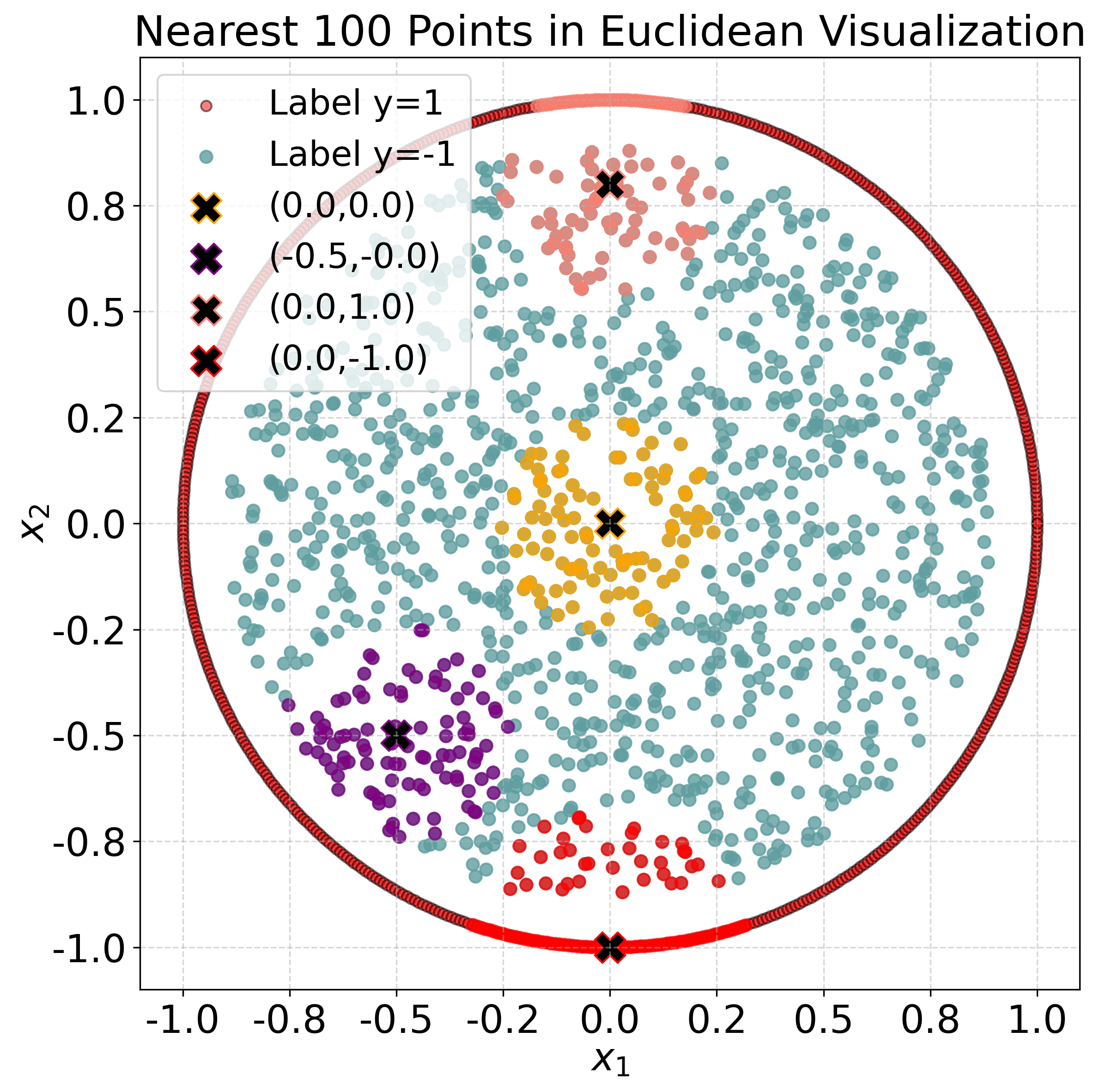}}

\subfigure[]{
\label{nearest - SGD0}
\includegraphics[width=0.32\textwidth]{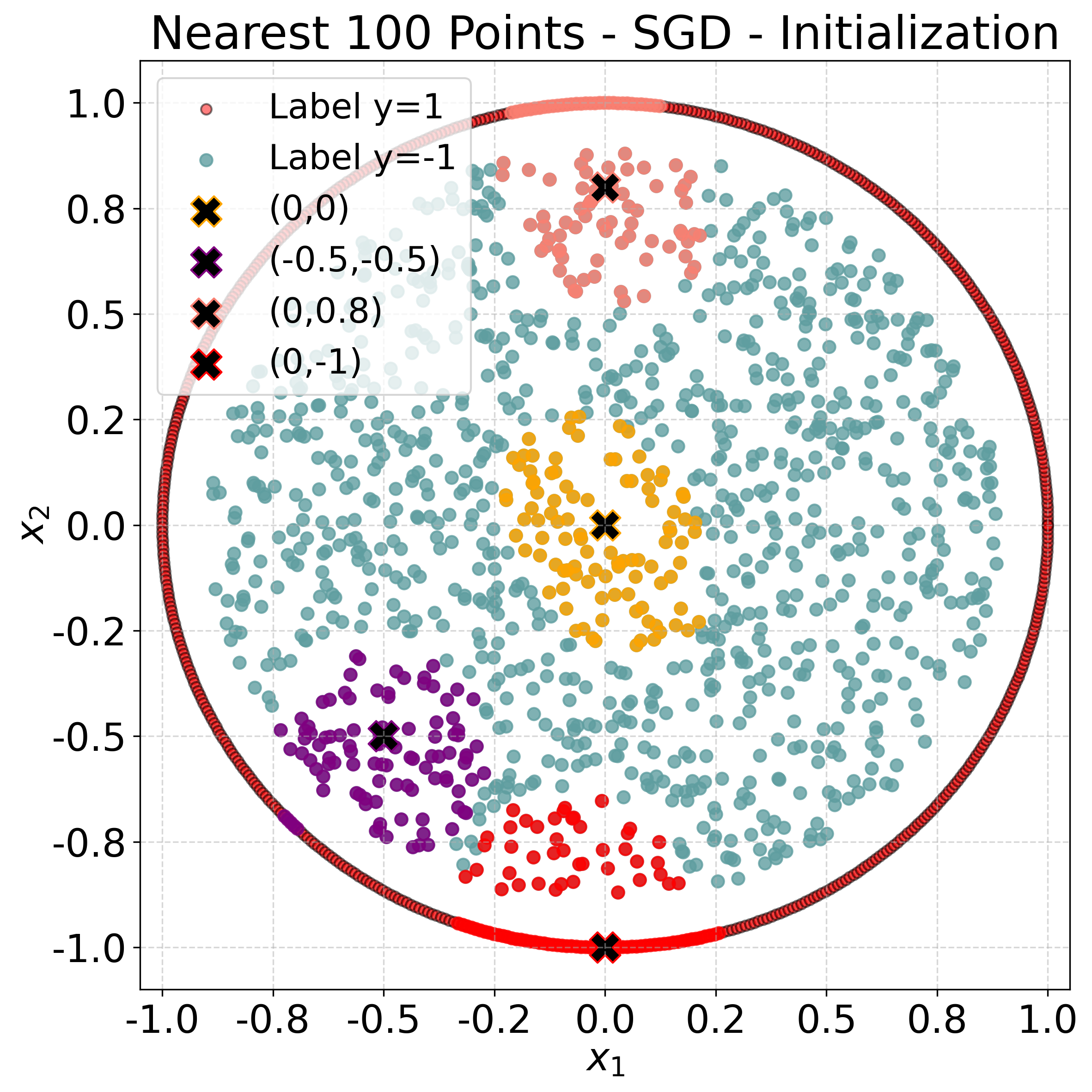}}
\subfigure[]{
\label{nearest - SGD10}
\includegraphics[width=0.32\textwidth]{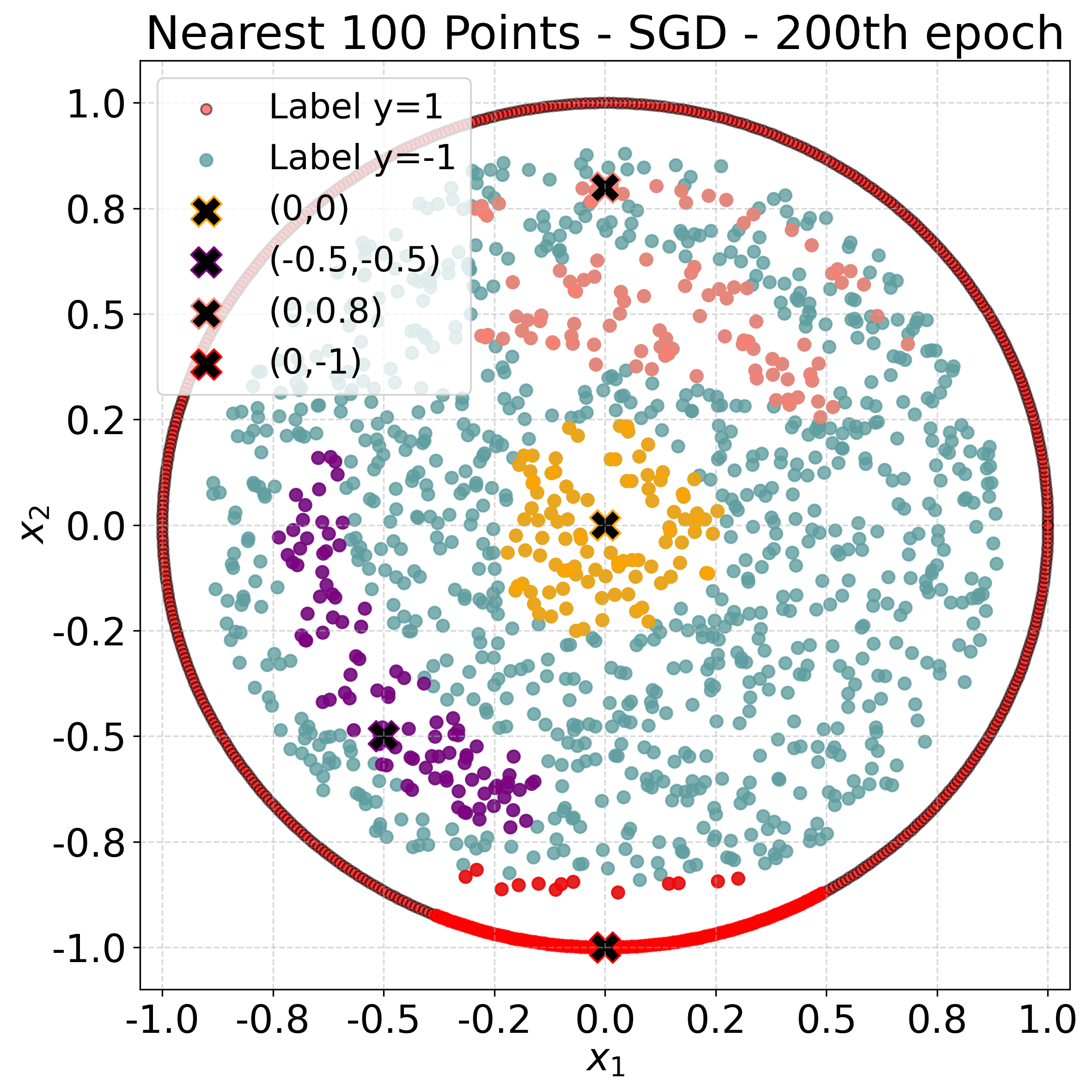}}
\subfigure[]{
\label{nearest - SGD}
\includegraphics[width=0.32\textwidth]{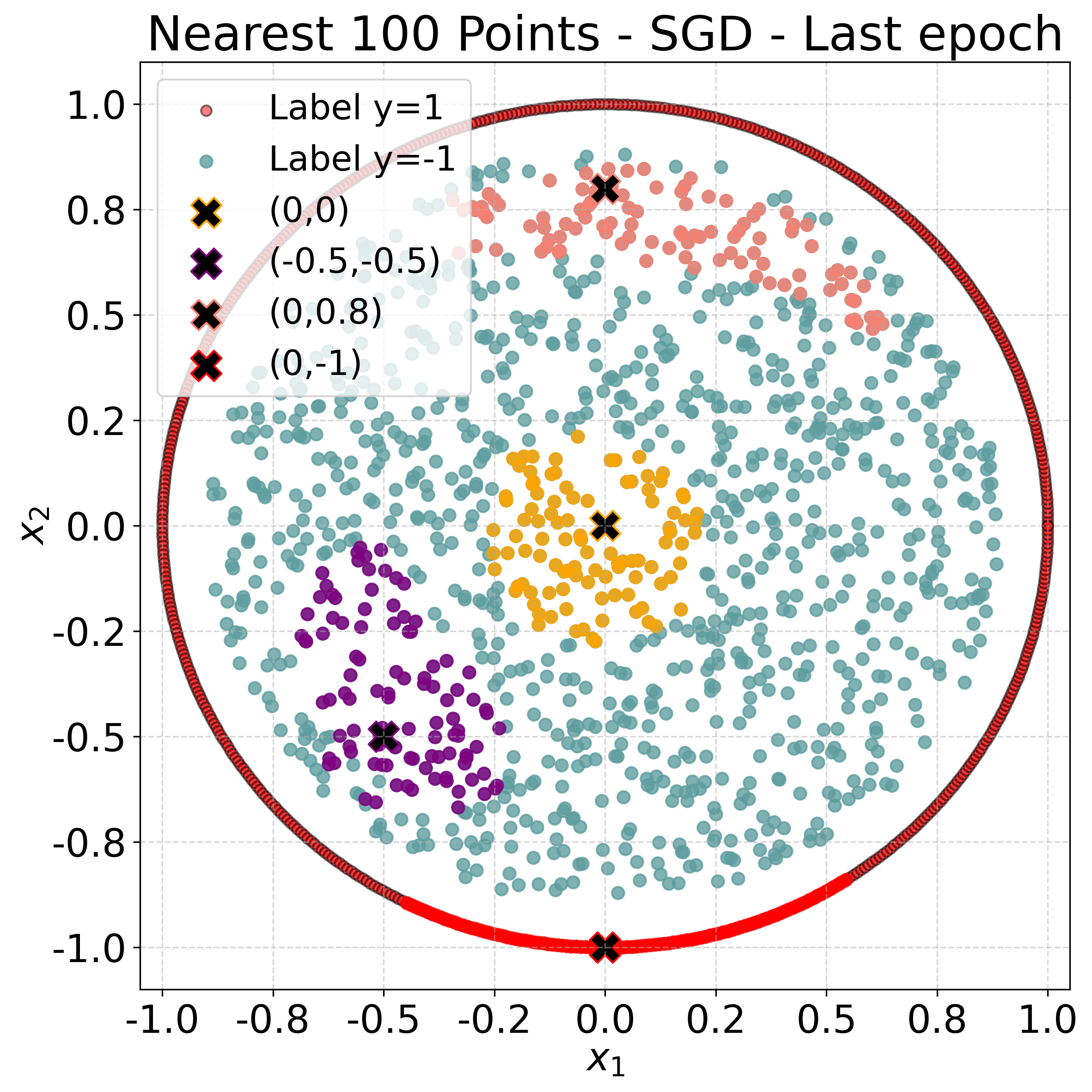}}
\caption{Evolving neighborhoods around anchor points (marked as "x"). (a) Binary dataset with class $y=1$ on the circle $x_1^2+x_2^2=1$ and class $y=-1$ inside the disk $x_1^2+x_2^2\le 0.8$. (b) SGD decision regions after training. In (c–f), black crosses denote fixed anchor inputs (coordinates listed in the legend). For each anchor, we highlight its $100$ nearest training points under the specified notion of similarity. (c) Neighbors in Euclidean input space, which are largely class mixed. (d–f) Neighbors under the normalized gradient kernel $\widehat K_t$ at initialization, the 200th epoch, and the last epoch. Training progressively sharpens these kernel neighborhoods toward label-homogeneous sets.}
\label{preiction and SGD nearest}
\end{figure}

In this section, we empirically validate the mechanisms implied by our stochastic Domingos representations (Theorems~\ref{thm:sgd_domingos}--\ref{thm:adam_domingos}) and the path-kernel viewpoint. Our theory shows that training reshapes the gradient kernel and tangent-feature geometry in a task-dependent manner, which in turn governs interpolation behavior, feature-space separability, and the emergence of extrapolation and generalization failures. 

\subsection{Evolution of gradient kernel during training}\label{sec:gradient_sim}
The stochastic Domingos theorems (Theorems~\ref{thm:sgd_domingos}--\ref{thm:adam_domingos}) show that the gradient kernel $K_t(x,x_n)$ determines how strongly each training point $x_n$ influences the prediction on a test input $x$. In this subsection, we visualise how this kernel evolves during training in several settings, using the normalized gradient kernel $\widehat{K}_t(x_i,x_j) := K_t(x_i,x_j)/\sqrt{K_t(x_i,x_i)\,K_t(x_j,x_j)}$.

We first consider a binary classification task: red points on the circle $x_1^2+x_2^2=1$ versus blue points inside the disk $x_1^2+x_2^2\le 0.8$ (Figure~\ref{visualization}). Figure~\ref{Classification SGD} reports the decision regions of an SGD-trained classifier. Figure~\ref{Nearest Euclidean} marks, for several anchor points, their 100 Euclidean nearest neighbors-which mix both classes, reflecting that geometric proximity in input space is not aligned with the label structure. In contrast, the 100 nearest neighbors under the normalized gradient kernel $\widehat K_t$ evolve markedly during training (second row in Figure~\ref{preiction and SGD nearest}): the neighborhoods become progressively homogeneous in label and most neighbors of each anchor end up sharing its label. This training induced sharpening indicates that $\widehat K_t$ captures task relevant similarity in the learned tangent feature geometry, beyond what is visible in Euclidean distance.

\begin{figure}[ht]
\centering 
\subfigure[]{
\label{11}
\includegraphics[width=0.23\textwidth]{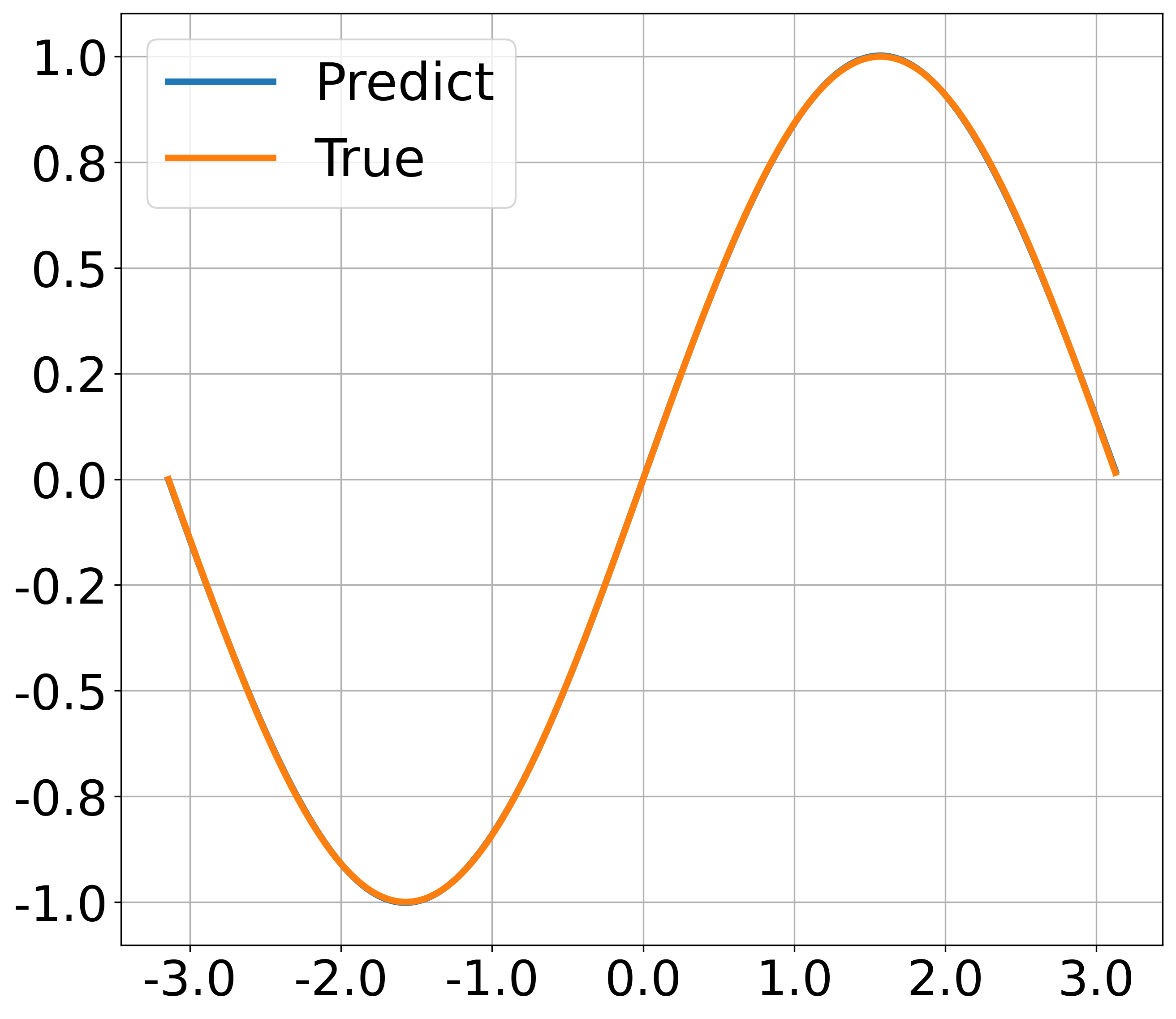}}
\subfigure[]{
\label{12}
\includegraphics[width=0.245\textwidth]{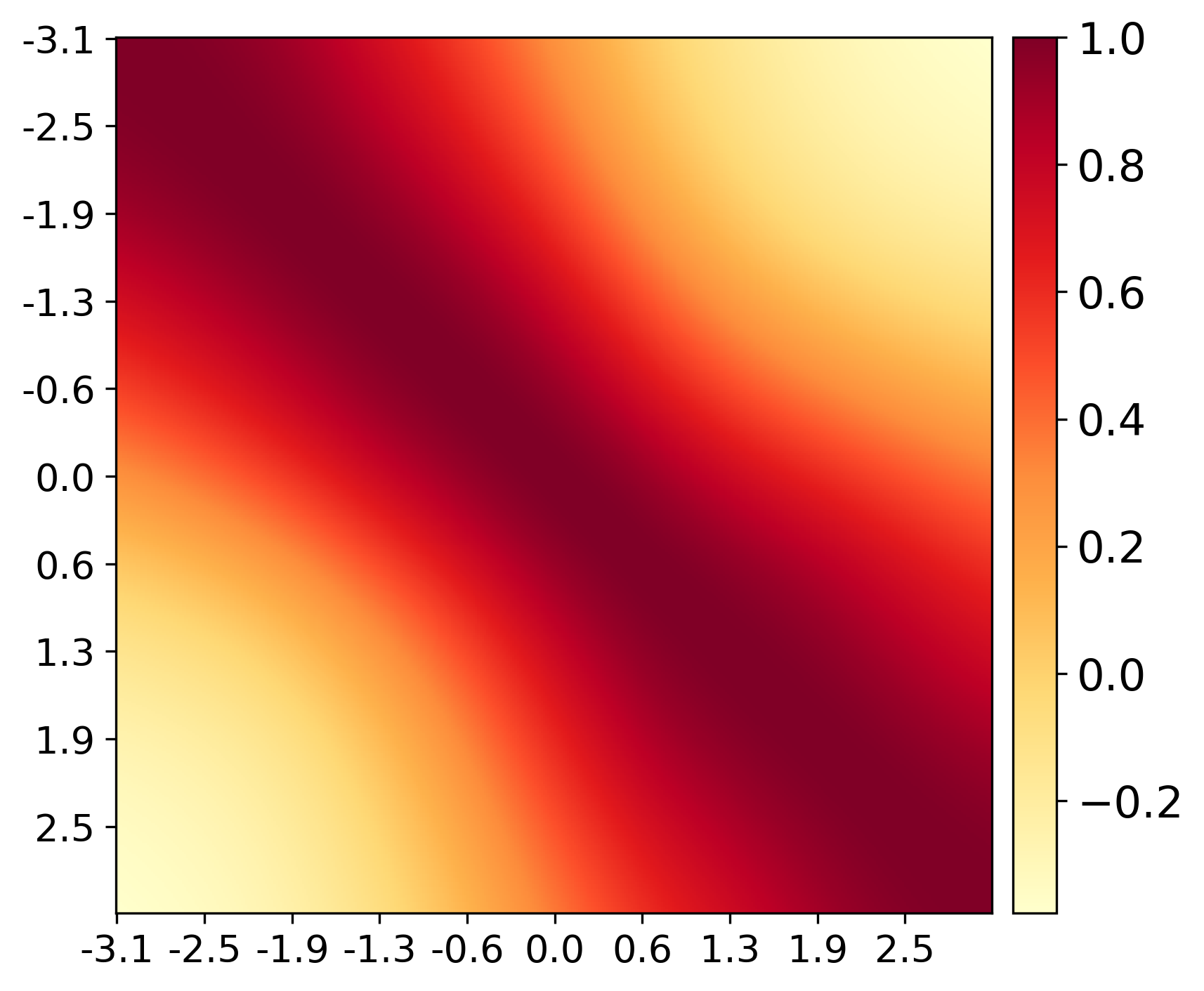}}
\subfigure[]{
\label{13}
\includegraphics[width=0.23\textwidth]{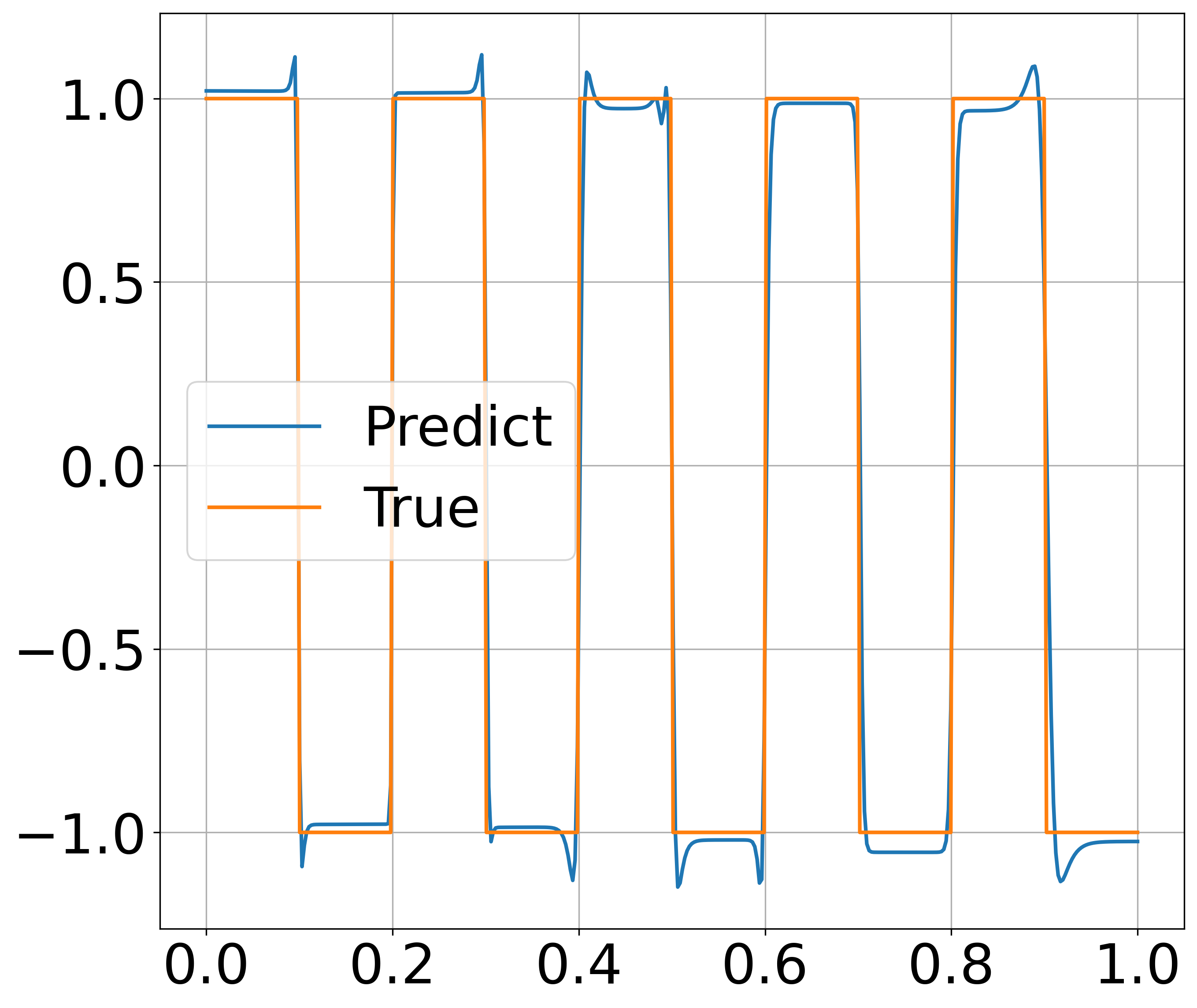}}
\subfigure[]{
\label{14}
\includegraphics[width=0.24\textwidth]{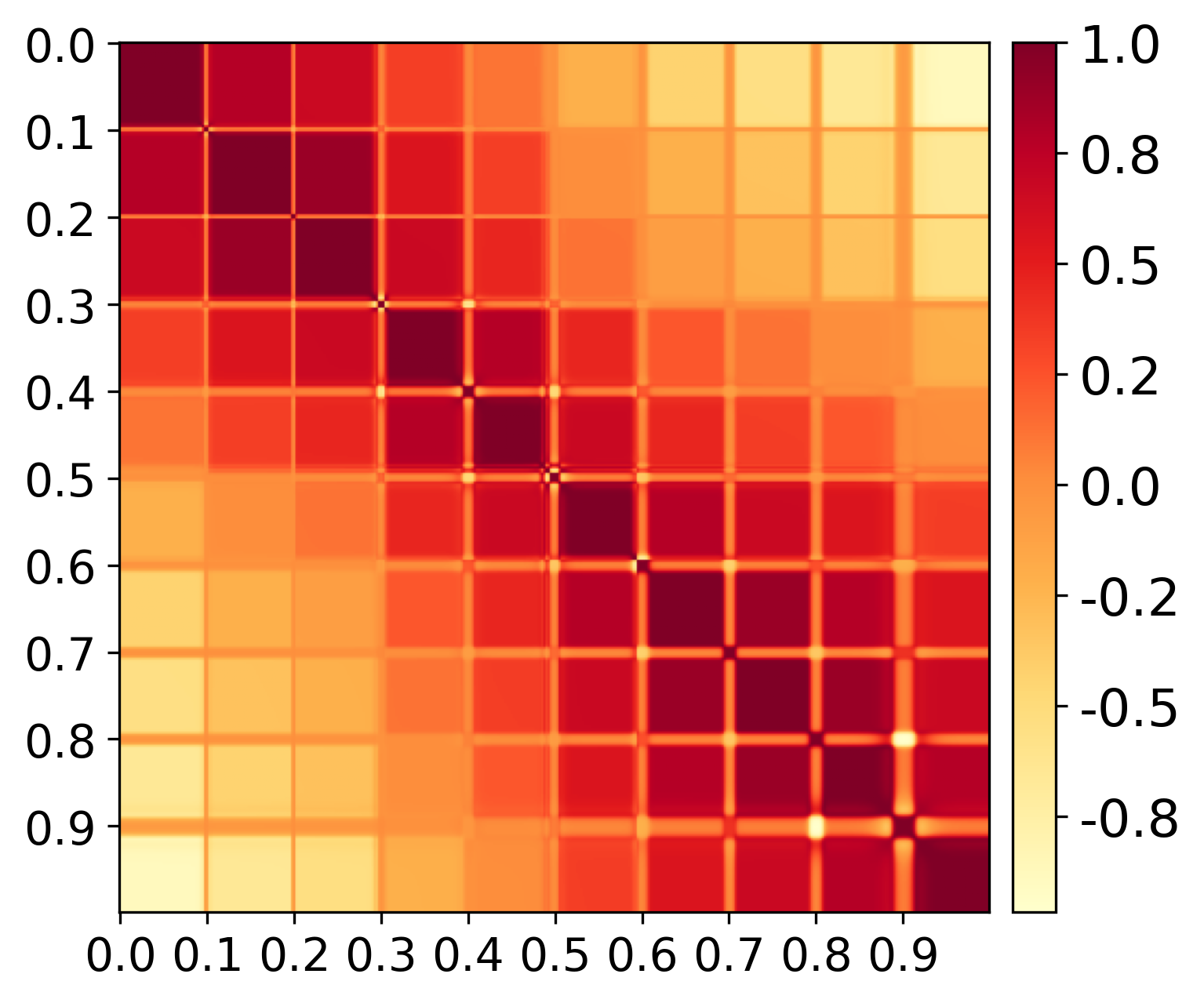}}

\subfigure[]{
\label{21}
\includegraphics[width=0.235\textwidth]{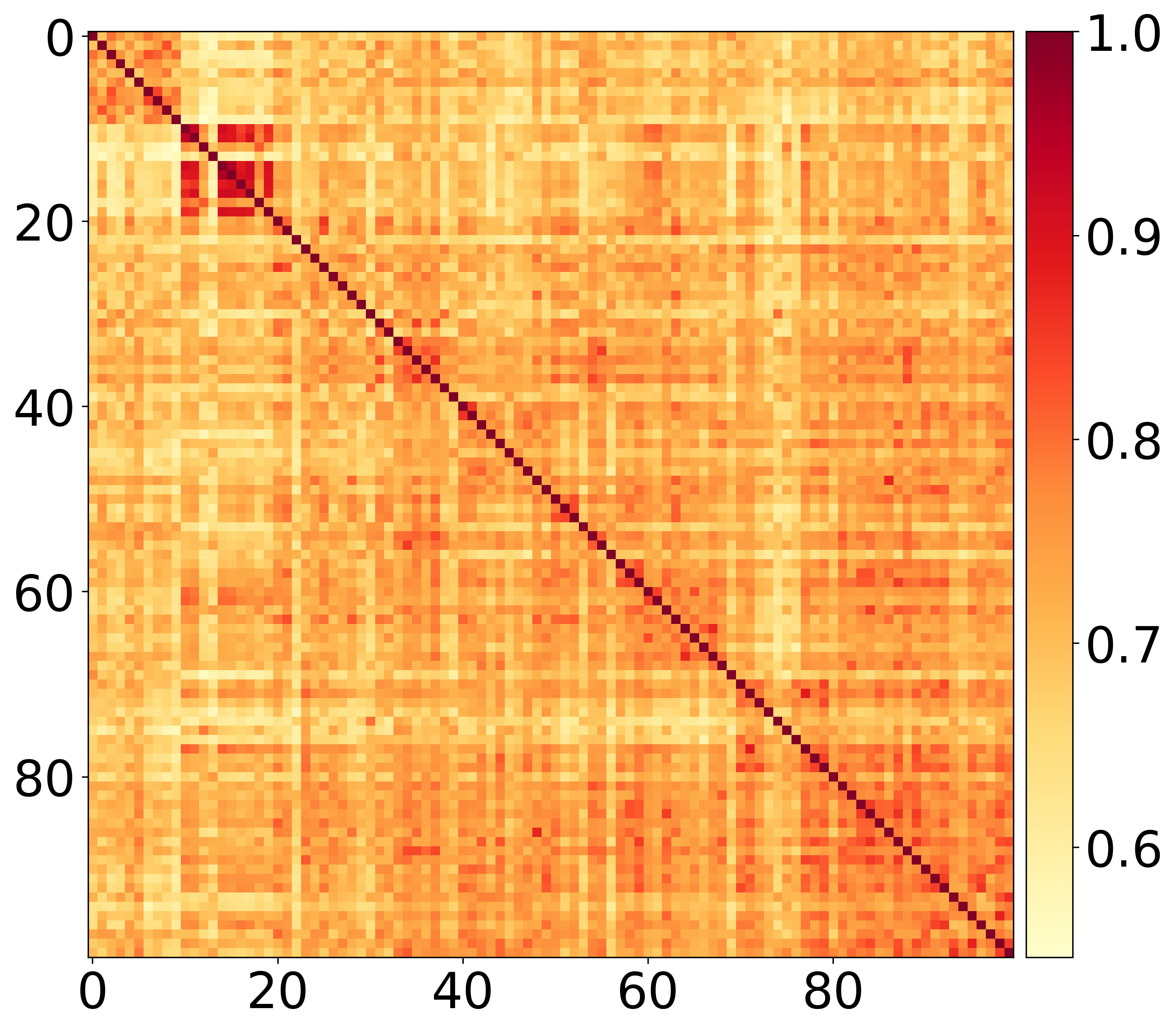}}
\subfigure[]{
\label{22}
\includegraphics[width=0.235\textwidth]{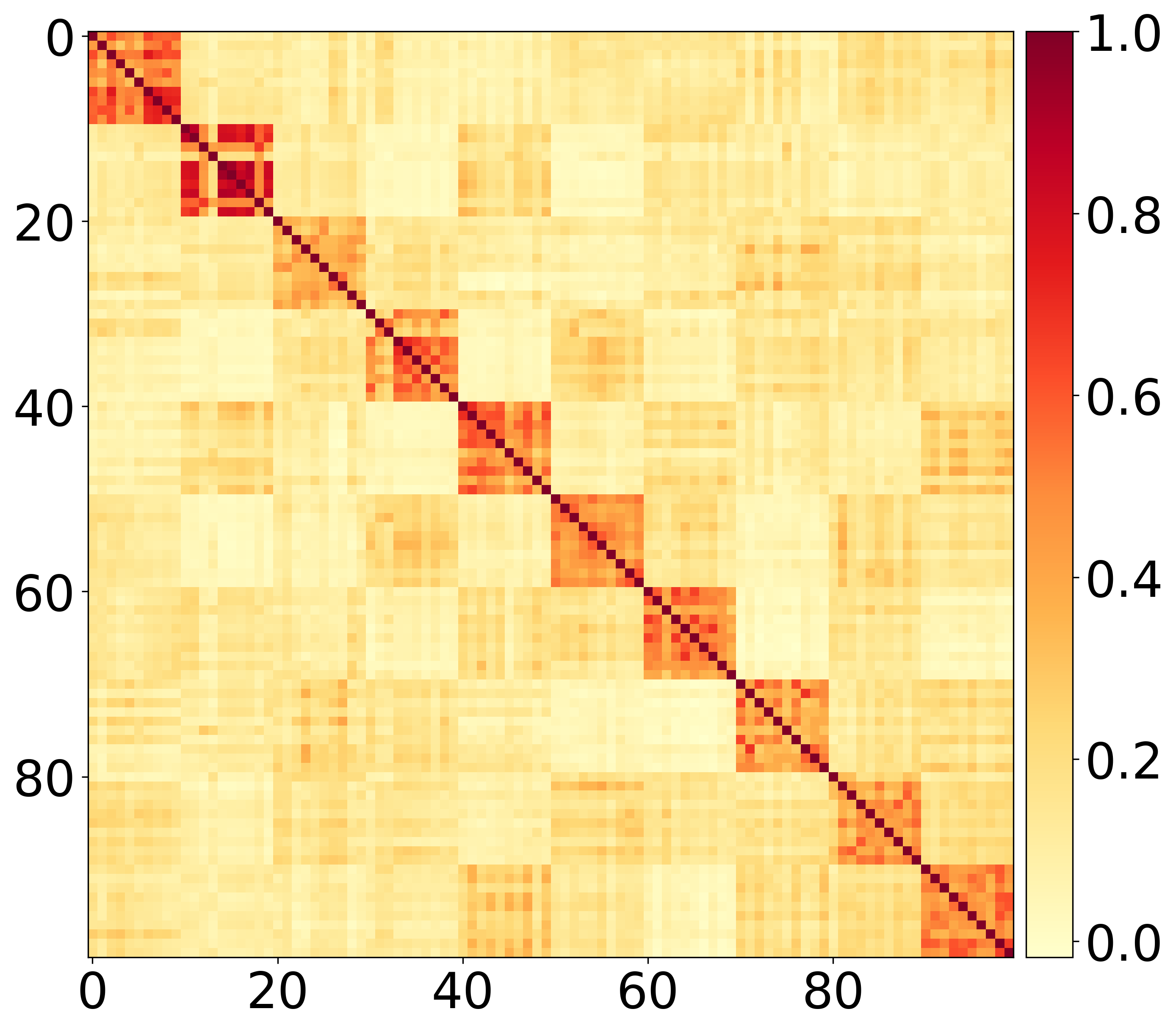}}
\subfigure[]{
\label{23}
\includegraphics[width=0.235\textwidth]{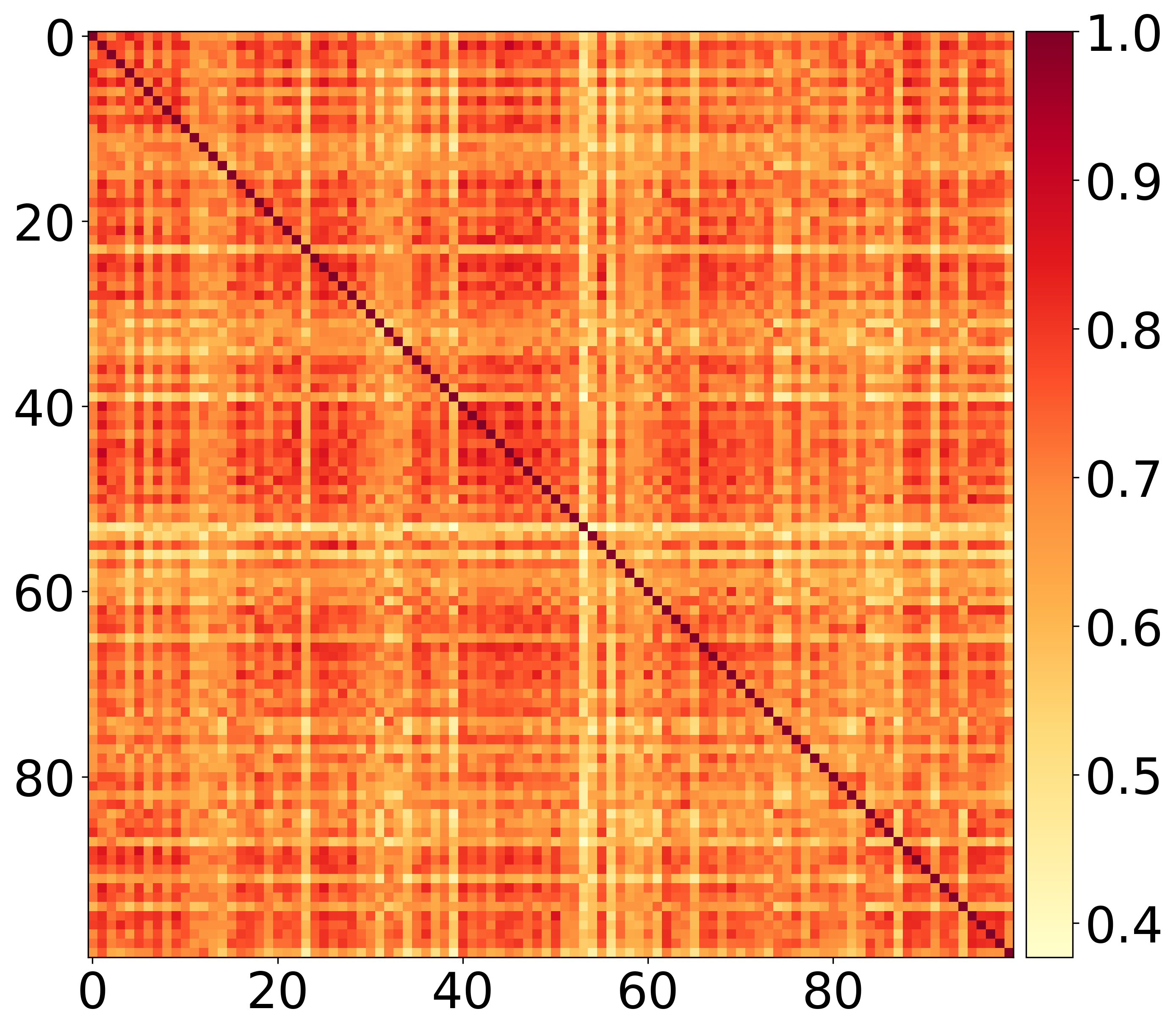}}
\subfigure[]{
\label{24}
\includegraphics[width=0.24\textwidth]{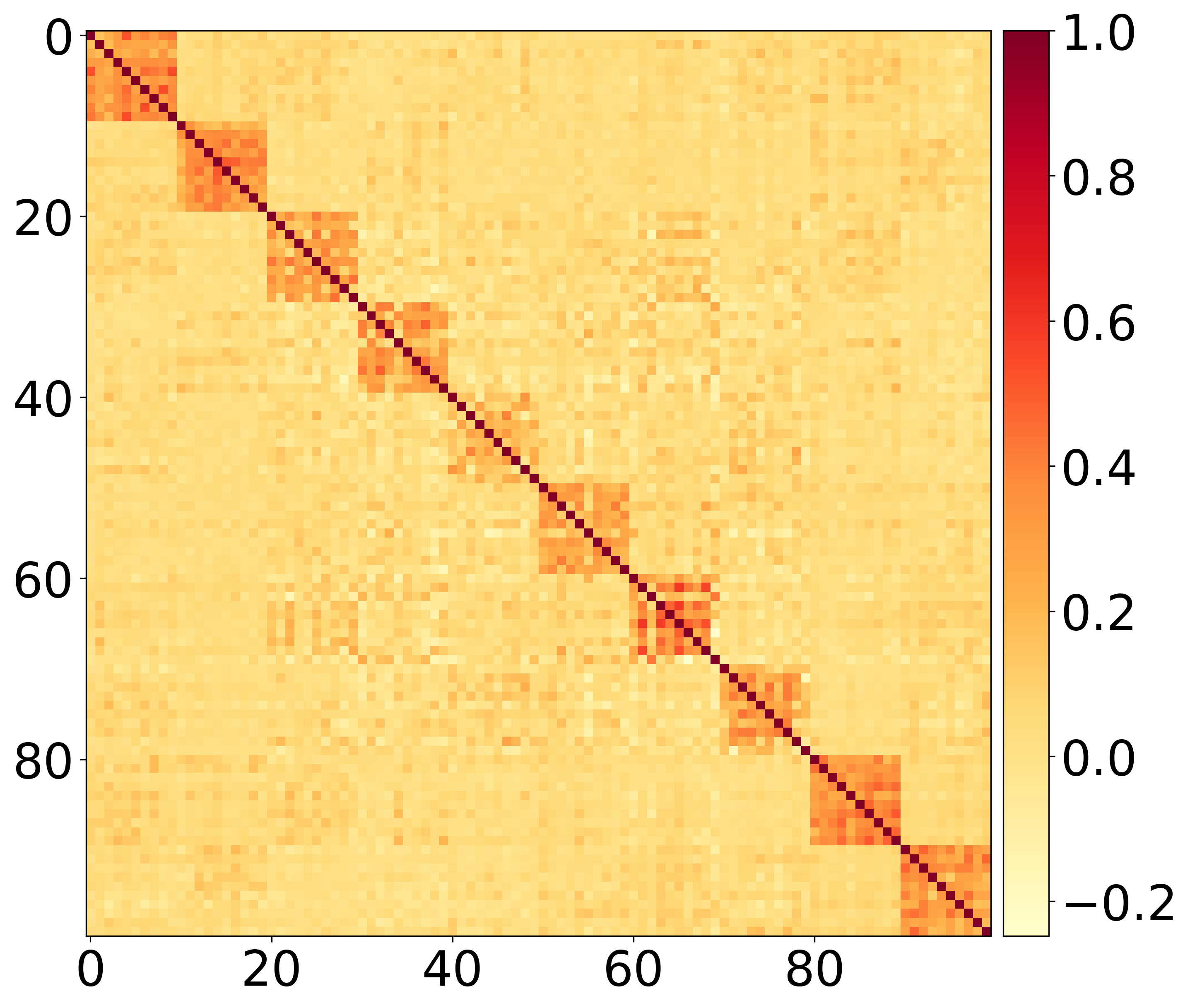}}
\caption{Prediction and gradient kernel evolution. (a) Learned prediction versus ground truth for the sine function $\sin x$. (c) Learned prediction versus ground truth for the square-wave target. (b,d) Normalized gradient kernel matrices at the final iteration (1000 steps) for $\sin x$ and the square wave, respectively. (e–f) MNIST and (g–h) CIFAR-10: normalized gradient kernel matrices at initialization (e,g) and after convergence (f,h). For MNIST/CIFAR-10, we use 10 images per class (100 samples total), ordered by class.}
\label{1}
\end{figure}

To further illustrate that proximity in the input space does not necessarily imply proximity in gradient feature space, we consider two 1D regression tasks. Figures~\ref{11},~\ref{13} shows the learned predictions for $f(x)=\sin x$ and the square wave $f(x)=\mathrm{sgn}(\sin(10\pi x))$ using a fully connected two-layer network with tanh activation, trained by Adam with a learning rate $\eta=0.01$. Figures~\ref{12},~\ref{14} displays the corresponding normalized gradient-similarity matrices in the final iteration ($t=1000$). For the smooth target $\sin x$, the nearby inputs have a high gradient similarity, while for the square wave, the points near the discontinuities at $x=0.1,0.2,\ldots,0.9$ can be close in Euclidean space yet far apart in gradient space (Figure~\ref{14}). This illustrates that the gradient kernel captures task-relevant similarity rather than input-space proximity.

Next, we examine image classification. Figures~\ref{21},~\ref{22},~\ref{23},~\ref{24} show the normalized gradient-similarity matrices at initialization and after convergence for MNIST and CIFAR-10, using 10 images per class (100 samples total, ordered by class). For MNIST, we used a compact convolutional neural network composed of three convolutional layers (with 16, 32, and 64 channels, respectively), each followed by ReLU activation and $2\times2$ max pooling. The resulting feature maps ($3\times3\times64$) were flattened and passed through two fully connected layers (128 and 10 units) with dropout regularization. For CIFAR-10, a similar three-layer convolutional backbone was used, adapted to RGB input with an increased fully connected head (256--128--10 units) to accommodate the larger input dimension ($32\times32\times3$). Both models used ReLU activations, dropout (0.3), and were trained using cross-entropy loss and stochastic gradient descent with momentum. Despite their architectural similarity, CIFAR-10 required deeper optimization (1100 epochs) due to its higher variability and input dimensionality. At initialization, the similarity matrices exhibit mixed correlations between classes. After training, a clear block-diagonal structure emerges: within-class similarities sharpen toward high values, while cross-class similarities are suppressed toward near-zero. This indicates that training aligns gradient features for samples sharing the same label and decorrelates features across labels, yielding a class-aware structure in the learned tangent representation.

\subsection{Linear separability of the tangent feature space}

The stochastic Domingos theorems (Theorems~\ref{thm:sgd_domingos}--\ref{thm:adam_domingos}) express the prediction of the neural network as a linear functional in the tangent feature space: $f(x,\bTheta_T) \approx f(x,\bTheta_0) + \sum_n a_n K_t(x,x_n)$, where the kernel $K_t(x,x_n) = \langle \nabla_{\bTheta} f(x,\bTheta_t), \nabla_{\bTheta} f(x_n,\bTheta_t)\rangle$ measures similarity in gradient space. For classification tasks, this representation can only produce correct decisions if training shapes the tangent features $\phi(x) = \nabla_{\bTheta} f(x,\bTheta_T)$ into a linearly separable configuration; otherwise, no linear weighting of kernel similarities can separate the classes. We now verify this prerequisite empirically using a support vector machine (SVM) study.

\begin{figure}[ht]
\centering 
\subfigure[Tangent feature SVM]{
\label{sgd kernel1}
\includegraphics[width=0.236\textwidth]{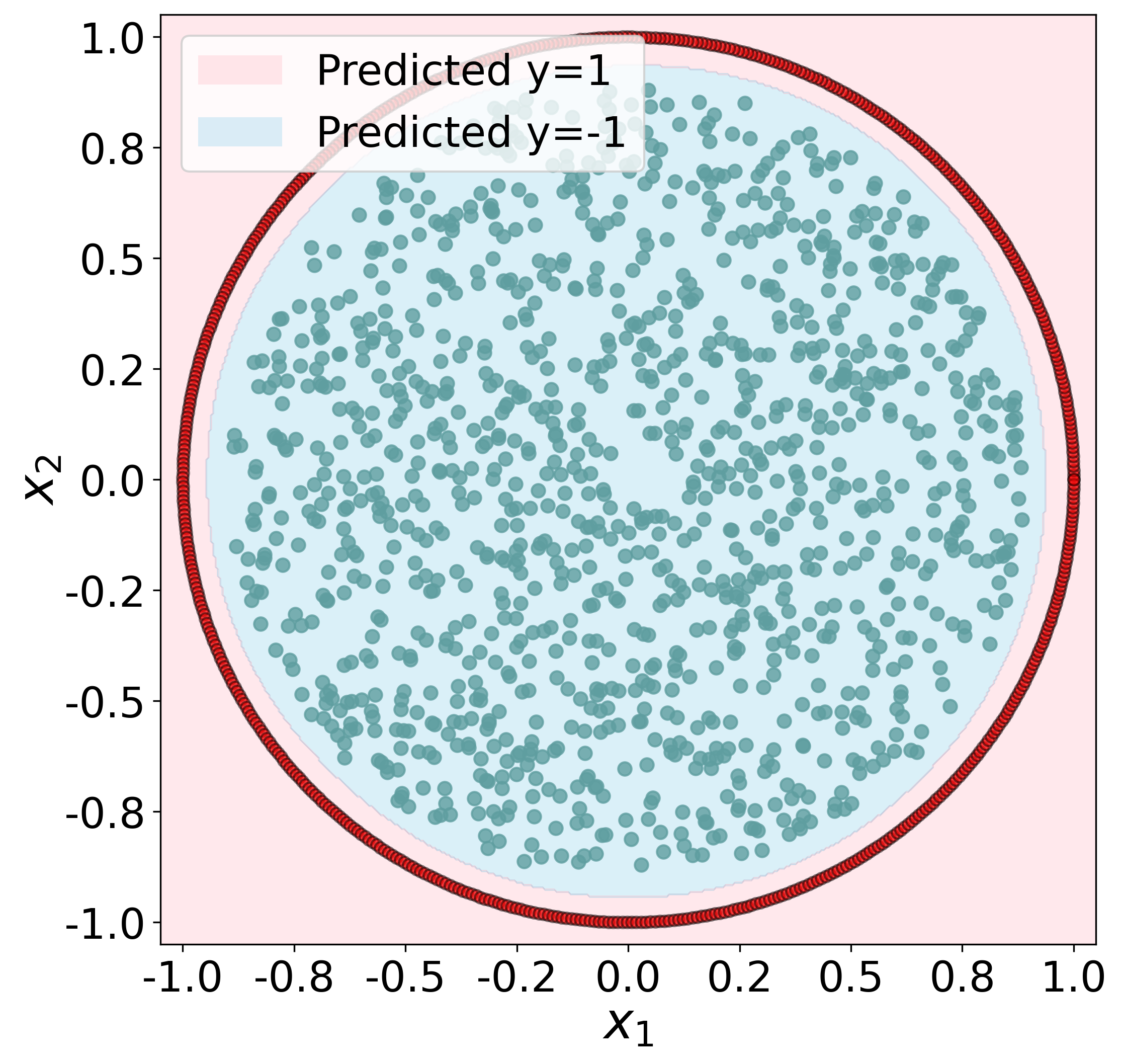}}
\subfigure[Tangent feature SVM]{
\label{sgd kernel2}
\includegraphics[width=0.238\textwidth]{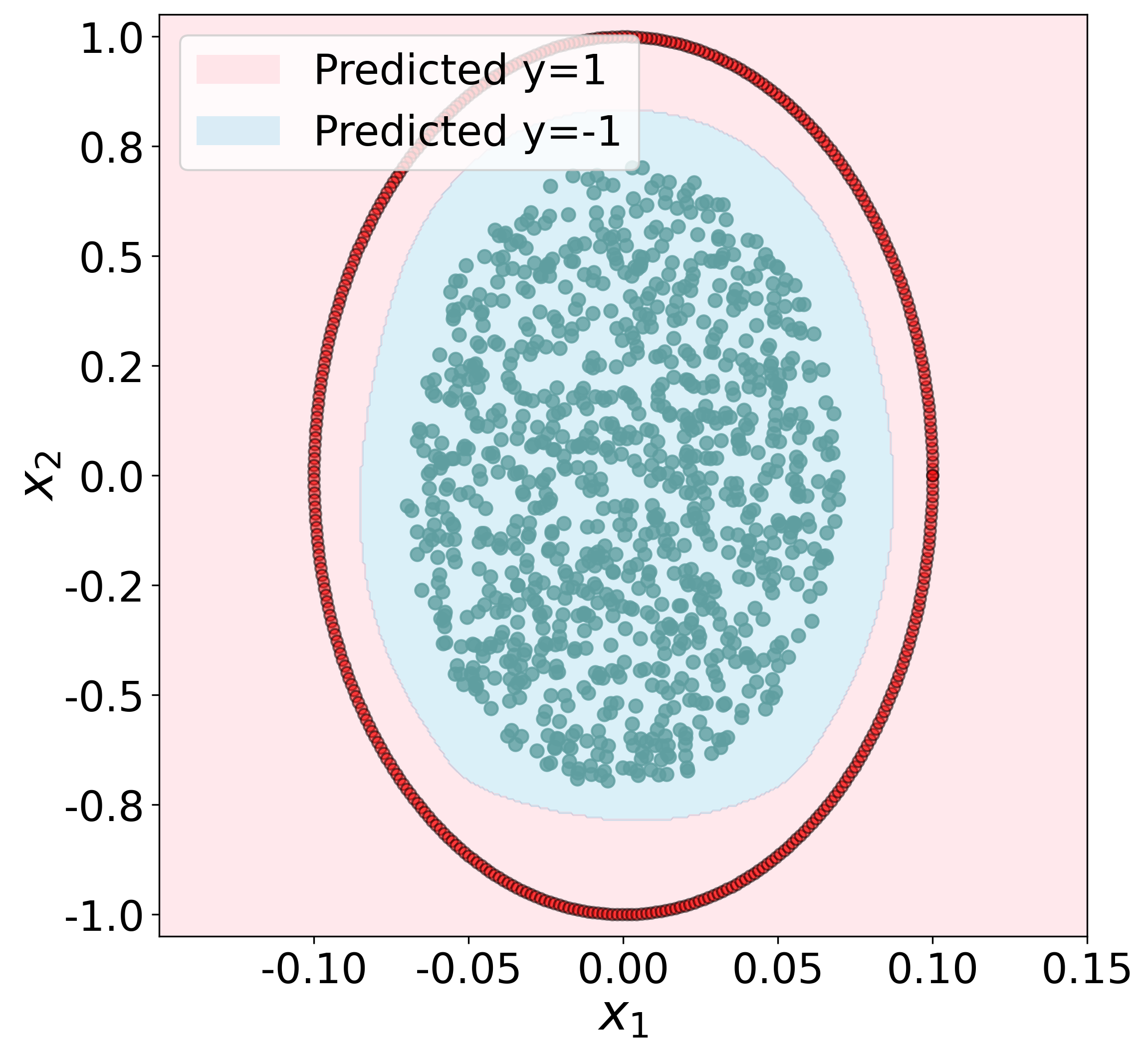}}
\subfigure[RBF SVM]{
\label{rbf}
\includegraphics[width=0.236\textwidth]{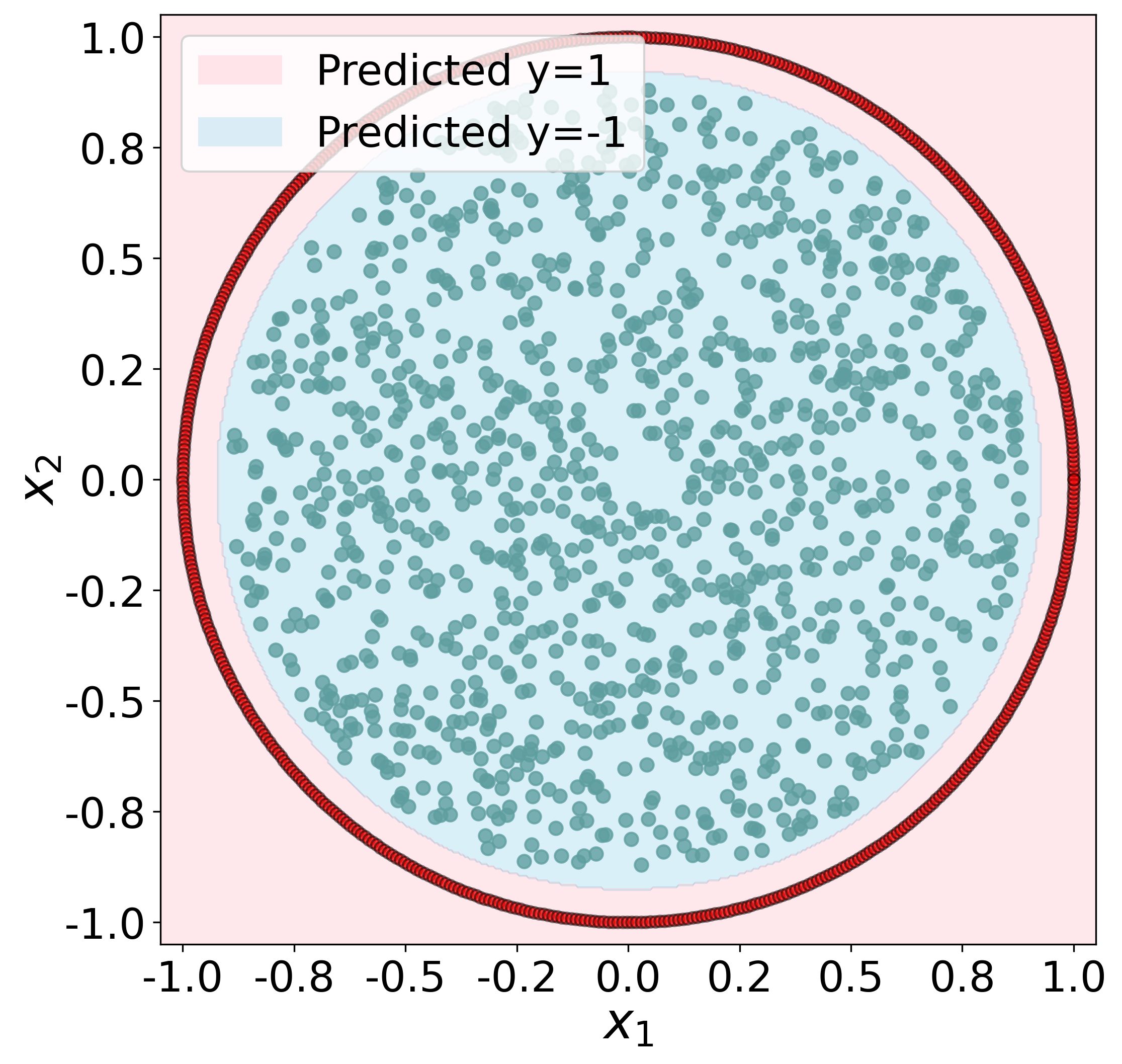}}
\subfigure[RBF SVM]{
\label{rbf elliptic}
\includegraphics[width=0.238\textwidth]{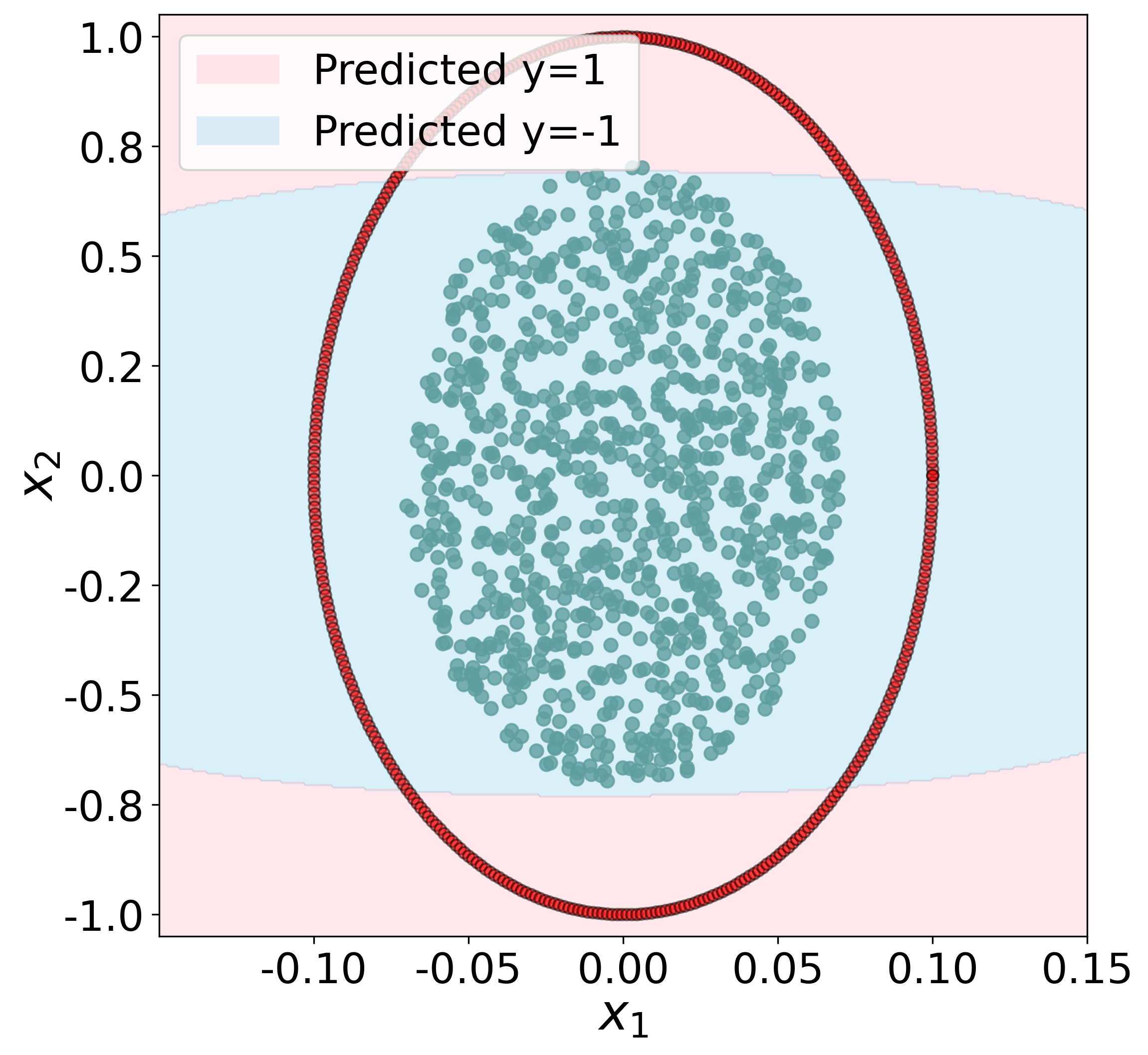}}
\caption{SVM classification in input space versus tangent-feature space. In (a,b), we fit a linear SVM using the learned tangent features $\phi(x_n)=\nabla_\bTheta f(x_n,\bTheta_T)$ as inputs (each sample $x_n$ is represented by $\phi(x_n)$, not by its coordinates). In (c,d), we fit an SVM with a Gaussian RBF kernel on the raw inputs. Panels (a,c) use the isotropic circle task $x: x_1^2+x_2^2\le 0.9$ (label $-1$) versus $x: x_1^2+x_2^2=1$ (label $+1$); panels (b,d) use the anisotropic ellipse variant $x: 100x_1^2+x_2^2\le 0.9$ (label $-1$) versus $x: 100x_1^2+x_2^2=1$ (label $+1$).}
\label{SVM}
\end{figure}

Figures~\ref{sgd kernel1},~\ref{sgd kernel2} train a linear SVM on the learned tangent representation $\phi(x_n)$, while Figures~\ref{rbf},~\ref{rbf elliptic} use an rbf SVM on the raw inputs. We consider two settings: an isotropic circle task $\{(x_1,x_2): x_1^2+x_2^2\le 0.9\}$ with label $-1$,
$\{(x_1,x_2): x_1^2+x_2^2=1\}$ with label $+1$ (Figure~\ref{sgd kernel1},~\ref{rbf}) and an anisotropic ellipse variant $\{(x_1,x_2): 100x_1^2+x_2^2\le 0.9\}$ with label $-1$, $\{(x_1,x_2): 100x_1^2+x_2^2=1\}$ with label $+1$(Figure~\ref{sgd kernel2},~\ref{rbf elliptic}). In both cases, the tangent feature SVM produces a boundary that closely matches the ground-truth circle or ellipse, while the input space rbf baseline is competitive on the isotropic case but less well aligned in the anisotropic setting. This supports that training organizes $\phi(x)$ into a task-aligned feature geometry in which the classes become linearly separable.

\subsection{Extrapolation failure in the gradient kernel view}
Building on Theorem~\ref{thm:sgd_domingos} and~\ref{thm:rkhs_decompose}, the prediction update can be written as a path kernel accumulation whose contributions lie in the training tangent span. The path kernel viewpoint predicts that extrapolation fails when test inputs probe tangent feature directions that are not populated by the training data. In this case, the learned predictor can only interpolate within the gradient feature geometry supported along the training trajectory, while out-of-distribution points fall outside the well-organized regions of the gradient kernel and receive weak or ambiguous corrections. 

\begin{figure}[ht]
    \centering
    \includegraphics[width=1.01\linewidth]{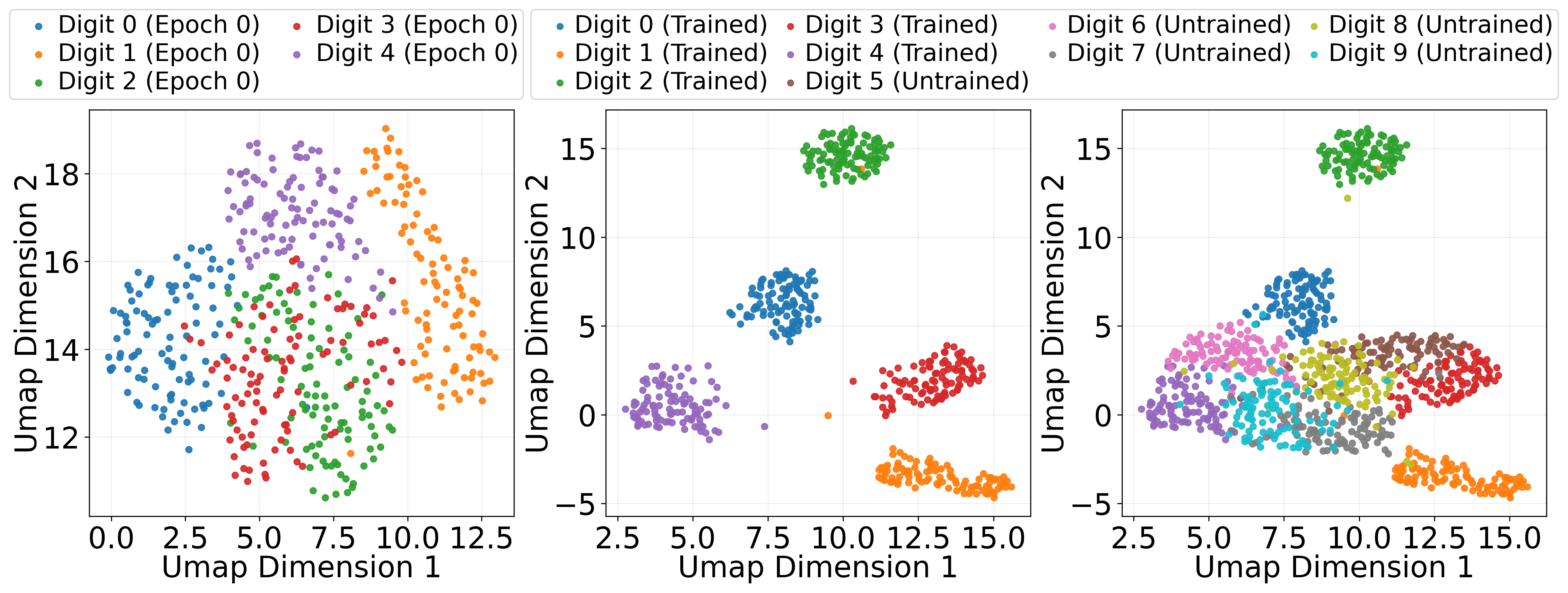}
    \caption{UMAP visualization of MNIST digits in the tangent feature space $\{\nabla_\bTheta f(x_n,\bTheta_k)\}_{n=1}^N$. \textbf{Left:} At initialization ($0$-th epoch), digits $0$–$4$ are not clearly separated, indicating limited discriminative structure. \textbf{Middle:} After training on digits $0$–$4$ ($200$-th epoch), the tangent-space representations of these classes become well clustered and separable. \textbf{Right:} Tangent-feature visualization of all digits: while trained digits ($0$–$4$) remain well separated, unseen digits ($5$–$9$) fail to form distinct clusters.}
    \label{umap_mnist}
\end{figure}

In this subsection, we first illustrate this mechanism empirically on MNIST (Figure~\ref{umap_mnist}): we train on the digits $0$--$4$ and visualize the representation of the resulting tangent space $\{\nabla_\bTheta f(x_n,\bTheta_k)\}_{n=1}^N$ via UMAP. We show that training sharpens and separates the seen classes in tangent feature space, whereas unseen digits $5$--$9$ remain entangled and fail to form coherent clusters, providing a concrete geometric signature of extrapolation failure under the path kernel perspective. At initialization (left), digits $0$--$4$ cannot be clearly separated. After training on digits $0$--$4$ (middle), these classes become well clustered and separable. The right panel includes all digits: while trained digits ($0$--$4$) remain well separated, unseen digits ($5$--$9$) fail to form distinct clusters. This shows that the tangent space representation is strongly adheres to the training distribution. Notably, even for MNIST, some unseen digits that are visually similar to seen ones (e.g., 7 versus the trained digit 1) still fail to form a distinct tangent-space cluster, suggesting a tangible generalization gap on such unseen data.

\begin{figure}[ht]
\centering 
\subfigure[]{
\label{print_MNIST_initialization}
\includegraphics[width=0.32\textwidth]{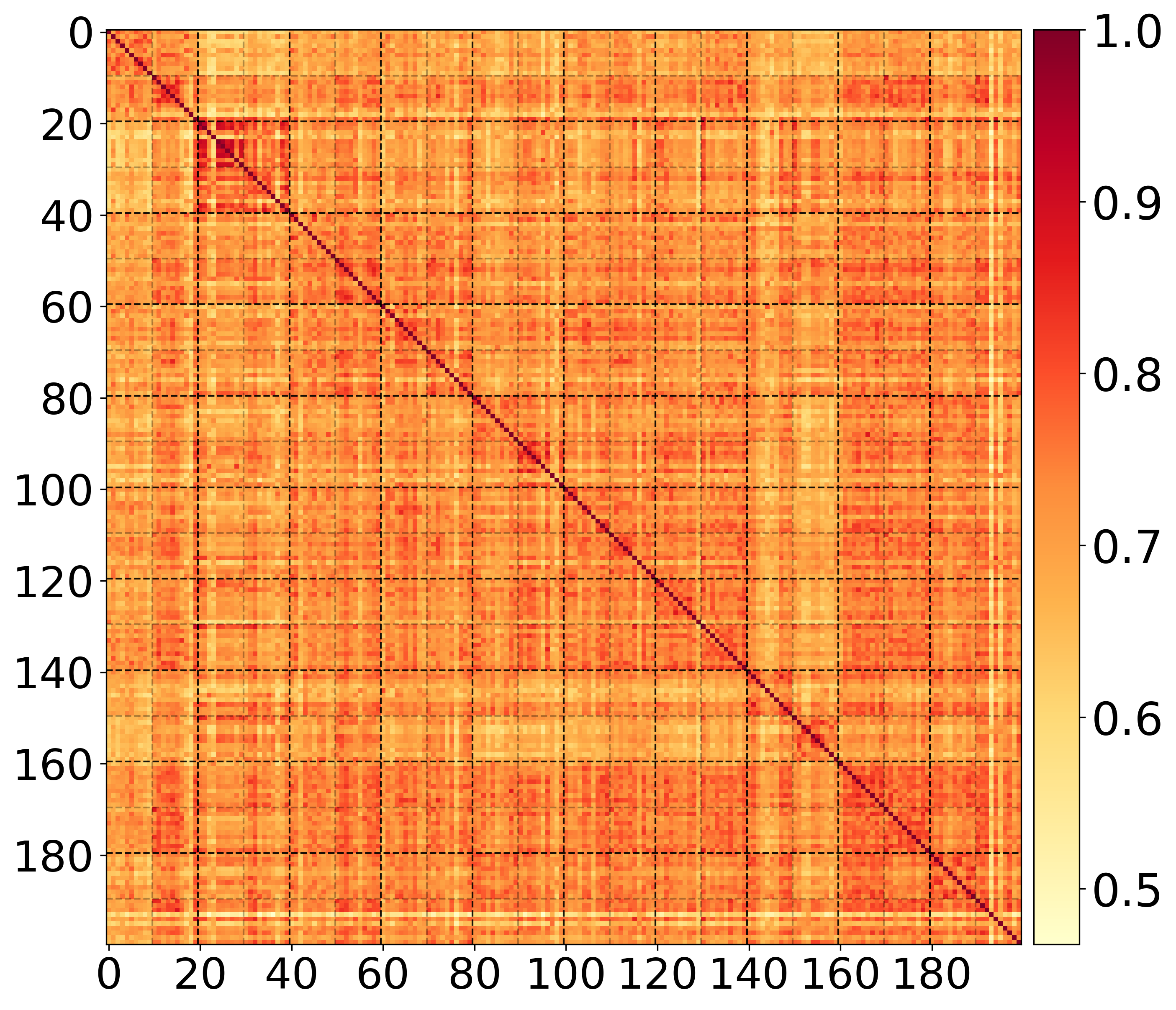}}
\subfigure[]{
\label{print_MNIST_converge}
\includegraphics[width=0.32\textwidth]{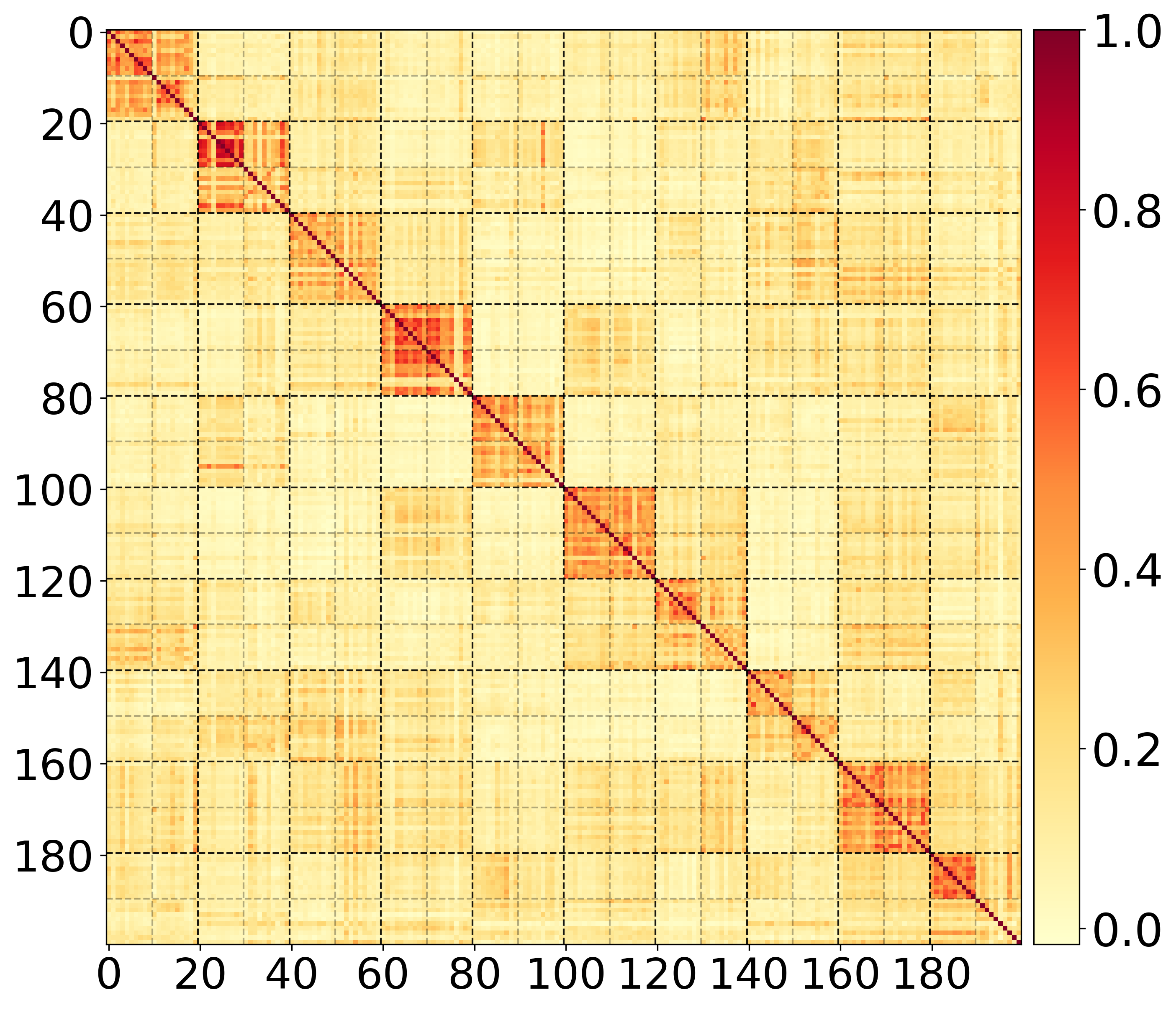}}
\subfigure[]{
\label{fig:confusion_matrixa}
\includegraphics[width=0.305\textwidth]{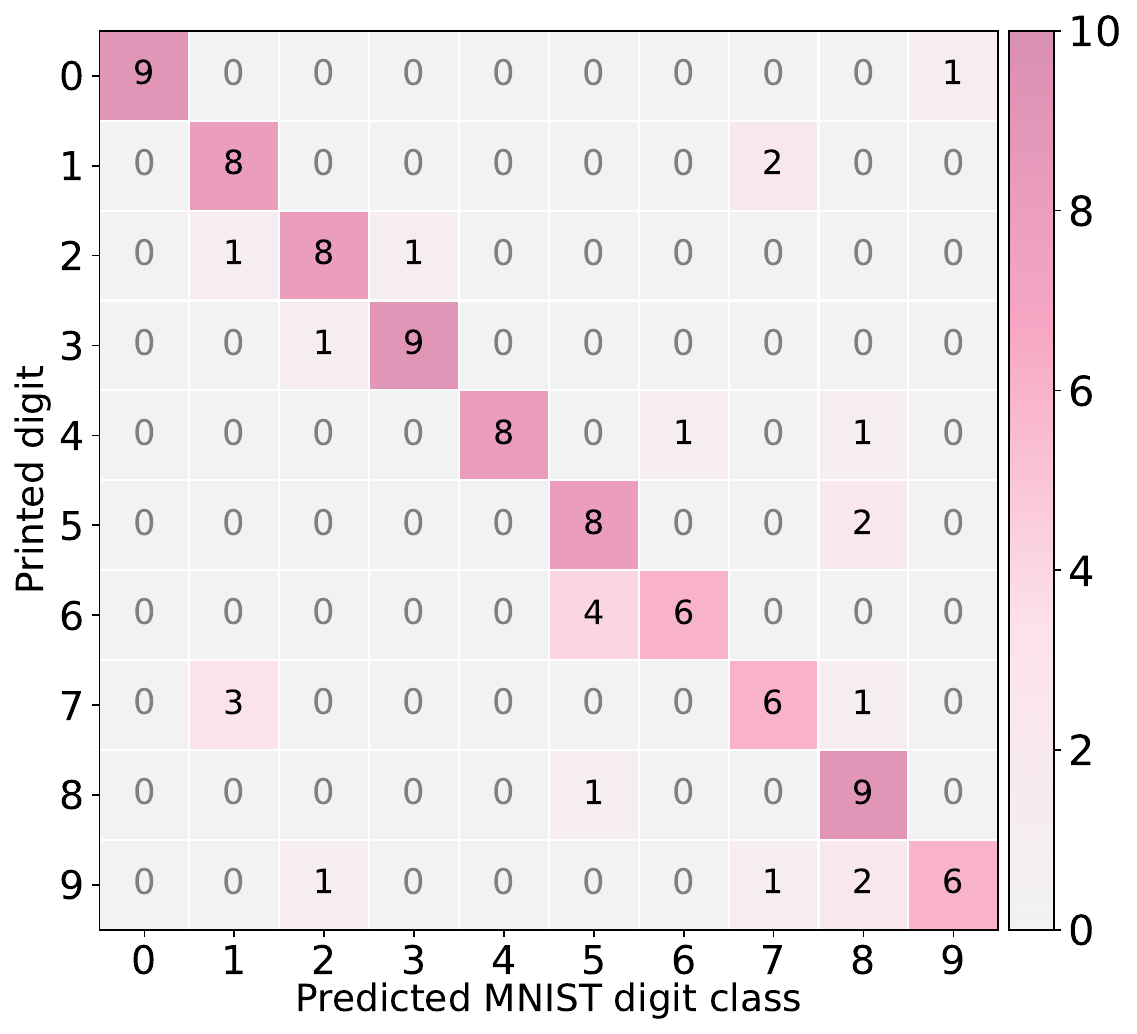}}
    \caption{Normalized gradient kernel for out-of-domain printed digits. Normalized gradient kernel between MNIST digits $k$ (indices $2k\times 10$-$2(k+1)\times 10$) and printed digits $k$ (blocks at indices $(2k+1)\times 10$-$(2k+2)\times 10$, where $k=0,1,\cdots,9$) at initialization (a) and after training ($200$-th epoch) (b). Samples are ordered by class, with 10 images per class for MNIST digits (0–9) and 10 printed images per digit; the block structure highlights within-digit and cross-digit similarities between the two domains. (c) Confusion matrix for classifying printed digits using the network trained on MNIST.}
\label{MNIST printed}
\end{figure}

To investigate the network’s extrapolation capacity, we evaluated performance on increasingly out-of-domain inputs and found that performance is poor: even when new samples are close to the training distribution, predictions remain inaccurate, and moving farther away only worsens the degradation.

\begin{figure}[ht]
    \centering
    \subfigure[]{
\label{alphabeta_MNIST_initialization}
\includegraphics[width=0.4\textwidth]{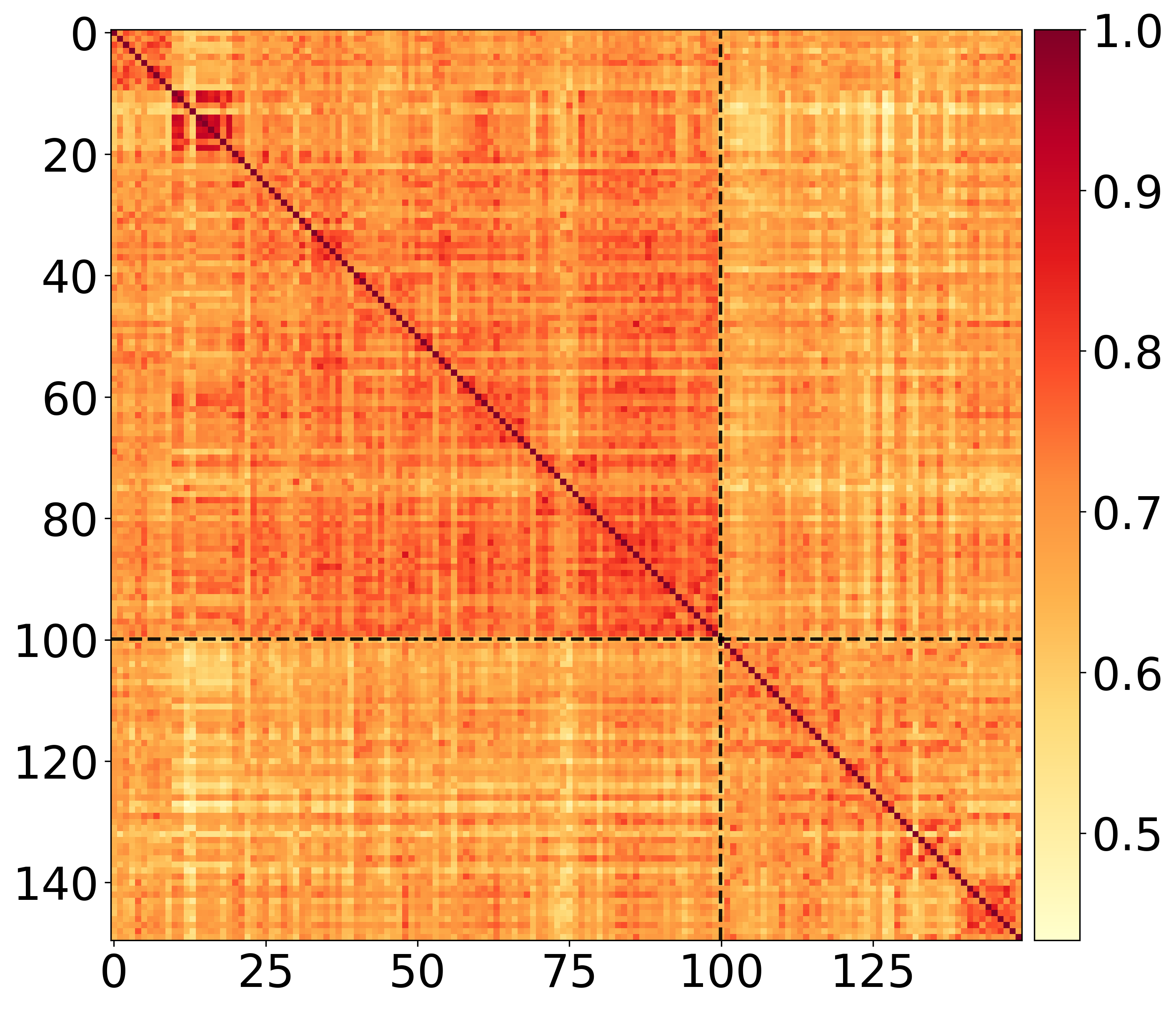}}\hspace{0.3cm}
\subfigure[]{
\label{alphabeta_MNIST_converge}
\includegraphics[width=0.4\textwidth]{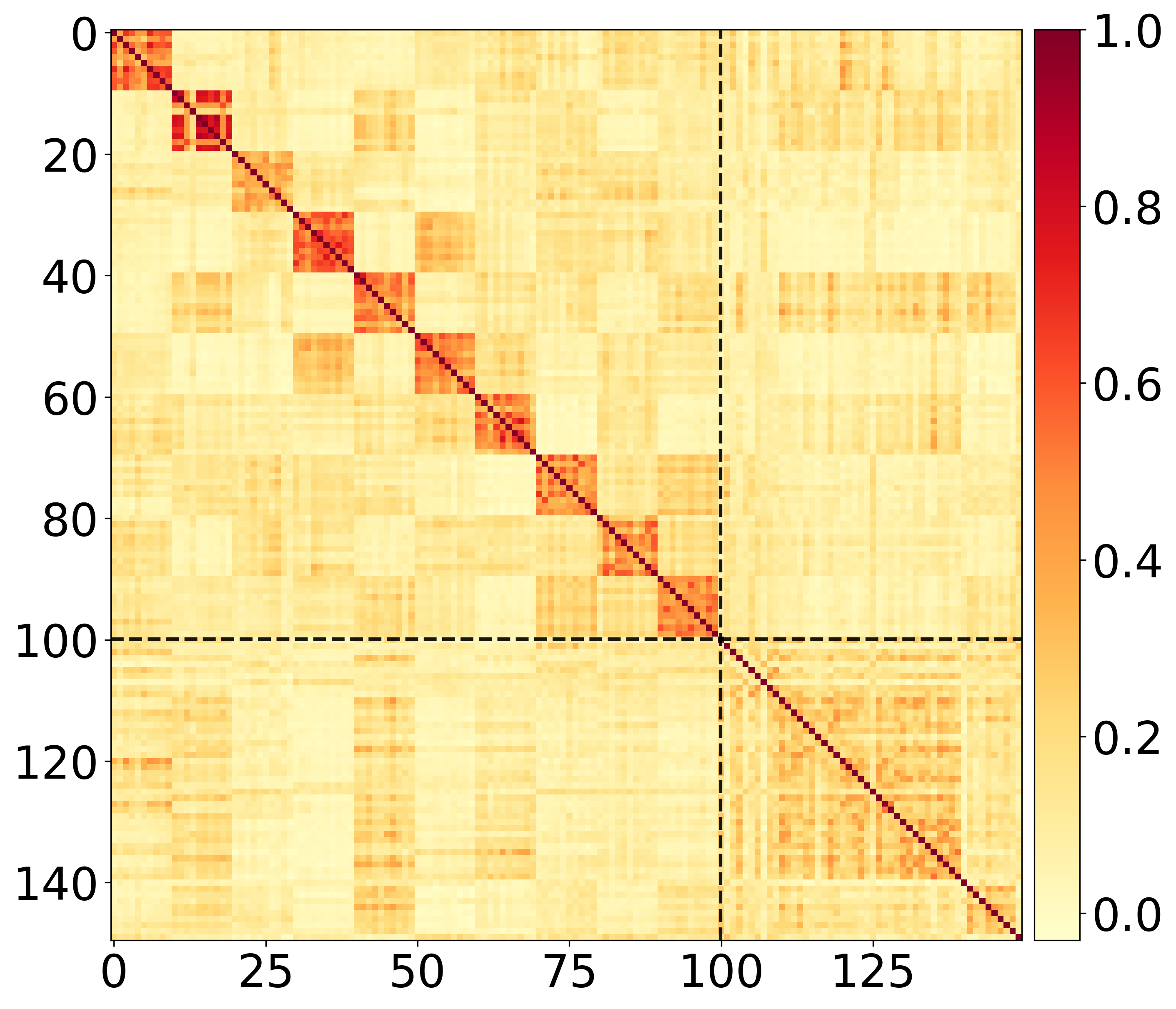}}
    \caption{Normalized gradient kernel for out-of-domain inputs (unseen letters). Normalized gradient kernel between MNIST digits (indices 0-100) and letters $M,N,U,V,A$ (indices 100-150) at initialization (a) and after training ($300$-th epoch) (b). Each subfigure displays 10 samples per digit (0-9) and 10 letters. The dashed lines indicate the boundary between digit and letter blocks.}
    \label{MNIST letters}
\end{figure}

In Figure~\ref{MNIST printed}, we employed a three-stage convolutional network for $28\times 28$ grayscale inputs: each stage applies a $3\times 3$ convolution (stride $1$, padding $1$) followed by $\tanh$ and $2\times 2$ max-pooling, with channel widths $1\!\to\!16\!\to\!32\!\to\!64$. The resulting $64\times 3\times 3$ feature map is flattened and passed through a fully connected layer of width $128$ with $\tanh$ and dropout ($p=0.3$), followed by a $10$-class linear classifier (approximately $9.84\times10^{4}$ parameters total). The model was trained with Adam (learning rate $10^{-3}$) and cross-entropy loss on 600 images per class from MNIST. The goal was to investigate generalization to out-of-domain inputs.

Figures~\ref{print_MNIST_initialization}--\ref{print_MNIST_converge} visualize the normalized gradient kernel between MNIST training samples and printed digits (10 images per digit each). Figure~\ref{fig:confusion_matrixa} reports the classification accuracy on the printed digits after the MNIST-trained network has converged: at the final epoch, digits $0,1,2,3,4,5,$ and $8$ achieve relatively high accuracy, while other classes remain poorly recognized, confirming that classification degrades when test data are outside the training domain. Moreover, a stronger similarity of the tangent space between training and test data correlates with a higher accuracy.

Figures~\ref{MNIST letters} extend this analysis to unseen letters $M, N, U, V, A$. While intra-class similarity among MNIST digits remains high, the similarity between unseen letters and MNIST digits stays uniformly low with no class-specific structure. These results suggest that the network primarily memorizes and interpolates within the training distribution, while its capacity for generalization to out-of-distribution inputs remains limited, highlighting a clear extrapolation gap when test inputs depart from the training domain.

\subsection{Tangent-space rank gap and generalization}
Theorem~\ref{thm:rkhs_decompose} shows that the generalization error concentrates on $\ker(\cK_T)$: target components orthogonal to the span of the training tangent features persist as an irreducible residual. In finite dimensions, this manifests itself as a rank gap between the empirical tangent matrix $\nabla_{\bTheta} f(\bm{X},\bTheta)$ and the mixture tangent matrix $\nabla_{\bTheta} f(\bm{X}^{mix},\bTheta)$. When test inputs excite tangent directions absent from the training set, those directions lie in the null space and cannot be corrected by gradient-based training. In this subsection, we make this mechanism concrete: first through a transparent 1-D ReLU regression example where dormant neurons produce an explicit rank gap, and then through MNIST experiments where out-of-domain samples (printed digits and unseen letters) exhibit reduced tangent-space similarity and degraded generalization.

\begin{figure}[ht]
\centering 
\subfigure[]{
\label{training-regime-(0,2)}
\includegraphics[width=0.23\textwidth]{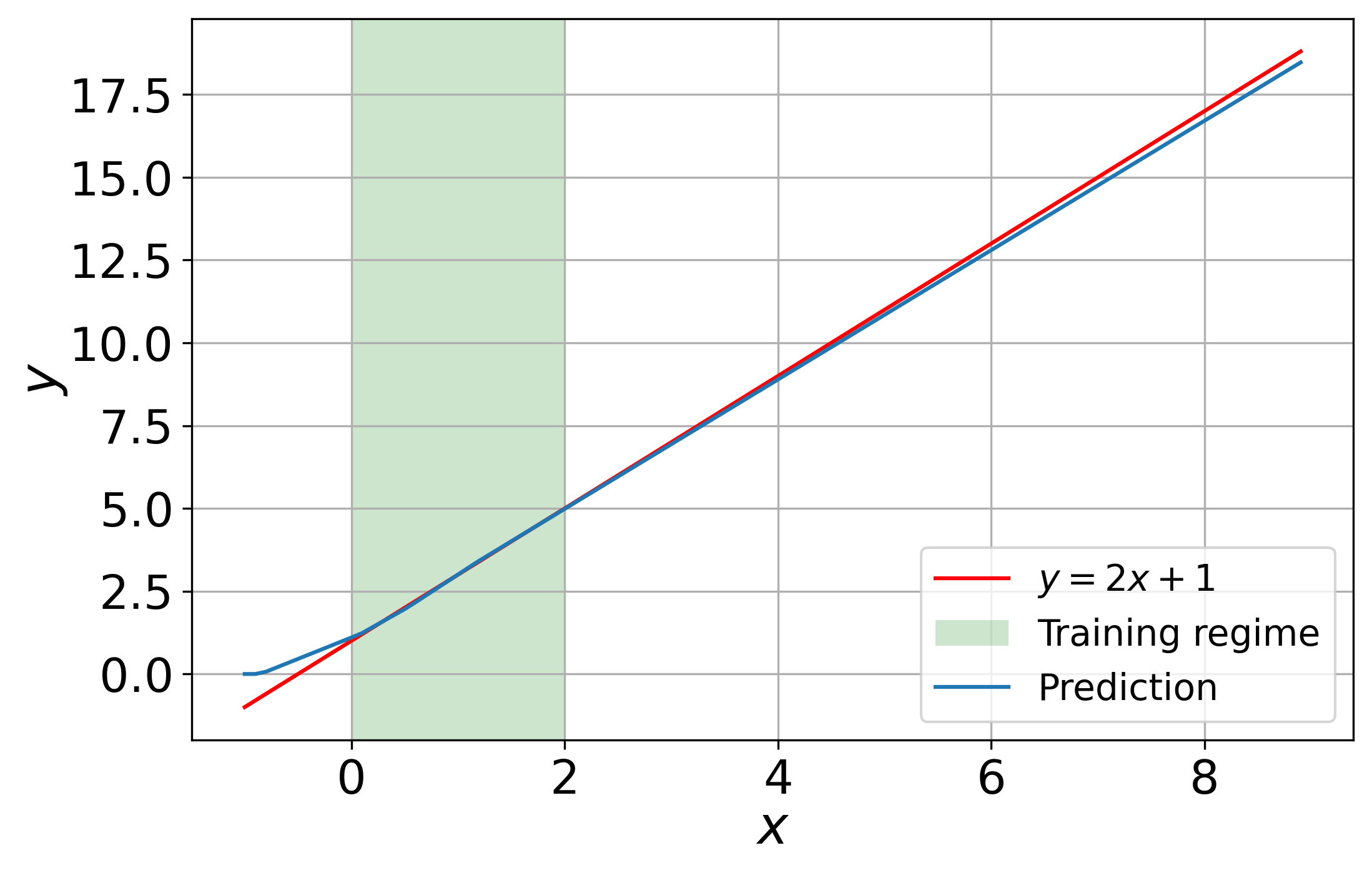}}
\subfigure[]{
\label{parameters-(1.5)}
\includegraphics[width=0.24\textwidth]{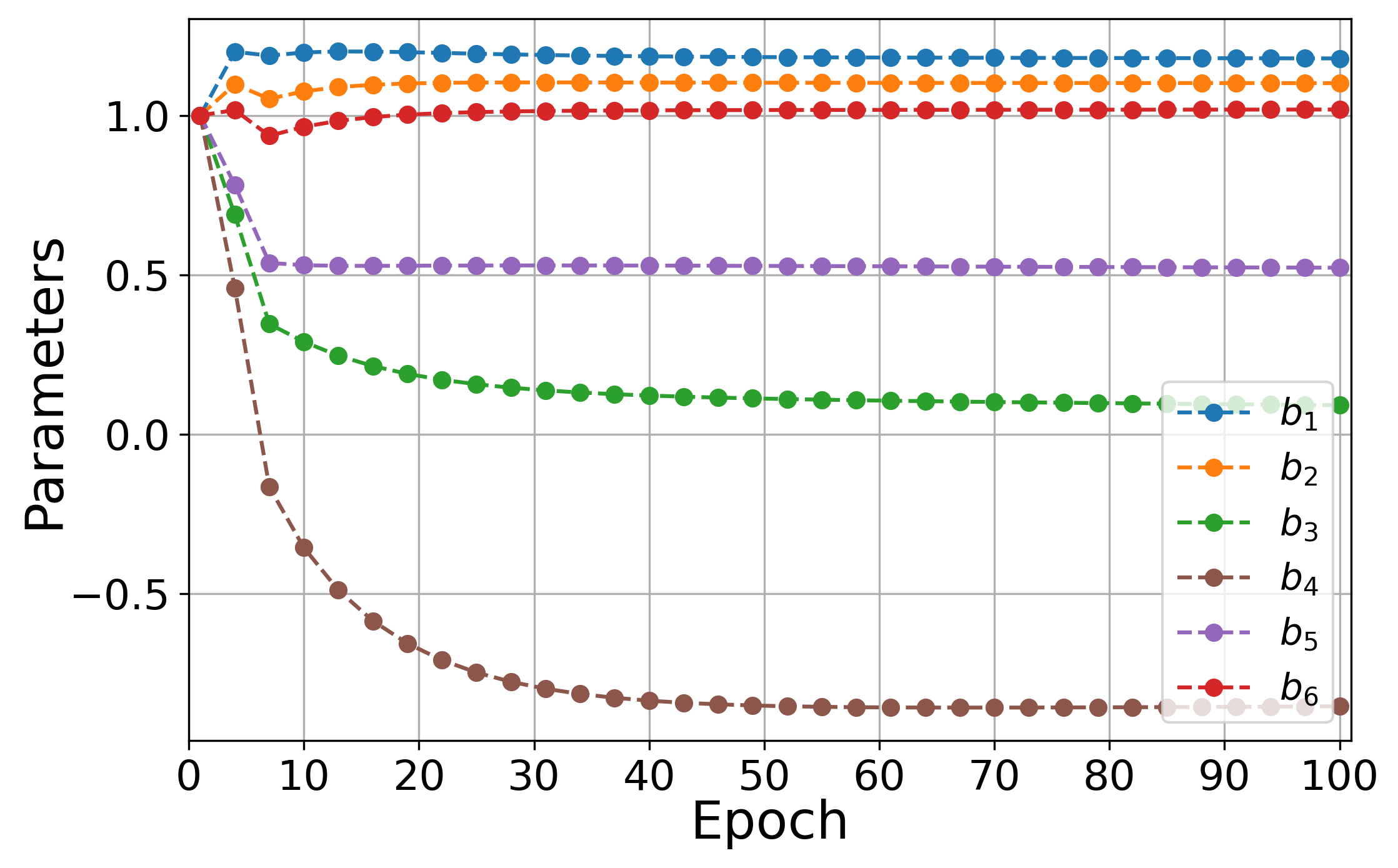}}
\subfigure[]{
\label{ranks-(1.0))}
\includegraphics[width=0.23\textwidth]{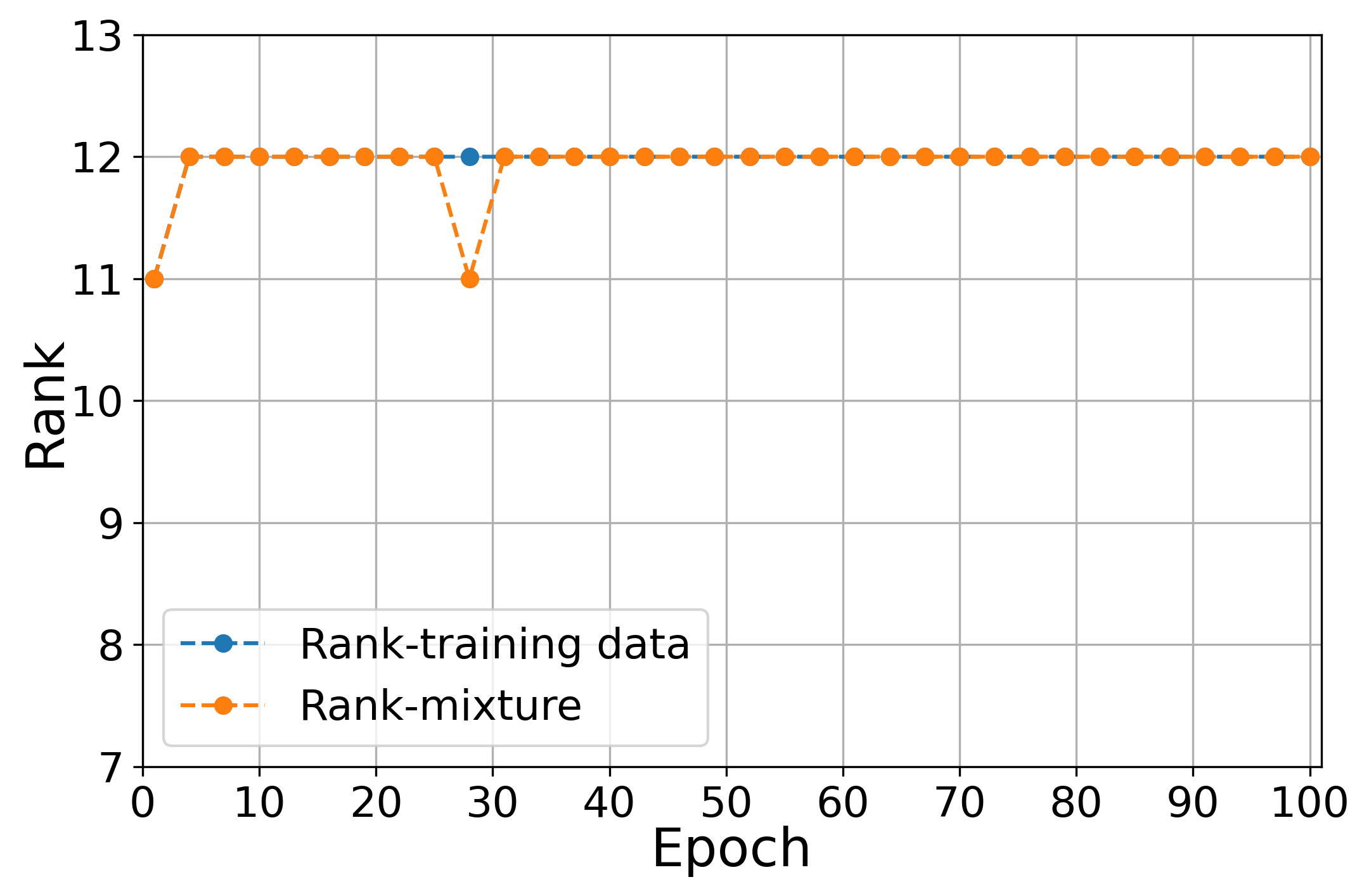}}
\subfigure[]{
\label{testloss-(1.0))}
\includegraphics[width=0.23\textwidth]{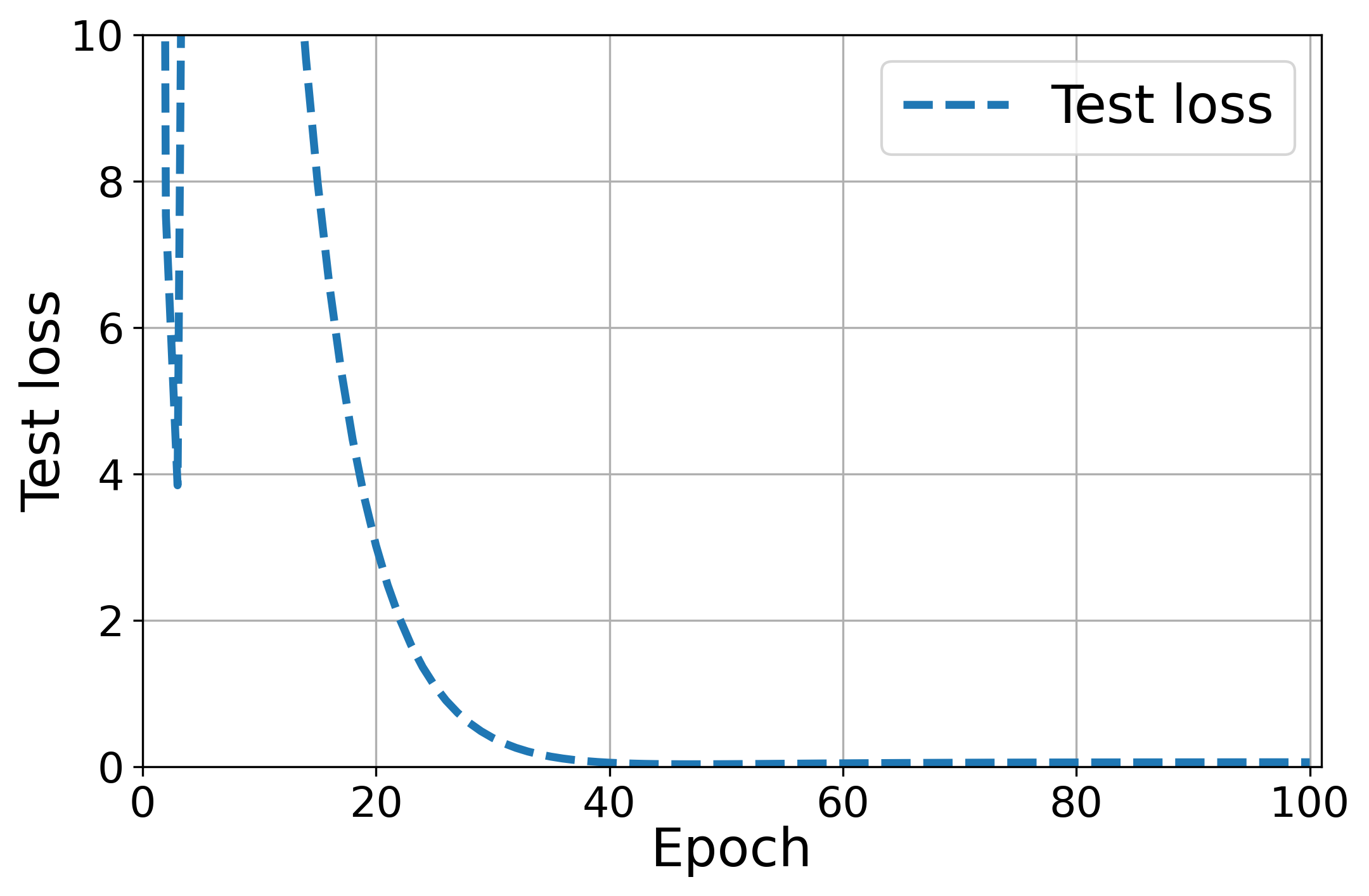}}

\subfigure[]{
\label{1training-regime-(0,2)}
\includegraphics[width=0.23\textwidth]{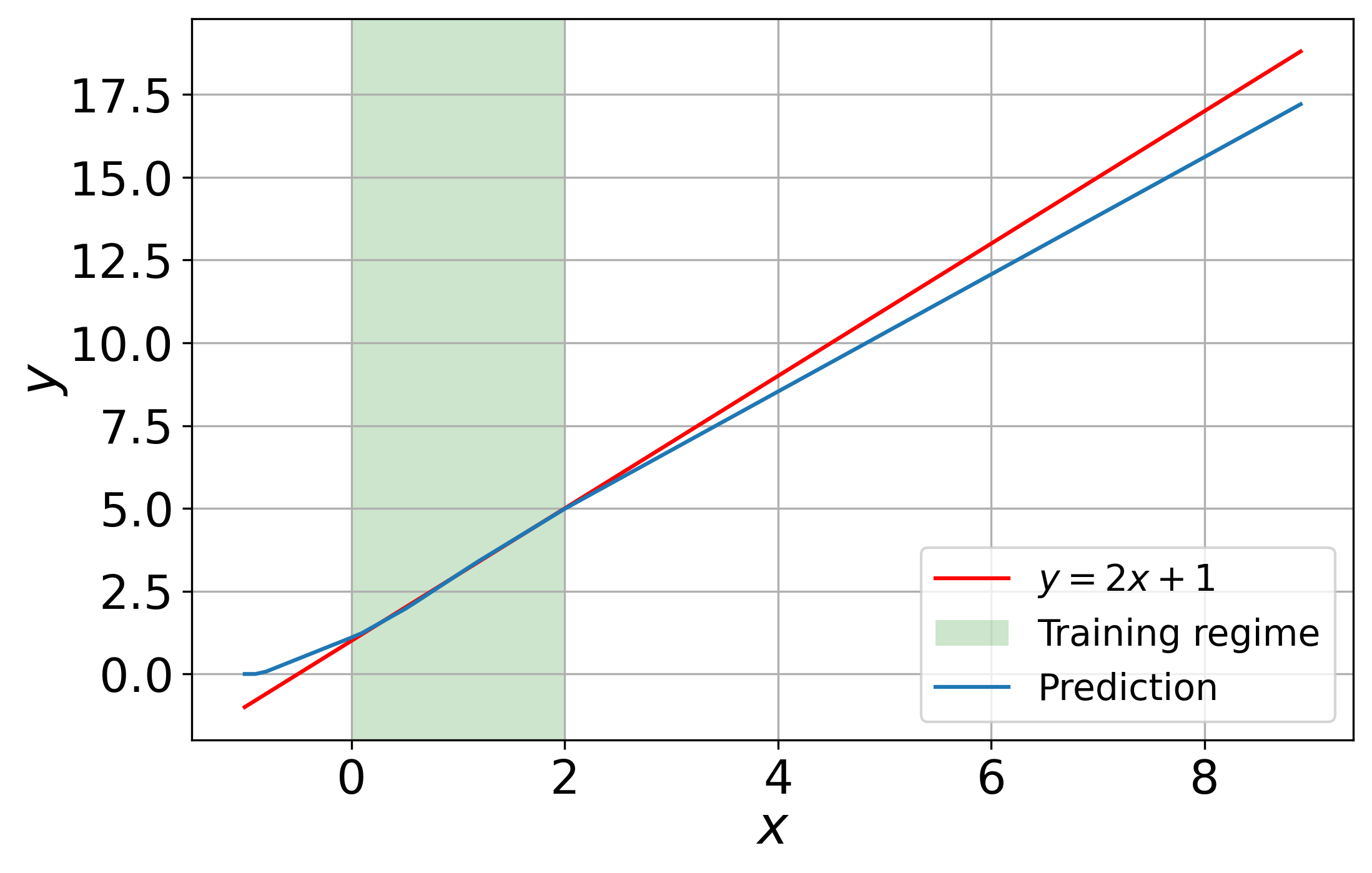}}
\subfigure[]{
\label{parameters-(2.0)}
\includegraphics[width=0.24\textwidth]{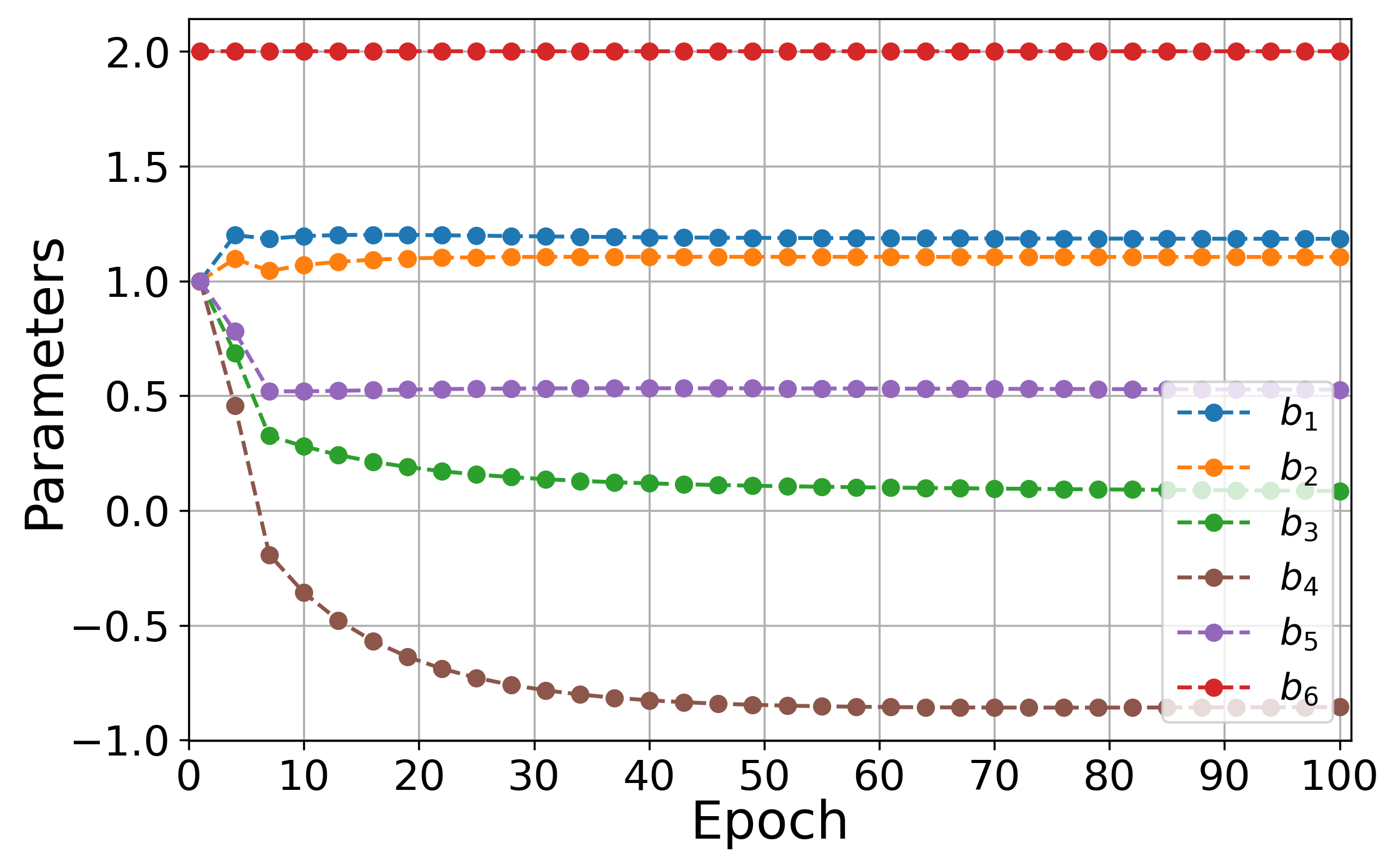}}
\subfigure[]{
\label{ranks-(0.5-2.0}
\includegraphics[width=0.23\textwidth]{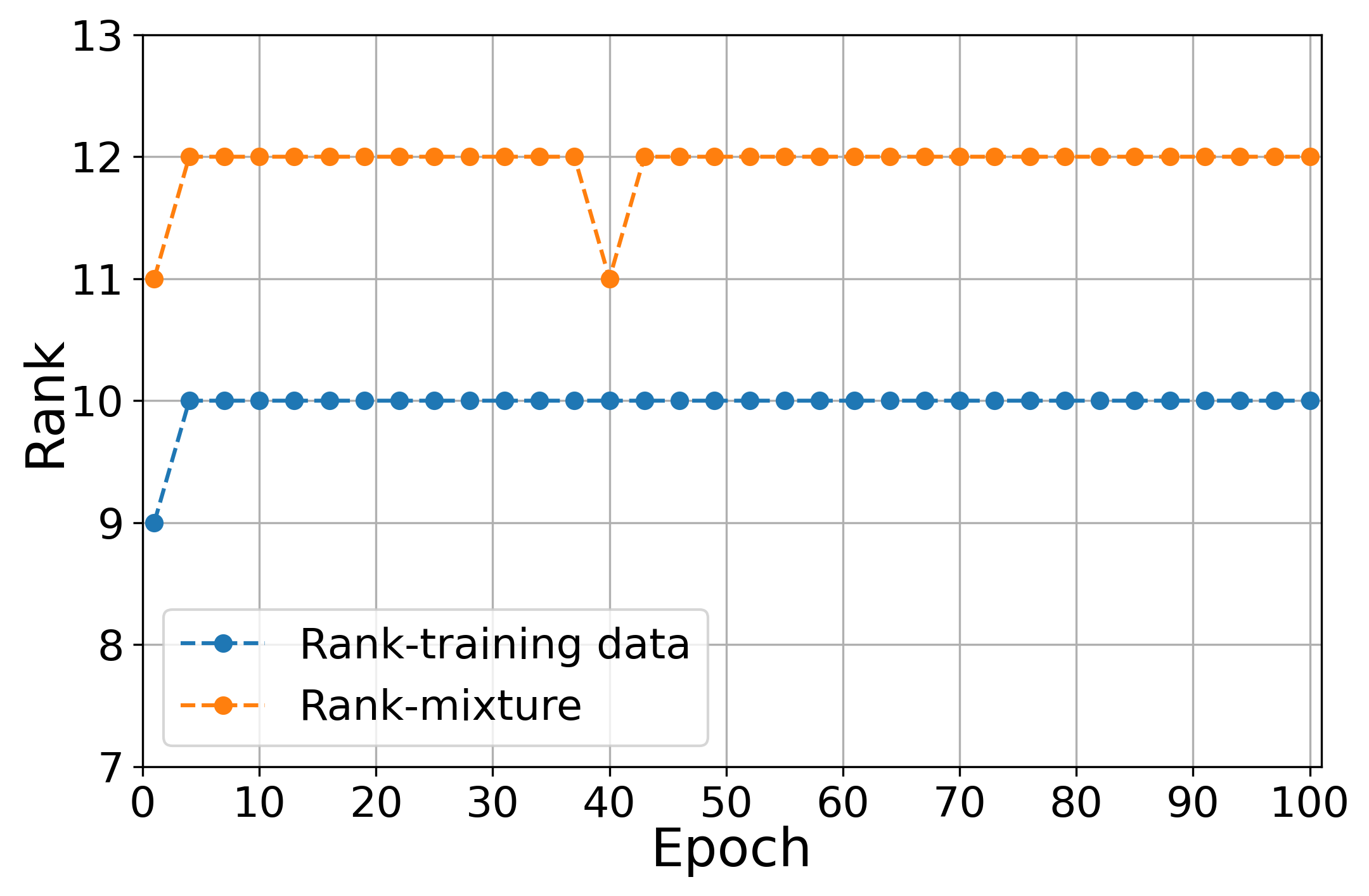}}
\subfigure[]{
\label{testloss-(0.5-2.0}
\includegraphics[width=0.23\textwidth]{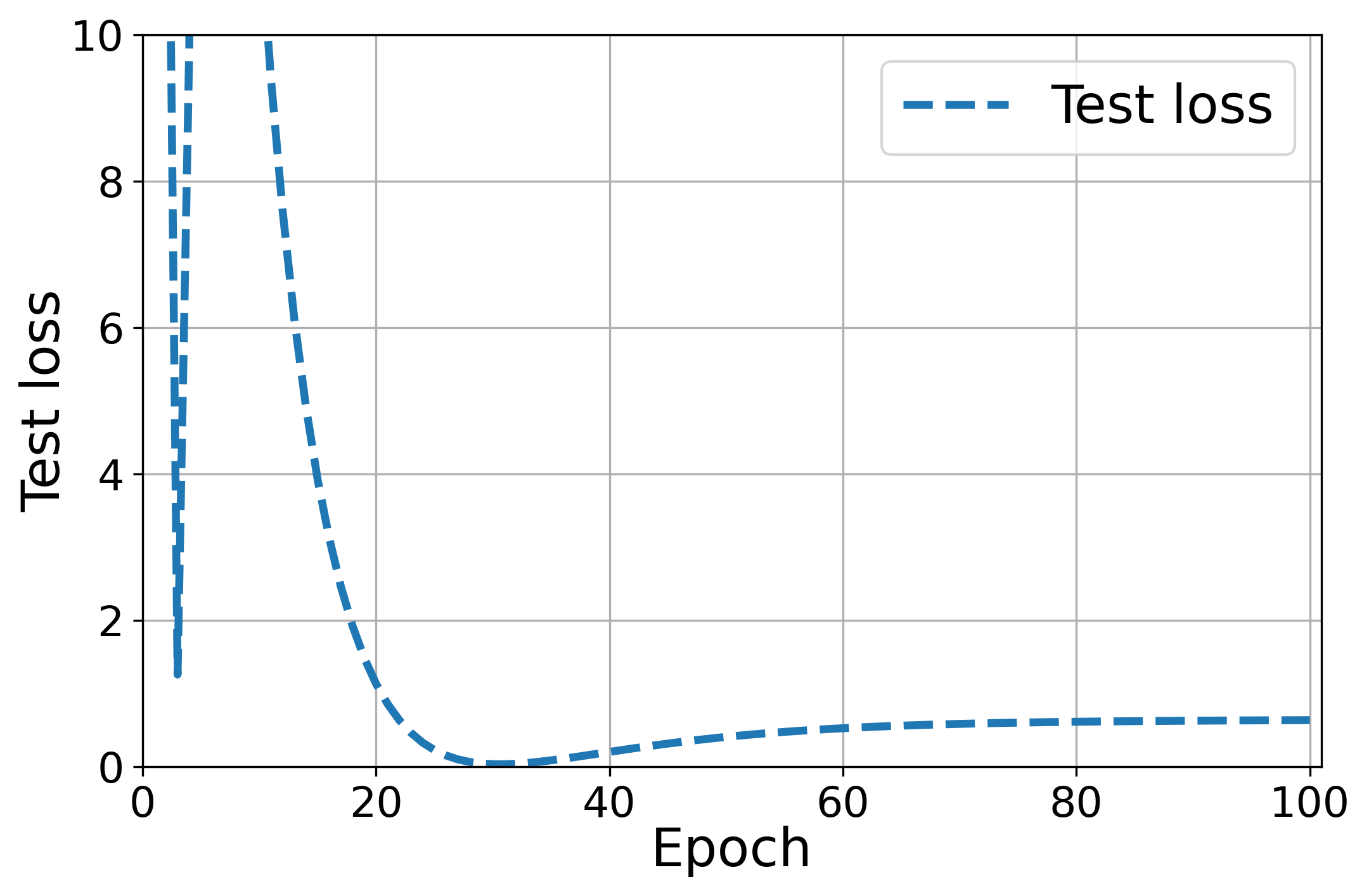}}
\caption{Two training settings for the 1-D regression $y=2x+1$. In the \textbf{first column}, the green band represents the training regime, the red line is the reference function, and the blue line is the network's prediction. The \textbf{second column} shows the evolution of bias parameters during training epochs under different initializations. The \textbf{third column} illustrates the model rank for training data (blue) $\mbox{rank}(\nabla_\bTheta f(\bm{X}_{train},\bTheta))$ and training-testing mixture data $\mbox{rank}(\nabla_\bTheta f(\bm{X}_{mixture},\bTheta))$. The \textbf{fourth column} shows the test loss. The bottom row has one bias outside the training domain, so the corresponding neuron is never activated during training. The rank gap $\mbox{rank}(\nabla_\bTheta f(\bm{X}_{mixture},\bTheta))>\mbox{rank}(\nabla_\bTheta f(\bm{X}_{train},\bTheta))$ emerges and the test loss remains higher, constituting evidence of extrapolation failure.}
\label{ReLU}
\end{figure}

Consider a  fully connected two-layer ReLU network
$$f(x,\bTheta)=\sum_{i=1}^m a_i\,\sigma(w_i^\top x-b_i),\qquad
\sigma(z)=\max\{0,z\},$$
with parameters $\bTheta=\{(a_i,b_i,w_i)\}_{i=1}^m$.
For any input $x$, its tangent feature is
$\nabla_{\bTheta} f(x,\bTheta)$.
In particular,
\begin{align*}
&\frac{\partial f}{\partial a_i}=\sigma(w_i^\top x-b_i), \frac{\partial f}{\partial w_i}=-a_i\,\sigma'(w_i^\top x-b_i)x
=-a_i\,\mathds{1}_{\{w_i^\top x-b_i\ge 0\}}x,\\
&\frac{\partial f}{\partial b_i}=-a_i\,\sigma'(w_i^\top x-b_i)
=-a_i\,\mathds{1}_{\{w_i^\top x-b_i\ge 0\}}.
\end{align*}
The tangent space at $\bTheta$ induced by the training inputs $\{x_n\}_{n=1}^N$ is
$$
\mathcal{V}\;
=\mathrm{span}\Big\{\left\{\sigma(w_i^\top x_n-b_i),\,-a_i\,\mathds{1}_{\{w_i^\top x_n-b_i\ge 0\}},-a_i\,\mathds{1}_{\{w_i^\top x-b_i\ge 0\}}x\right\}_{i=1,n=1}^{m,N}\Big\}.
$$
To make the discussion concrete, we consider the 1-D regression target $y=2x+1$ with $m=6$ hidden neurons. We set $w_i\equiv 1$ to isolate the role of the biases $b_i$.

To quantify the representational capacity of a network relative to its test inputs, we introduce the following definitions.

\begin{definition}
     Given any neural network $f(\cdot,\bTheta)$ with training data $\{(x_n,y_n^*)\}_{n=1}^N$, $\nabla_{\bTheta} f(\bm{X},\bTheta_k)=[\nabla_{\bTheta} f(x_1,\bTheta_k),\nabla_{\bTheta} f(x_2,\bTheta_k),\cdots,\nabla_{\bTheta} f(x_N,\bTheta_k)]$ is referred to as the empirical tangent matrix at $k$-th iteration. For test data $\bm{X}^{test}:=\{x_i\}_{i=1}^M$, the mixture tangent matrix at $k$-th iteration is defined as the column-wise concatenation: $\nabla_{\bTheta} f(\bm{X}^{mix},\bTheta_k)=(\nabla_{\bTheta} f(\bm{X},\bTheta_k),\nabla_{\bTheta} f(\bm{X}^{test},\bTheta_k))\in\mathbb R^{d\times (N+M)}$.
\end{definition}

In Figure \ref{ReLU}, the input is 
$x_n=\frac{n}{50}, n=0,1,\ldots,N=99$, with $x_{\max}=1.98$,
so the training regime is $[0,x_{\max}]$, which is plotted in green. In Figure \ref{training-regime-(0,2)}, the initialization for the bias parameters is $b_i=1,i=1,\cdots,6$, then every neuron is activated by a nontrivial portion of the training inputs, and all bias parameters receive gradients during training (as observed in Fig.~\ref{parameters-(1.5)}) since all the bias parameters $b_i=1\in [0,x_{\max}]$.

In the second row of Figure~\ref{ReLU}, we initialize $b_i=1$ for $i=1,\cdots,5$ and $b_6=2\notin [0,x_{\max}]$. Since $x_{\max}=1.98<b_6$, the $6$-th neuron is never activated on the training set, so its columns in the tangent feature matrix are identically zero. However, for any test input $x>b_6$, this unit is activated, producing tangent directions absent from the training-induced tangent space: $\nabla_\bTheta f(x,\bTheta) \notin \mathcal{V}$. Figure~\ref{parameters-(2.0)} confirms that $b_6$ receives no gradient updates.

\begin{figure}[ht]
    \centering 
\subfigure[]{
    \label{activation regime1}
\includegraphics[width=0.46\textwidth]{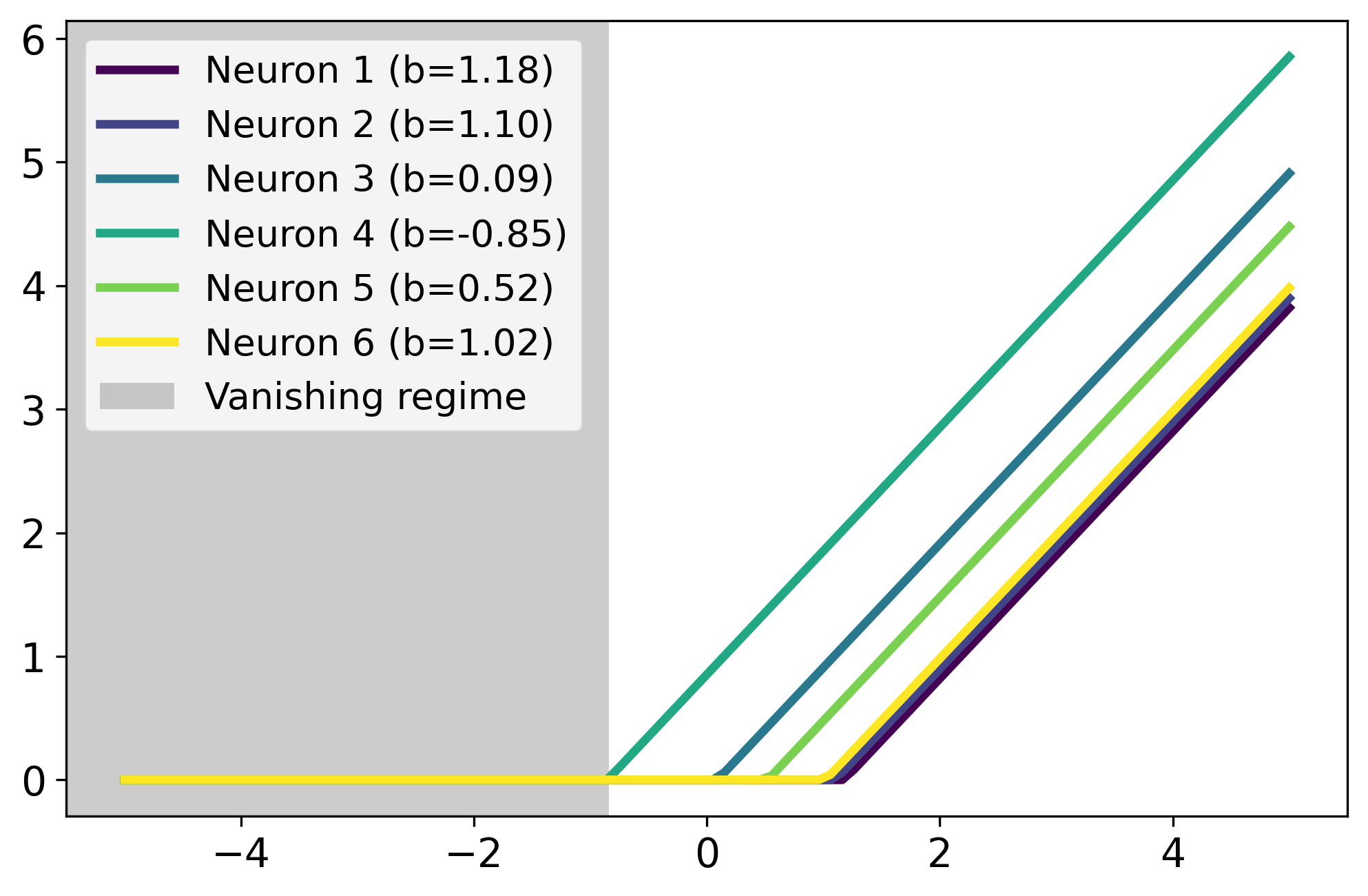}}
    \centering 
\subfigure[]{
    \label{activation regime}
\includegraphics[width=0.46\textwidth]{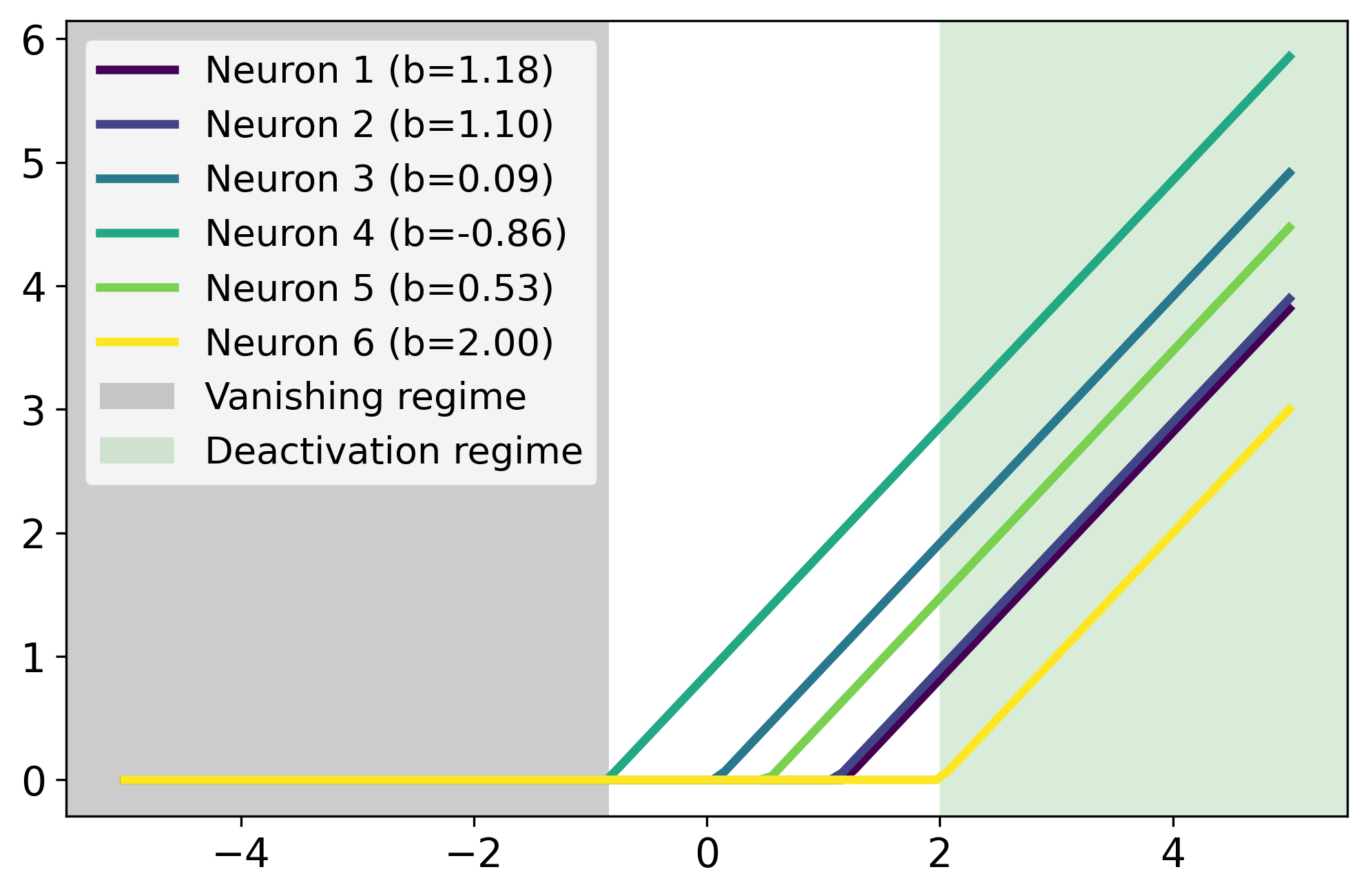}}
    \caption{The visualization for the first layer ReLU activation regimes at convergence. The left panel corresponds to Figure~\ref{parameters-(1.5)} and the right panel to Figure~\ref{parameters-(2.0)}.  Colored curves show the responses of neurons with biases $b_i,i=1,\cdots,6$. Inputs fall in dark-gray vanishing region have zero output and zero tangent features. \textbf{Left}: all units are trained within $[0,1.98]$, no previously unseen tangent directions appear at test time. \textbf{Right}: the parameter $b_6=2.0$ never activated during training but turns on at test time, thus $\nabla_{\bTheta}f(x)\notin \mathrm{span}\{\nabla_{\bTheta}f(x_n,\bTheta)\}_{n=1}^N$ producing extrapolation failure.}
\end{figure}

Comparing the two settings in Figure~\ref{ReLU}: when all biases lie inside the training domain (top row), the empirical and mixture tangent matrices have the same rank throughout training (Figure~\ref{ranks-(1.0))}), and the test loss decreases steadily (Figure~\ref{testloss-(1.0))}).  When one bias lies outside (bottom row), a persistent rank gap $\mbox{rank}(\nabla_\bTheta f(\bm{X}_{mixture},\bTheta))>\mbox{rank}(\nabla_\bTheta f(\bm{X}_{train},\bTheta))$ emerges (Figure~\ref{ranks-(0.5-2.0}), and the test loss remains higher (Figure~\ref{testloss-(0.5-2.0}). The rank gap thus serves as a diagnostic for untrained tangent directions that degrade generalization.

Figures~\ref{activation regime1}--\ref{activation regime} visualize the first-layer ReLU activation regimes at convergence for the two settings. When all thresholds $b_i$ lie within the training domain $[0,1.98]$ (Figure~\ref{activation regime1}), every neuron is activated during training, so test samples do not trigger any previously unseen tangent directions. In contrast, when $b_6=2.0$ lies outside the training domain (Figure~\ref{activation regime}), this unit is never activated during training but turns on at test time for $x>2$, injecting a nonzero block of untrained coordinates into $\nabla_{\bTheta} f(x,\bTheta)$. The resulting activation-pattern shift creates an extrapolation gap: the predictor is forced to make decisions along directions of parameter space that were never trained.

\section{Conclusion}\label{conclude}

This work generalizes Domingos' theorem to stochastic, non-asymptotic settings for the optimizers most widely used in practice: SGD, SGD with momentum, RMSprop, and Adam. A central object in our framework is the stochastic gradient kernel, which extends the deterministic gradient kernel introduced by Domingos to discrete mini-batch training via a continuous-time diffusion approximation. Using this kernel, we showed that the expected network output admits a path-kernel representation in which training samples contribute through loss-dependent weights and gradient alignment along the stochastic trajectory, with optimizer-specific weighting mechanisms. Momentum introduces an exponential temporal memory, while adaptive methods reshape the interpolation geometry through time-varying preconditioners that act on the stochastic gradient kernel. These results establish that the kernel-machine structure of gradient-based learning persists under mini-batch noise, and that stochasticity perturbs the weighting of the path kernel rather than fundamentally altering the underlying learning mechanism.

Viewed through the path-kernel lens, training acts as a sequence of weighted projections onto the current tangent-feature span. Only the components of the target function that fall inside this span are modified by learning; the orthogonalcomponents orthogonal remain untouched. This induces an adaptive RKHS whose geometry evolves with the optimization trajectory. We formalize this by showing that, after convergence, the terminal residual concentrates near the null space of the integral operator induced by the stochastic gradient kernel, so that the generalization error is governed by the target energy in the unlearnable modes.

Beyond supervised learning, we extended the path-kernel perspective to generative modeling. For diffusion models, the path kernel localizes corrections by noise level through the time embedding, producing stage-wise refinement. For GANs, the discriminator's density-ratio geometry globally couples corrections across the latent space, offering a principled lens on mode collapse as an amplification of dominant trajectories through the kernel.

Our numerical experiments corroborate these theoretical findings. Training drives the normalized gradient kernel toward a block-diagonal, class-aware structure; the tangent feature space becomes linearly separable after training, a prerequisite for the Domingos representation to yield correct classification; a tangent-space rank gap between training and test inputs diagnoses extrapolation failure, consistent with the null-space characterization of generalization error; and centered kernel alignment between the tangent-feature kernel and the label kernel tracks training progress and correlates with final accuracy across architectures.

Path kernels offer a flexible framework for understanding deep learning through the lens of kernel methods. The expressive power of deep networks hinges on the richness of the function space defined by the architecture, which shapes the gradients and consequently the path kernel. This architectural dependence can be ported directly to kernel design: kernels can inherit translation invariance from convolutional networks~\cite{lecun1998gradient} or selective attention from transformer architectures~\cite{vaswani2017attention}.

\appendix
\section{Path kernel with scheduler}\label{sec_path_kernel_scheduler}

A widely used learning‑rate schedule in deep learning is the cosine scheduler~\cite{loshchilov2016sgdr}, which smoothly anneals the step size from a large initial value to a small final value according to
\[
\eta_t = \eta_{\min} + \frac12 (\eta_{\max} - \eta_{\min})
\Bigl(1 + \cos\!\,\bigl(\tfrac{t}{T}\pi\bigr)\Bigr),\qquad t = 0,\dots,T,
\]
where \(\eta_{\max}>0\) and \(\eta_{\min}\ge 0\) are the maximum and minimum learning rates, \(T\) is the total number of training iterations, and \(t \in \{0,1,\dots,T\}\) is the iteration index. Compared to a constant step size or piecewise decay, cosine annealing provides a continuous, progressively reduction of the learning rate, which is empirically associated with more stable late‑stage optimization and improved final performance. The update magnitude gradually shrinks, reducing oscillations near minima while still allowing rapid progress early in training.

With this scheduler, the discrete gradient‑descent update
\[
\bTheta_{t+1} - \bTheta_t = -\eta_t \nabla \mathcal L(\bTheta_t)
\]
admits a continuous‑time interpretation as a gradient flow with a time‑dependent speed factor. In the limit where the iteration index is embedded in continuous time (and, formally, when \(\eta_{\min}\to 0\)), one may write
\[
\frac{d\bTheta_t}{dt}
= -\Bigl(1 + \tfrac12(\kappa-1)\bigl(1 + \cos(\tfrac{\pi t}{T})\bigr)\Bigr)
\nabla \mathcal L(\bTheta_t),
\qquad \kappa = \eta_{\max}/\eta_{\min}.
\]
Applying Theorem~\ref{thm_domingo} to this time-rescaled gradient flow yields the following scheduler-weighted path-kernel evolution.
\begin{corollary}[Scheduler-weighted path-kernel evolution]
Assume the same regularity conditions as in Theorem~\ref{thm_domingo}, and consider the continuous-time embedding of the cosine learning-rate schedule, which produces a time-dependent speed factor $w(t)$. Then for any model output $f(x,\bTheta_t)$, the terminal prediction satisfies
\begin{equation}\label{eq:cosine_path_kernel}
f(x,\bTheta_T)
=
f(x,\bTheta_0)
-\frac{1}{N}\int_{0}^{T} 
w(t)\sum_{n=1}^{N}\,
\bigl\langle \nabla_{\bTheta} f(x,\bTheta_t),\nabla_{\bTheta} f(x_n,\bTheta_t)\bigr\rangle
\frac{\partial \ell}{\partial f}\bigl(f(x_n,\bTheta_t),y_n\bigr) \,dt,
\end{equation}
with the scalar weight $w(t)= 1 + \tfrac12(\kappa-1)\bigl(1 + \cos(\tfrac{\pi t}{T})\bigr), t\in[0,T].$
\end{corollary}

This expression shows that cosine annealing does not alter the path‑kernel structure, but reweights the contribution of each infinitesimal training step along the optimization trajectory by \(w(t)\). Early iterations (large \(\eta_t\)) receive a larger weight, emphasizing rapid representation/parameter movement, while later iterations (small \(\eta_t\)) are downweighted, reflecting fine‑tuning behavior near convergence.

\section{Weak approximations of stochastic optimizers}\label{appendix optimizers}

\subsection{SGDM}\label{appendix sgdm}

Empirical evidence shows that neural networks trained with stochastic gradient descent with momentum (SGDM)~\cite{sutskever2013importance,leclerc2020two} achieve superior performance compared to those trained with plain SGD, prompting extensive theoretical investigation of this phenomenon. From an optimization point of view, Defazio~\cite{defazio2020understanding} argues that momentum accelerates the training loss convergence by counteracting stochastic gradient noise in the early phases of training.

The $k$-th iteration of SGDM is given by
\[
\bTheta_{k+1} = \bTheta_k - \eta \, m_{k+1}, \qquad
m_{k+1} = \beta m_k + \nabla_\bTheta L_{\mathcal{B}_k}(\bTheta_k),
\]
where $\beta \in [0,1)$ is the momentum coefficient and $m_k$ accumulates a moving average of past gradients. Summing the updates from $i=0$ to $k$ yields
\[
m_{k+1} = \sum_{i=0}^{k} \beta^{\,k-i} \, \nabla_\bTheta L_{\mathcal{B}_i}(\bTheta_i),
\]
so that the training process can be written as
\[
\bTheta_{k+1} = \bTheta_k - \eta \sum_{i=0}^{k} \beta^{\,k-i} \, \nabla_\bTheta L_{\mathcal{B}_i}(\bTheta_i).
\]
This momentum term will be incorporated when we extend Domingos’ theorem from SGD to SGDM. In particular, the following theorem characterizes SGDM~\cite{li2019stochastic,wang2023marginal} as a weak first-order discretization of an underlying stochastic differential equation.

\begin{lemma}[\textbf{First-order weak approximation of  SGDM}\textsuperscript{\cite{li2019stochastic,wang2023marginal}}]\label{1 order SGDM}
Take the  same assumptions  as in Theorem \ref{lem:1order} and  assume that the drift terms in the following SDE are Lipschitz. Let $\mu=\frac{1-\beta}{{\eta}}$ and define $\{\bM_t,\bZ_t:t\in [0, k\eta]\}$ as the stochastic process satisfying the SDEs
    \begin{align}\label{sde SGDM1}
        & d\bM_t = -[\mu \bM_t+\nabla_\bTheta L(\bZ_t)]dt+\sqrt{\eta}\bs_{|\mathcal{B}|}^\frac{1}{2}(\bZ_t)d\bW_t, \bM_0=m_0\notag\\
        & d\bZ_t = \bM_t dt,\bZ_0=\bTheta_0,
    \end{align}
    where $\bs_{|\mathcal{B}|}(\bZ_t)$ is defined in Theorem \ref{lem:1order}. Then $\{\bM_t,\bZ_t:t\in [0, k\eta]\}$ is a first-order approximation of the SGDM.
\end{lemma}

\subsection{Adam}\label{appendix adam}
Recent advances in optimization have led to the development of adaptive gradient methods, including AdaGrad~\cite{duchi2011adaptive}, RMSProp~\cite{tieleman2012lecture}, and Adam~\cite{kingma2014adam}. These methods dynamically adjust an element‑wise (coordinate‑wise) effective step size for each parameter and have become standard tools in numerous applications. Adam, in particular, has gained widespread adoption due to its empirically validated robustness: its default hyperparameters consistently deliver strong performance across diverse problem domains~\cite{harutyunyan2019multitask, xu2015show, oh2017value}.

For the Adam optimizer, the training process can be expressed as 
\begin{align*}
    &m_{k+1} = \frac{1}{1-\beta_1^{k+1}}\Bigl(\beta_1 m_k + (1-\beta_1)\nabla_\bTheta L_{\mathcal{B}_k}(\bTheta_k)\Bigr),\\[4pt]
    &v_{k+1} = \frac{1}{1-\beta_2^{k+1}}\Bigl(\beta_2 v_k + (1-\beta_2)\nabla_\bTheta L_{\mathcal{B}_k}(\bTheta_k) \odot \nabla_\bTheta L_{\mathcal{B}_k}(\bTheta_k)\Bigr),\\[4pt]
    &\bTheta_{k+1} = \bTheta_k - \eta \,\frac{m_{k+1}}{\sqrt{v_{k+1}}+\varepsilon},
\end{align*}
where \(\varepsilon > 0\) is a small constant, \(\beta_1,\beta_2 \in [0,1)\) are the decay rates for the first‑ and second‑order momentum, respectively, and \(\odot\) denotes the Hadamard product. The square root \(\sqrt{v_k}\) is taken element‑wise, i.e., \((\sqrt{v_k})_i = \sqrt{(v_k)_i}\) for \(i = 1,\dots,d\). Here, \(m_k\) and \(v_k\) represent exponential moving averages of the first and second moments of the stochastic gradient \(\nabla_\bTheta L_{\mathcal{B}_k}(\bTheta_k)\). Motivated by viewing Adam as a discretization of an underlying stochastic dynamics, we now introduce an SDE that approximates Adam in the first-order weak sense~\cite{malladi2022sdes}.

\begin{theorem}[\textbf{First-order weak approximation of the Adam}\textsuperscript{\cite{malladi2022sdes}}]\label{weak approximation of the Adam}
    Take the  same assumptions  as in Theorem \ref{lem:1order} and  assume that the drift terms in the following SDE are Lipschitz. Let $\beta_1=1-c_1\eta, \beta_2=1-c_2\eta$ with $c_1=O(\eta^{-\xi}),c_2=O(1)$ such that $\xi\in (0,1)$, $\varepsilon$ is a small constant. Define the state of the SDEs as $\{\bZ_t,\bM_t,\bV_t\}$ satisfying
    \begin{align}
        & d\bZ_t = -\frac{\sqrt{1-e^{-c_2t}}}{1-e^{-c_1t}}\bP_t^{-1}(\bM_t+\eta c_1(\nabla L(\bZ_t)-\bM_t) )dt,\bZ_0=\bTheta_0,\label{sde Adam}\\
        & d\bM_t = -c_1[ \bM_t-\nabla_\bTheta L(\bZ_t)]dt+\sqrt{\eta}c_1\bs^\frac{1}{2}_{|\mathcal{B}|}(\bZ_t)d\bW_t, \bM_0=0,\\
        & d\bV_t=c_2(\text{diag}(\bs_{|\mathcal{B}|}(\bZ_t))+(\nabla L(\bZ_t))^2-\bV_t)dt,\bV_0=0
    \end{align}
    where $\bP_t=\mbox{diag}(\bV_t)^\frac{1}{2}+\varepsilon\sqrt{1-e^{-c_2t}}\mathbf{I}_d$ is a preconditioner matrix. Then $\{\bZ_t,\bM_t,\bV_t\}$ is a first-order approximation of  Adam.
\end{theorem}

\subsection{RMSprop}\label{appendix rmsprop}
RMSprop~\cite{tieleman2012rmsprop} is a widely used adaptive gradient method in modern machine learning. It can be seen as a non-momentum member of the Adam family, using second-moment statistics to precondition gradients. Empirically, RMSprop remains competitive in several demanding applications, notably GAN training~\cite{yaz2018unusual} and reinforcement learning~\cite{mnih2016asynchronous}.

The updated rule of RMSprop can be written as
\begin{align*}
    &s_{k+1}=\beta s_{k}+(1-\beta)\nabla_\bTheta L_{\mathcal{B}_{k+1}}(\bTheta_{k+1}) \odot \nabla_\bTheta L_{\mathcal{B}_{k+1}}(\bTheta_{k+1})\\
    &\bTheta_{k+1}=\bTheta_k-\eta\frac{\nabla_\bTheta L_{\mathcal{B}_{k+1}}(\bTheta_{k+1})}{\sqrt{s_{k+1}}+\varepsilon}
\end{align*}
where \(\varepsilon > 0\) is a small constant, $\beta$ is the decay factor. The square root \(\sqrt{v_k}\) is taken element‑wise and \(\odot\) denotes the Hadamard product. We now state a first-order weak SDE approximation for RMSprop~\cite{compagnoni2024adaptive}.
\begin{theorem}[\textbf{First-order weak approximation of the RMSprop}\textsuperscript{\cite{compagnoni2024adaptive}}]
    Take the same assumptions as in Theorem~\ref{lem:1order} and and  assume that the drift terms in the following SDE are Lipschitz. Let $\varepsilon$ be a small constant, define the state of the SDEs as $\{\bZ_t,\bS_t\}$ satisfying
    \begin{align}\label{sde rmsp}
        &d\bZ_t=-\bP_t^{-1}(\nabla L(\bZ_t)dt+\sqrt{\eta}\bs^\frac{1}{2}_{|\mathcal{B}|}(\bZ_t)d\bW_t)\notag\\
        &d\bS_t=\mu((\nabla L(\bZ_t))^2+\mathrm{diag}(\bs(\bZ_t))-\bS_t)dt,
    \end{align}
    where $\bP_t=\mbox{diag}(\bS_t)^\frac{1}{2}+\varepsilon\mathbf{I}_d$, $1-\mu \eta=\beta$, $O(\mu)=1$. Then $\{\bZ_t,\bS_t\}$ is a first-order approximation of RMSprop.
\end{theorem}

\section{Proof of Theorem \ref{thm:adam_domingos}\label{proofofAdam}}
\begin{proof}
    For any input $x$, the expectation of output in $k$-iteration can be written as
    \begin{align}
         \mathbb{E}[f(x,\bTheta_{k})]&=\mathbb{E}[f(x,\bZ_{k\eta})]+\mathbb{E}[f(x,\bTheta_{k})]-\mathbb{E}[f(x,\bZ_{k\eta})]\notag\\
        &=\mathbb{E}[f(x,\bZ_{k\eta})]+O(\eta).
    \end{align}
    Consider integrating factor $e^{c_1t}$, we can obtain
    \begin{align*}
        d(e^{c_1t}\bM_t)=&c_1e^{c_1t}\bM_tdt+e^{c_1t}d\bM_t\\
        =&c_1e^{c_1t}\bM_tdt+e^{c_1t}\left(-c_1[ \bM_t-\nabla_\bTheta L(\bZ_t)]dt+\sqrt{\eta}c_1\bs^\frac{1}{2}_{|\mathcal{B}|}(\bZ_t)d\bW_t\right)\\
        =&c_1e^{c_1t}\nabla_\bTheta L(\bZ_t)dt+\sqrt{\eta}c_1e^{c_1t}\bs_{|\mathcal{B}|}^\frac{1}{2}(\bZ_t)d\bW_t
    \end{align*}
    from Theorem~\ref{weak approximation of the Adam}.
    Taking integral in both sides,
    \begin{align*}
        e^{c_1t}\bM_t-\bM_0=\int_0^tc_1e^{c_1s}\nabla_\bTheta L(\bZ_s)ds+\sqrt{\eta}c_1e^{c_1s}\bs^\frac{1}{2}_{|\mathcal{B}|}(\bZ_s)d\bW_s
    \end{align*}
    we get
    \begin{align*}
        \bM_t=c_1\int_0^te^{-c_1(t-s)}\nabla_\bTheta L(\bZ_s)ds+\sqrt{\eta}\,c_1\int_0^te^{-c_1(t-s)}\bs^\frac{1}{2}_{|\mathcal{B}|}(\bZ_s)d\bW_s.
    \end{align*}
    By the same derivation, we have
    \begin{align*}
        \bV_t=c_2\int_0^te^{-c_2(t-s)}\left(\text{diag}(\bs(\bZ_s))+(\nabla L(\bZ_s))^2\right)ds
    \end{align*}
    Then Eq. (\ref{sde Adam}) becomes
    \begin{align*}
        \frac{d\bZ_t}{dt}&=-\frac{\sqrt{1-e^{-c_2t}}}{1-e^{-c_1t}}\bP_t^{-1}(\bM_t+\eta c_1(\nabla L(\bZ_t)-\bM_t))\\&=-\frac{\sqrt{1-e^{-c_2t}}}{1-e^{-c_1t}}\bP_t^{-1}c_1\bigg(\beta_1\int_0^te^{-c_1(t-s)}\nabla_\bTheta L(\bZ_s)ds\\& \qquad+\beta_1\sqrt{\eta}\int_0^te^{-c_1(t-s)}\bs^\frac{1}{2}(\bZ_s)d\bW_s+(1-\beta_1)\nabla L(\bZ_t)\bigg).
    \end{align*}
    Since $\bZ_t$ is continuous with respect to time $t$, applying the chain rule to the model $f(x,\bZ_t)$ yieds
    \begin{align*}
        \frac{df}{dt}(x,\bZ_t)&=\sum_{j=1}^d \frac{\partial f}{\partial \theta^j}(x,\bZ_t)\frac{d Z_t^j}{dt}\\
        &= \nabla_{\bTheta} f(x,\bZ_t) \bigg[-\frac{\sqrt{1-e^{-c_2t}}}{1-e^{-c_1t}}\bP_t^{-1}c_1\bigg(\beta_1\int_0^te^{-c_1(t-s)}\nabla_\bTheta L(\bZ_s)ds\\&+\sqrt{\eta}\beta_1\int_0^te^{-c_1(t-s)}\bs^\frac{1}{2}(\bZ_s)d\bW_s+(1-\beta_1)\nabla L(\bZ_t)\bigg)\bigg].
    \end{align*}
    Integrating of both sides, we get
    \begin{align*}
        &f(x,\bZ_{k\eta})=f(x,\bZ_0)\\&-c_1\beta_1\int_0^{k\eta}\frac{\sqrt{1-e^{-c_2t}}}{1-e^{-c_1t}}\int_0^te^{-c_1(t-s)}\frac{1}{N}\sum_{n=1}^N\widetilde{K}_{t,s}^{\bP_t}(x,x_n)\frac{\partial \ell}{\partial f}(f(x_n,\bZ_s),y_n^*) dsdt\\&-c_1\beta_1\sqrt{\eta}\int_0^{k\eta}\frac{\sqrt{1-e^{-c_2t}}}{1-e^{-c_1t}}(\nabla_\bTheta f(x,\bZ_t))^\top \bP_t^{-1}\int_0^te^{-c_1(t-s)}\bs^\frac{1}{2}_{|\mathcal{B}|}(\bZ_s)d\bW_sdt\\
        &-c_1(1-\beta_1)\int_0^{k\eta}\frac{\sqrt{1-e^{-c_2t}}}{1-e^{-c_1t}}\frac{1}{N}\sum_{n=1}^N\widetilde{K}_{t}^{\bP_t}(x,x_n) \frac{\partial \ell}{\partial f}(f(x_n,\bZ_t),y_n^*)dt
    \end{align*}
    Taking expectation to both sides, we finally have
        \begin{align*}
            &\mathbb{E}[f(x,\bTheta_{k\eta})]=f(x,\bTheta_0)\\&-c_1\beta_1\mathbb{E}\left[\int_0^{k\eta}\frac{\sqrt{1-e^{-c_2t}}}{1-e^{-c_1t}}\int_0^te^{-c_1(t-s)}\frac{1}{N}\sum_{n=1}^N\widetilde{K}_{t,s}^{\bP_t}(x,x_n)  \frac{\partial \ell}{\partial f}(f(x_n,\bZ_s),y_n^*)dsdt\right]\\&-c_1\beta_1\sqrt{\eta}\mathbb{E}\left[\int_0^{k\eta}\frac{\sqrt{1-e^{-c_2t}}}{1-e^{-c_1t}}\bP_t^{-1}(\nabla_\bTheta f(x,\bZ_t))^\top\int_0^te^{-c_1(t-s)}\bs^\frac{1}{2}(\bZ_s)d\bW_sdt\right]\\&
            -c_1(1-\beta_1)\mathbb{E}\left[\int_0^{k\eta}\frac{\sqrt{1-e^{-c_2t}}}{1-e^{-c_1t}}\frac{1}{N}\sum_{n=1}^N\widetilde{K}_{t}^{\bP_t}(x,x_n)  \frac{\partial \ell}{\partial f}(f(x_n,\bZ_t),y_n^*)dt\right]+O(\eta)\notag
            \end{align*}
            \begin{align*}
            =&f(x,\bTheta_0)-c_1\beta_1\mathbb{E}\left[\int_0^{k\eta}\frac{\sqrt{1-e^{-c_2t}}}{1-e^{-c_1t}}\int_0^te^{-c_1(t-s)}\frac{1}{N}\sum_{n=1}^N\widetilde{K}_{t,s}^{\bP_t}(x,x_n)  \frac{\partial \ell}{\partial f}(f(x_n,\bZ_s),y_n^*)dsdt\right]\\&-c_1(1-\beta_1)\mathbb{E}\left[\int_0^{k\eta}\frac{\sqrt{1-e^{-c_2t}}}{1-e^{-c_1t}}\frac{1}{N}\sum_{n=1}^N\widetilde{K}_{t}^{\bP_t}(x,x_n) \frac{\partial \ell}{\partial f}(f(x_n,\bZ_t),y_n^*)dt\right]+O(\sqrt{\eta}),
        \end{align*}
        where $\bP_t=c_2^\frac{1}{2}\mathrm{diag}\left(\int_0^te^{-c_2(t-s)}(\text{diag}(\bs(\bZ_s))+(\nabla L(\bZ_s))^2)ds\right)^\frac{1}{2}+\varepsilon\sqrt{1-e^{-c_2t}}\mathbf{I}$.
\end{proof}

\section{Classification results for different optimizers}

\begin{figure}[ht]
\centering 
\subfigure[]{
\label{Classification results - SGD}
\includegraphics[width=0.23\textwidth]{figures/circle/Classification_results_SGD.png}}
\subfigure[]{
\label{Classification results - SGDM}
\includegraphics[width=0.23\textwidth]{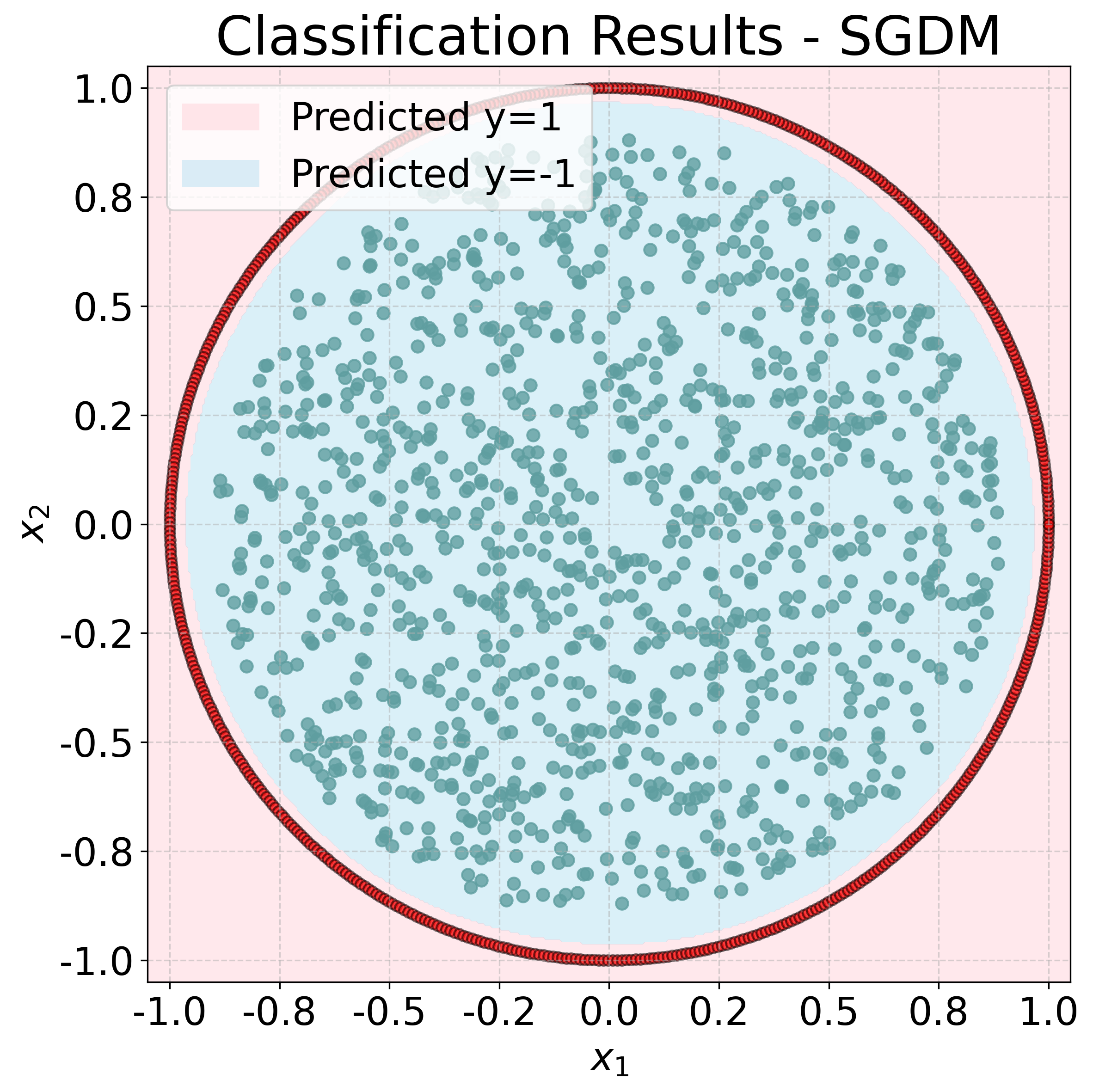}}
\subfigure[]{
\label{Classification result - adam}
\includegraphics[width=0.23\textwidth]{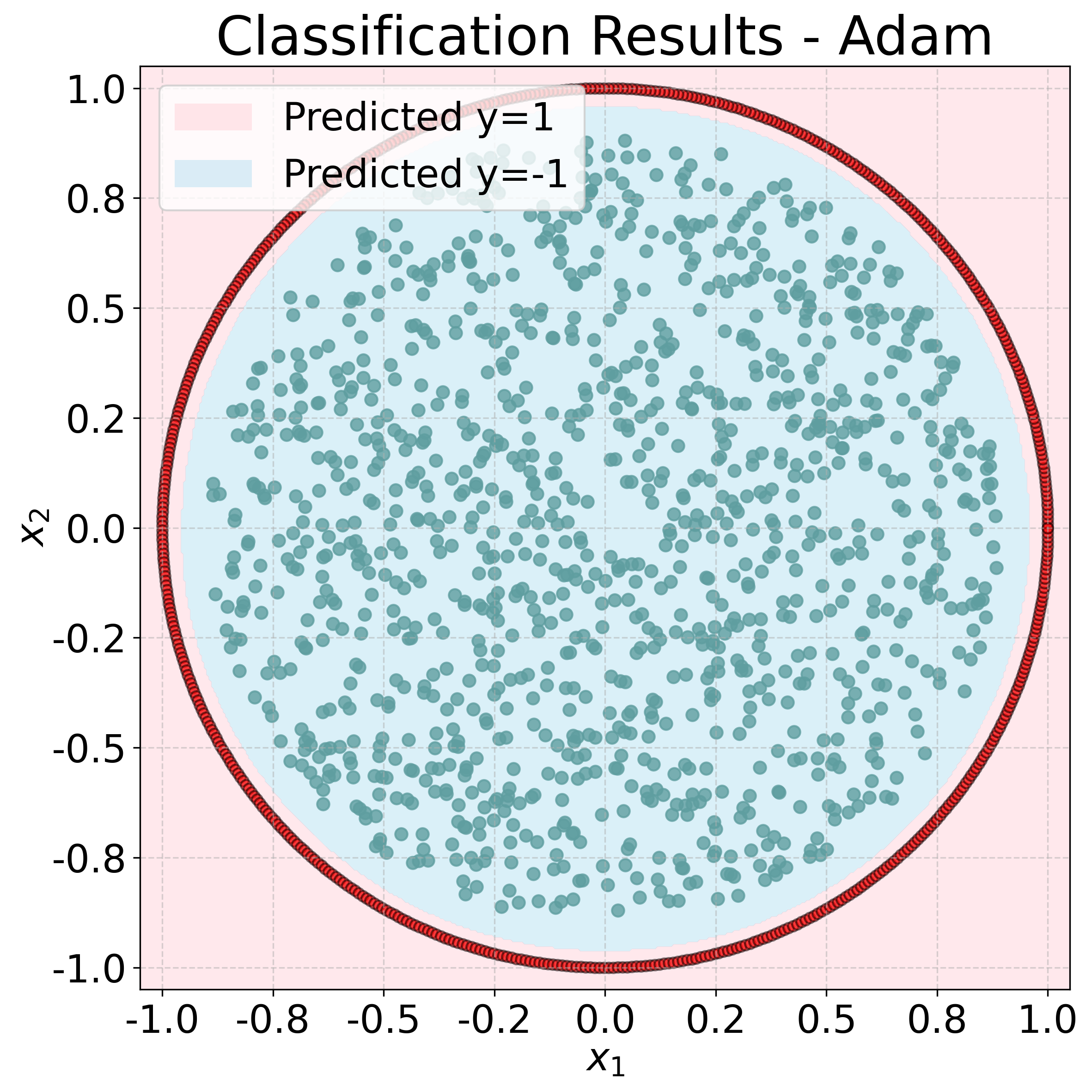}}
\subfigure[]{
\label{Classification result - rms}
\includegraphics[width=0.23\textwidth]{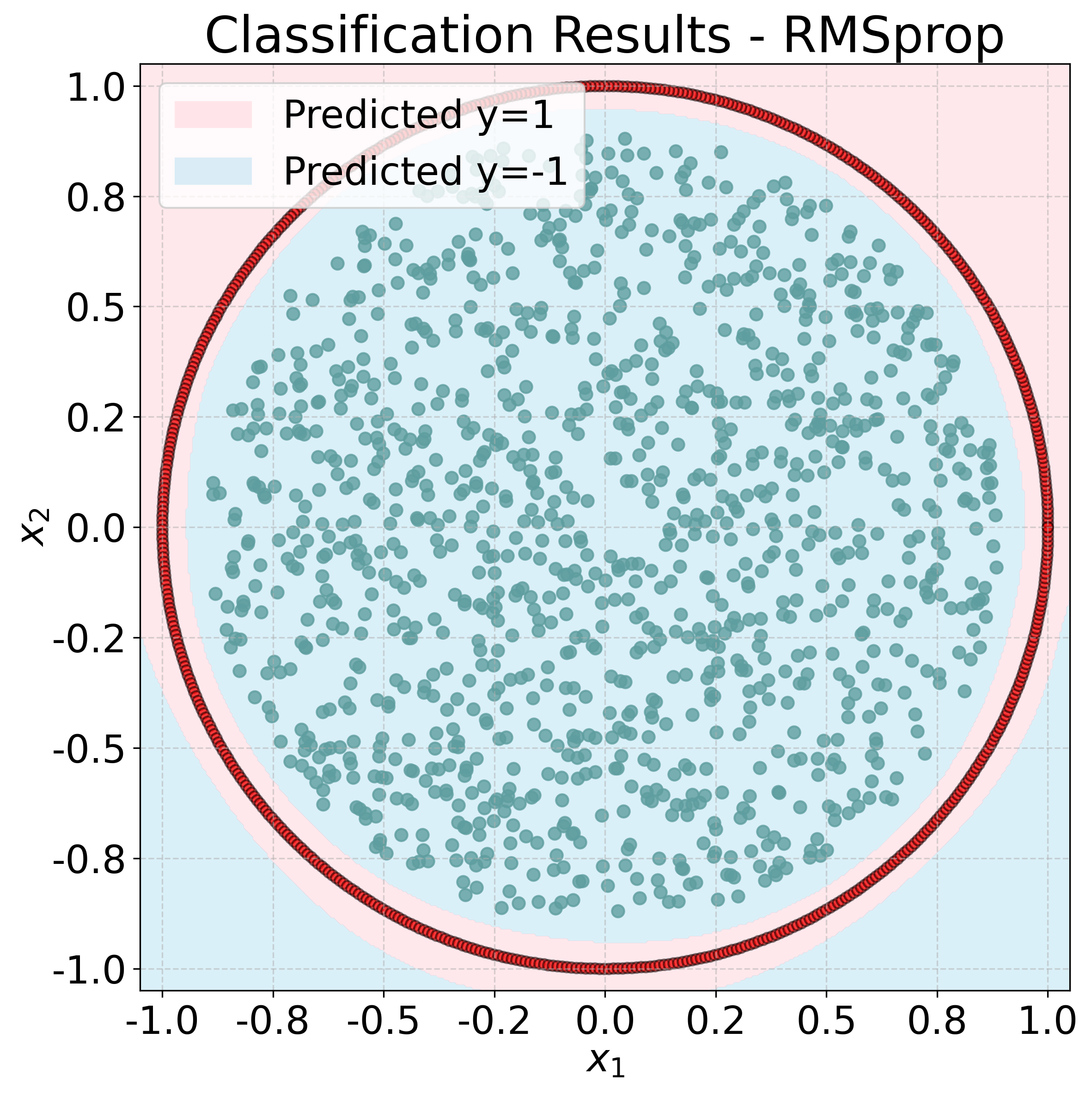}}

\subfigure[]{
\label{nearest-SGD}
\includegraphics[width=0.23\textwidth]{figures/circle/nearest_SGD.png}}
\subfigure[]{
\label{nearest - SGDm}
\includegraphics[width=0.23\textwidth]{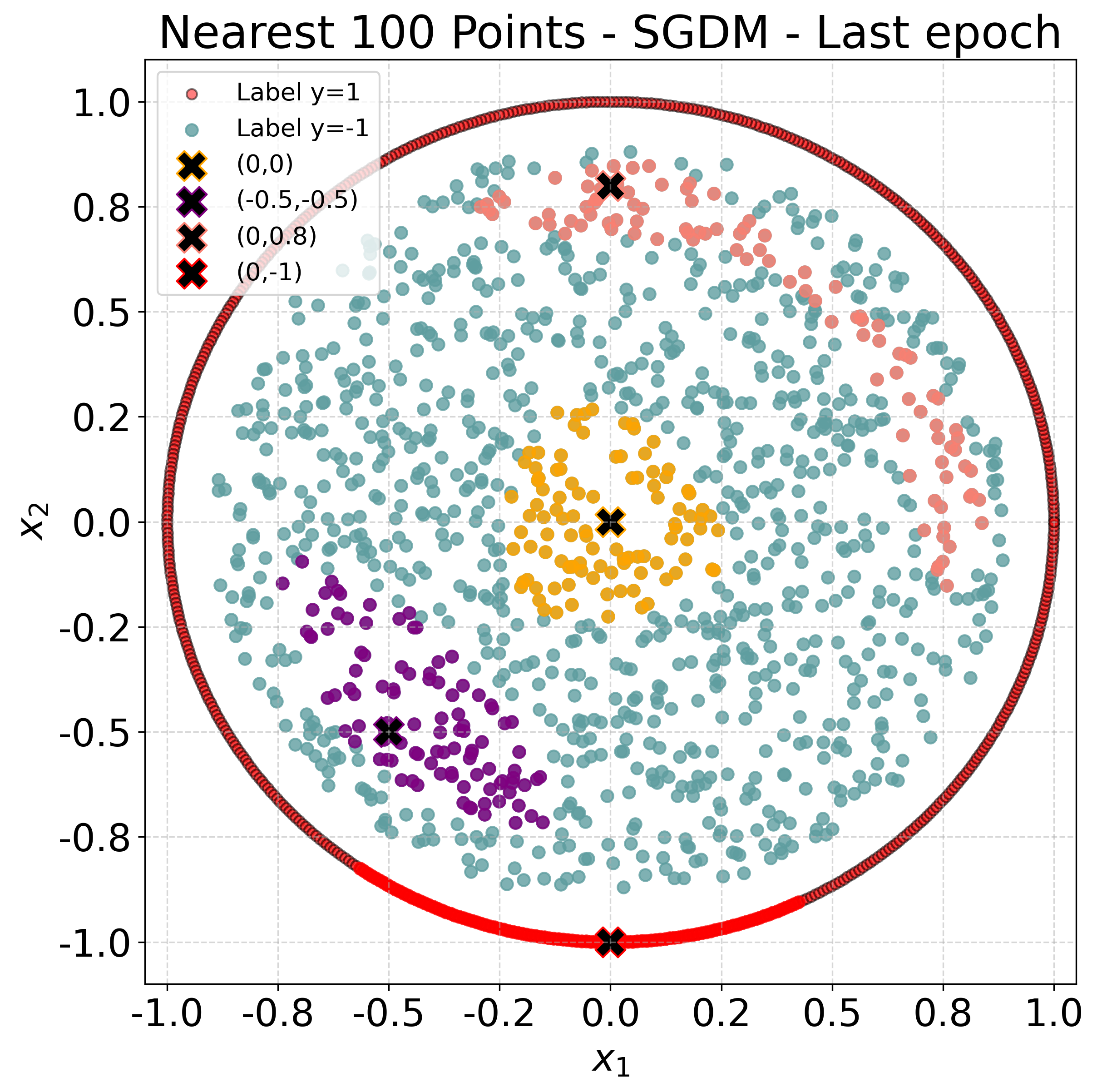}}
\subfigure[]{
\label{nearest-adam}
\includegraphics[width=0.23\textwidth]{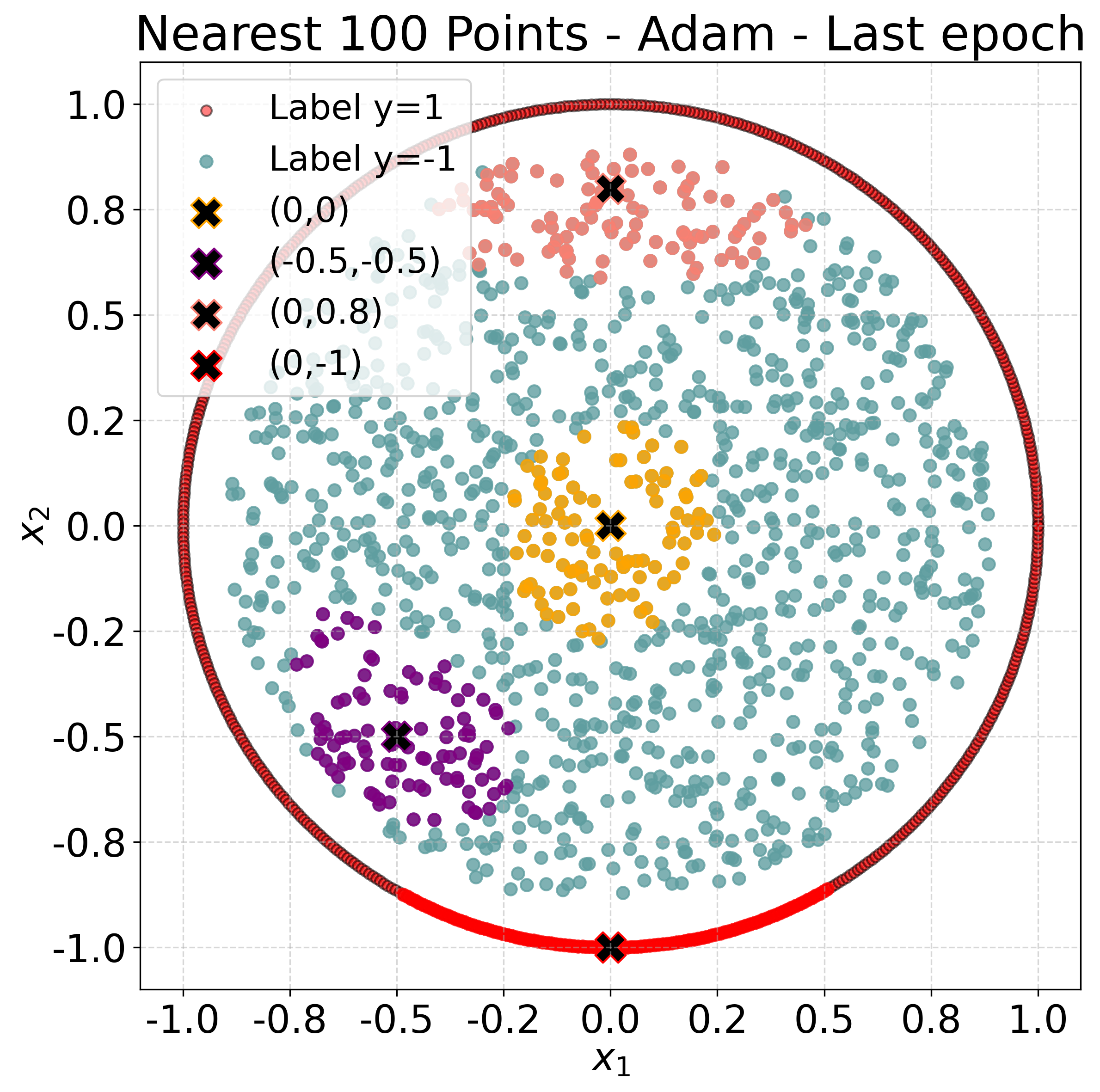}}
\subfigure[]{
\label{nearest-rms}
\includegraphics[width=0.23\textwidth]{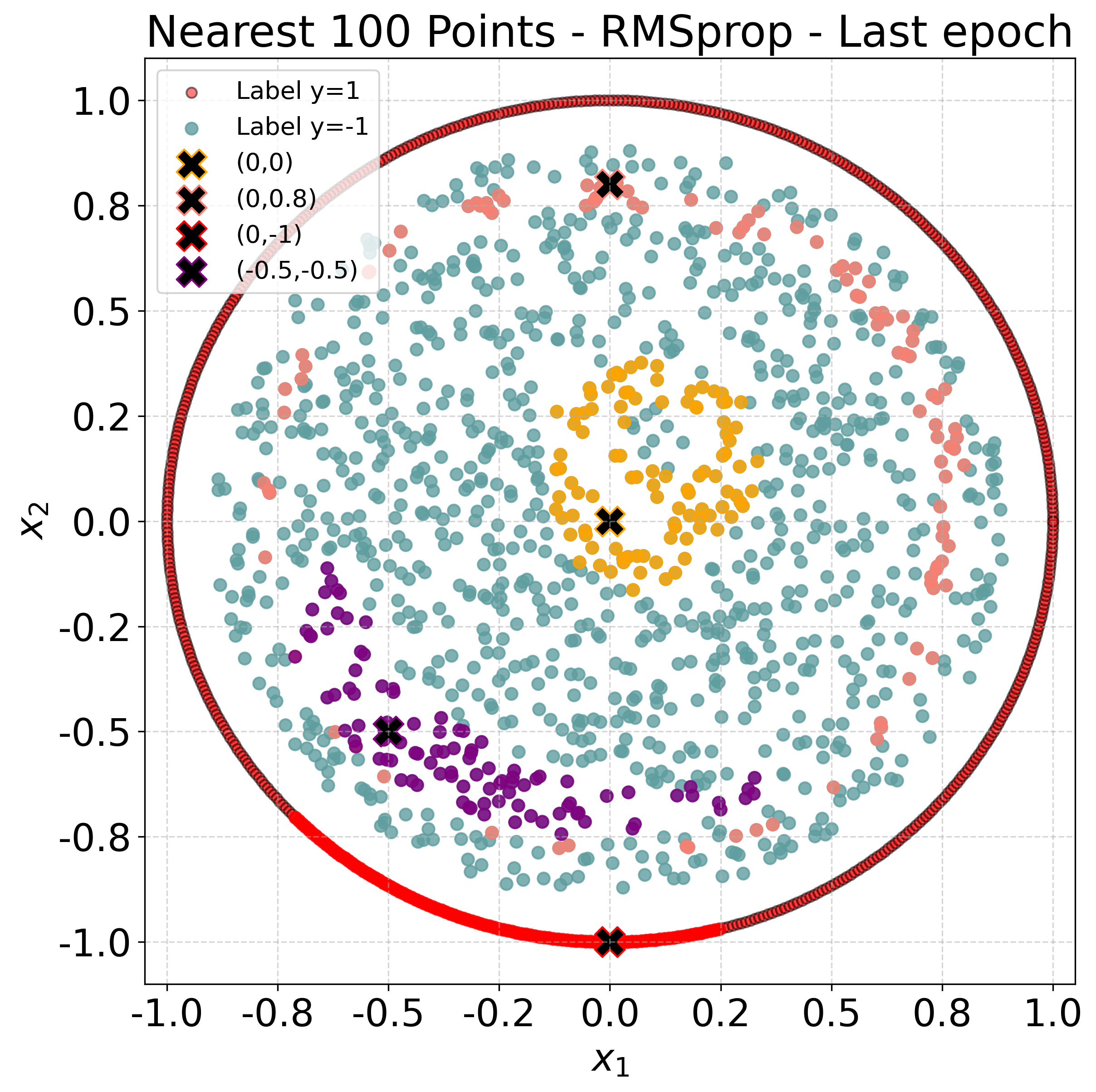}}
\caption{Figure (a)(b)(c)(d) are the classification results for SGD, SGDM, Adam and RMSprop. The second row shows the nearest 100 points under the tangent space generated by SGD, SGDM, Adam and RMSprop (Figure (e)-(h)). }
\label{preiction and nearest}
\end{figure}
Figure \ref{preiction and nearest} shows the prediction results for different optimizers (Figure (\ref{Classification results - SGD},\ref{Classification results - SGDM},\ref{Classification result - adam},\ref{Classification result - rms}) under training data (Figure \ref{visualization}). The blue points satisfy $x_1^2+x_2^2\leq 0.9$ with the label $-1$, the dark red points satisfy $x_1^2+x_2^2=1$. The second row in Figure \ref{preiction and nearest} illustrates which points were considered as similar from the tangent space perspective; we computed the 100 points with the largest similarity according to the normalized NTK i.e. $\frac{\langle \nabla_\bTheta f(x_1,\bTheta),\nabla_\bTheta f(x_2,\bTheta) \rangle}{\|\nabla_\bTheta f(x_1,\bTheta)\|\|\nabla_\bTheta f(x_2,\bTheta)\|}$. That means the data which are close to each other in Euclidean space may not similar in neural network perspective. In Figure \ref{nearest-SGD},\ref{nearest - SGDm},\ref{nearest-adam},\ref{nearest-rms}, the red points represent the 100 closest points to $(0,-1)$, the pink points represent the 100 closest points to $(0,0.8)$. When we measure the distance using the Euclidean metric, the nearest points to $(0,-1)$ are partly in $x_1^2+x_2^2=1$ with the label $1$, partly in $x_1^2+x_2^2\leq 0.9$ with the label $-1$. From the neural network's perspective, the closest points to $(0,-1)$ have the same label $1$ instead.

Figure~\ref{evolution ntk} tracks $\widehat{K}_t(x,y)$ over training for $x=(0,0)$ (label $-1$) and several choices of $y$ in the binary classification task in Figure~\ref{visualization}, across SGD, SGDM, RMSprop, and Adam. The self-pair $y=x$ remains at $\widehat K_t\equiv 1$ by construction, and the same-class point $y=(0.5,0)$ quickly increases to and remains close to $1$ for all optimizers, indicating persistent gradient alignment and a large contribution in the Domingos path-kernel correction at $x$. By contrast, the cross-class points $y=(1.0,0)$ and $y=(1.2,0)$ (label $+1$) rapidly lose alignment: their $\widehat{K}_t$ shrinks toward zero and may turn negative, so their Domingos weights become ineffective or actively oppose the correction at $x$. The SGDM produces the fastest separation, whereas SGD shows slower reorganization; RMSprop and Adam exhibit additional non-monotone transients, reflecting optimizer-specific effects on how gradient geometry evolves.

\begin{figure}[ht]
\centering
\subfigure[]{
\label{sgd_ntk_plot}
\includegraphics[width=0.45\textwidth]{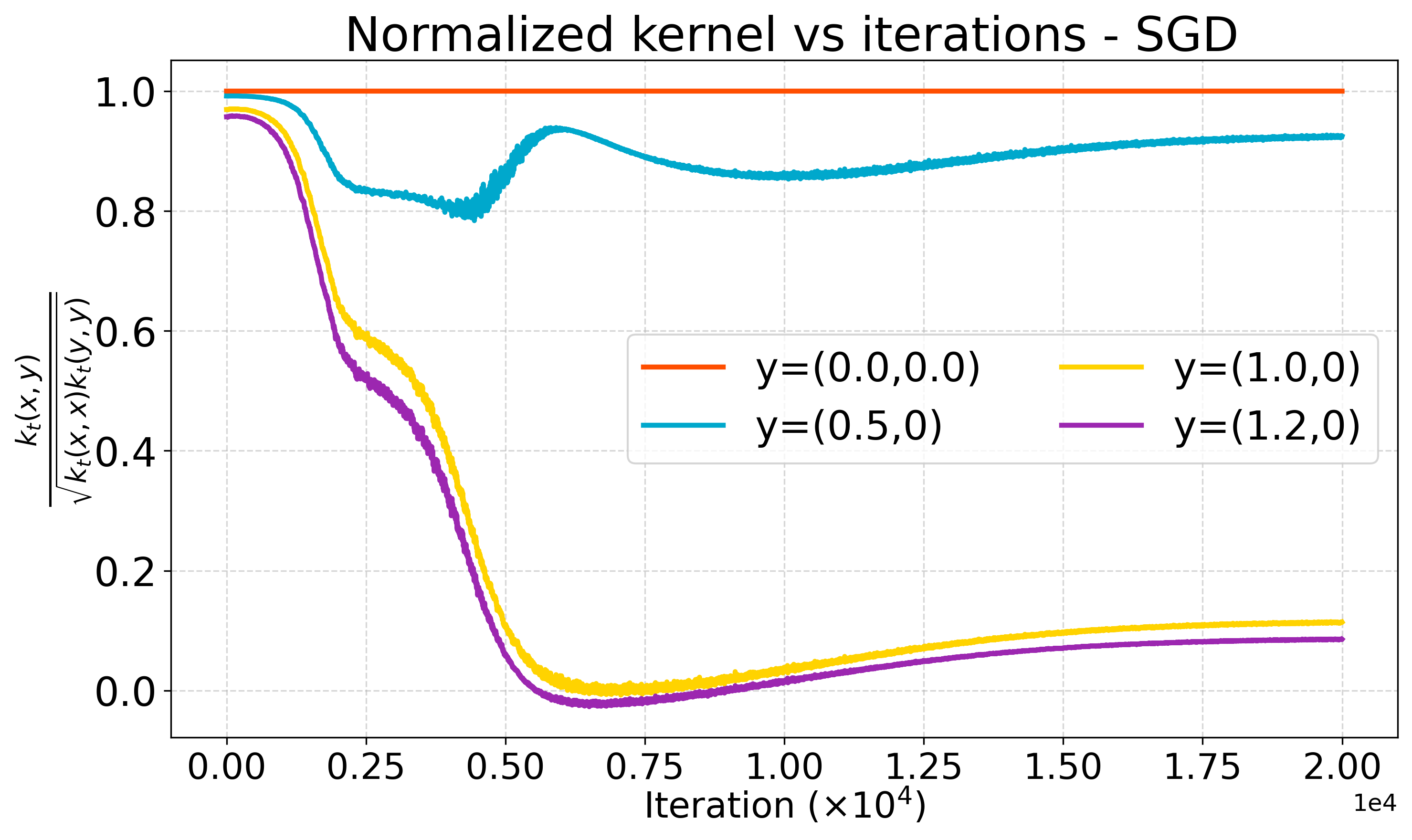}}
\subfigure[]{
\label{sgdm_ntk_plot}
\includegraphics[width=0.45\textwidth]{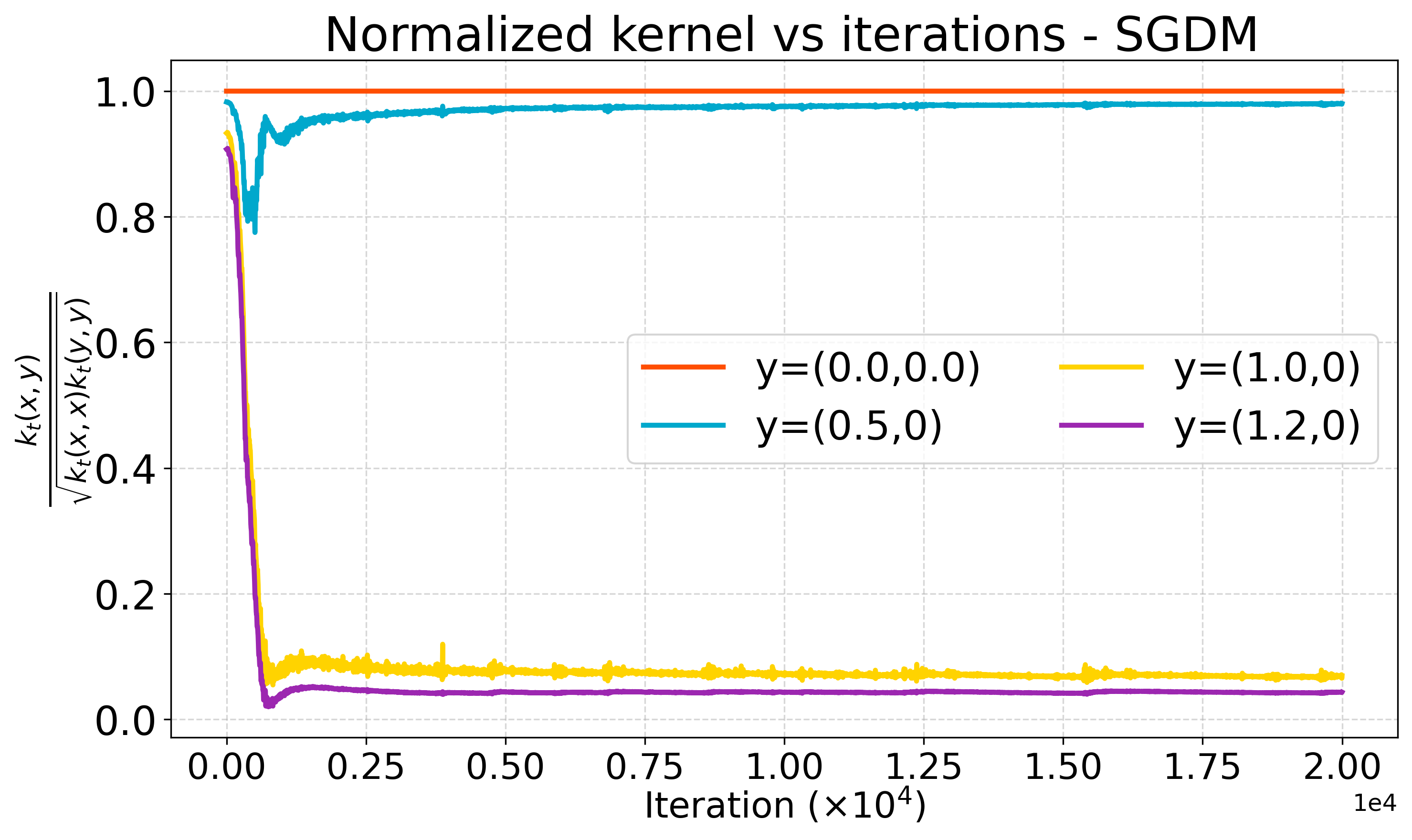}}

\subfigure[]{
\label{rms_ntk_plot}
\includegraphics[width=0.45\textwidth]{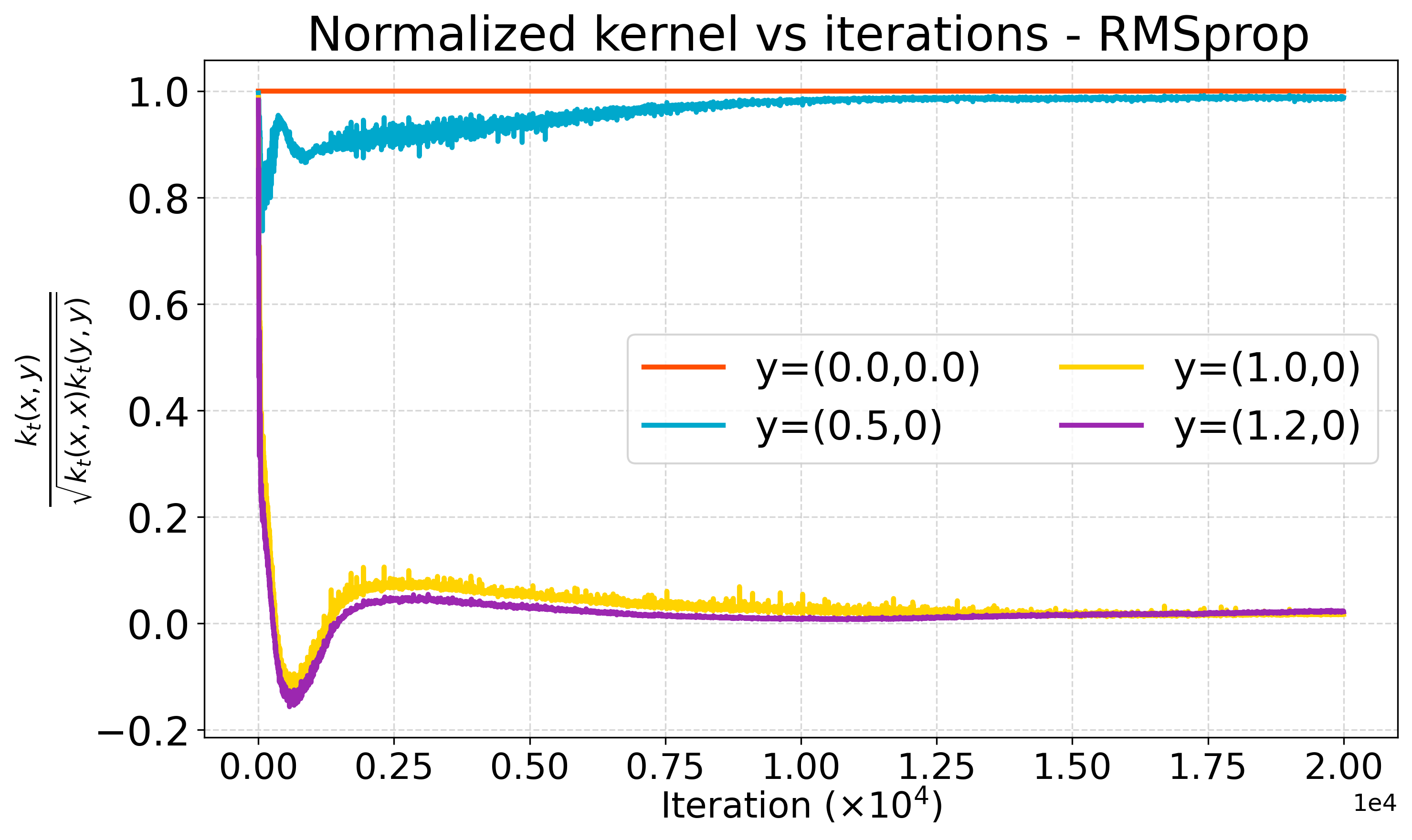}}
\subfigure[]{
\label{adam_ntk_plot}
\includegraphics[width=0.45\textwidth]{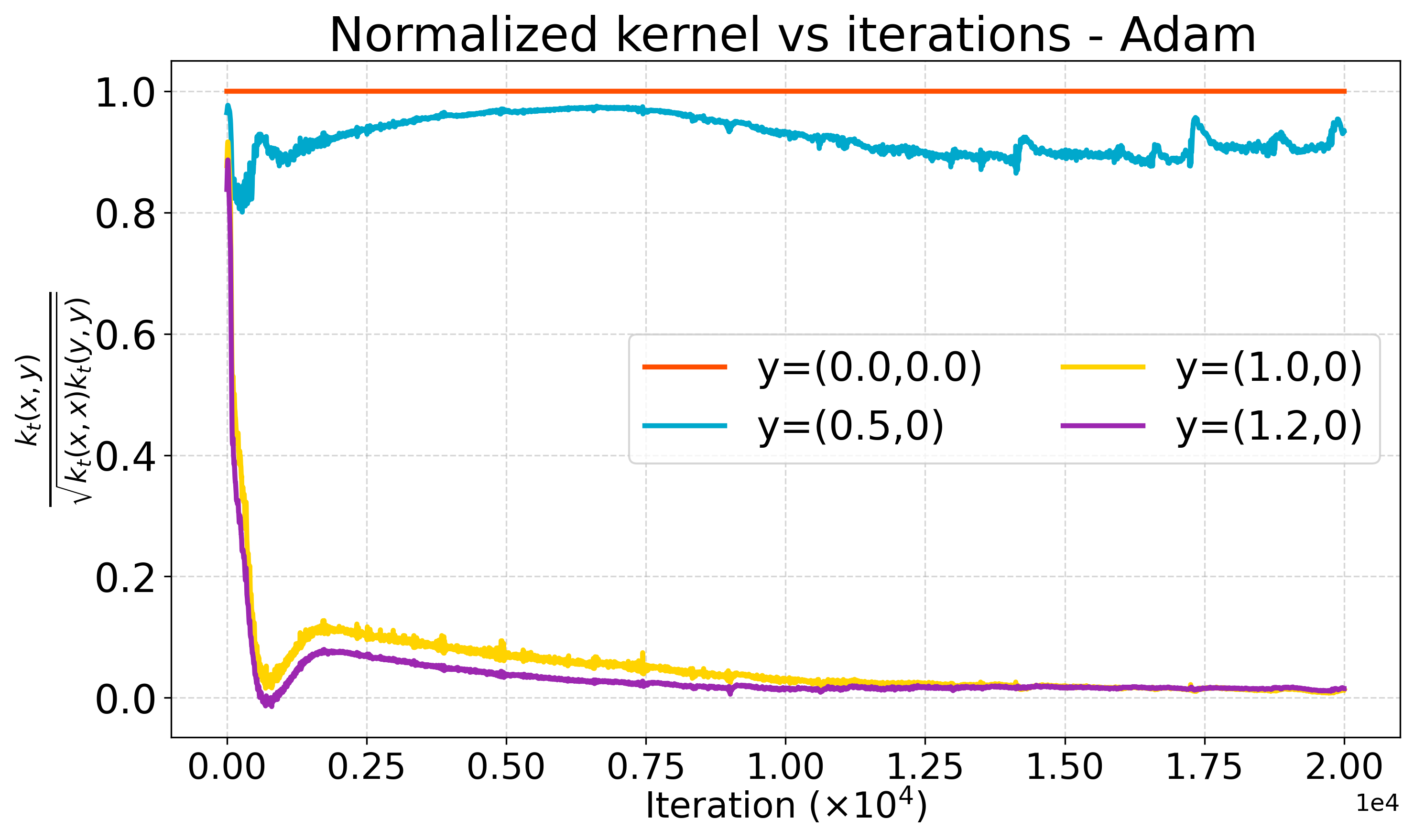}}
    \caption{Evolution of the normalized gradient kernel $\widehat{K}_t(x,y)=K_t(x,y)/\sqrt{K_t(x,x)K_t(y,y)}$ over training in the binary classification task in Figure~\ref{visualization}. We fix $x=(0,0)$ with label $-1$ and track $y=\{(0,0),(0.5,0),(1.0,0),(1.2,0)\}$, where $(0.5,0)$ has label $-1$ while $y=(1.0,0)$ and $y=(1.2,0)$ have label $+1$. Same-label pairs maintain high NTK similarity (near $1$), while opposite-label pairs rapidly decay toward $0$ and may become negative, indicating persistent gradient anti-alignment across classes. The transient fluctuations and steady-state behaviors differ across (a) SGD, (b) SGDM, (c) RMSprop, and (d) Adam.}
    \label{evolution ntk}
\end{figure}

\appendix

\bibliographystyle{plain}
\bibliography{sample}

@inproceedings{mandt2016variational,
  title={A variational analysis of stochastic gradient algorithms},
  author={Mandt, Stephan and Hoffman, Matthew and Blei, David},
  booktitle={International conference on machine learning},
  pages={354--363},
  year={2016},
  organization={PMLR}
}

@article{song2019generative,
  title={Generative modeling by estimating gradients of the data distribution},
  author={Song, Yang and Ermon, Stefano},
  journal={Advances in neural information processing systems},
  volume={32},
  year={2019}
}

@article{loshchilov2016sgdr,
  title={{SGDR}: Stochastic gradient descent with warm restarts},
  author={Loshchilov, Ilya and Hutter, Frank},
  journal={arXiv preprint arXiv:1608.03983},
  year={2016}
}

@article{li2019stochastic,
  title={Stochastic modified equations and dynamics of stochastic gradient algorithms i: Mathematical foundations},
  author={Li, Qianxiao and Tai, Cheng and Weinan, E},
  journal={Journal of Machine Learning Research},
  volume={20},
  number={40},
  pages={1--47},
  year={2019}
}

@book{oksendal2013stochastic,
  title={Stochastic differential equations: an introduction with applications},
  author={Oksendal, Bernt},
  year={2013},
  publisher={Springer Science \& Business Media}
}

@article{domingos2020every,
  title={Every model learned by gradient descent is approximately a kernel machine},
  author={Domingos, Pedro},
  journal={arXiv preprint arXiv:2012.00152},
  year={2020}
}

@article{ho2020denoising,
  title={Denoising Diffusion Probabilistic Models},
  author={Jonathan Ho and Ajay Jain and Pieter Abbeel},
  year={2020},
  journal={arXiv preprint arxiv:2006.11239}
}

@article{fort2020deep,
  title={Deep learning versus kernel learning: an empirical study of loss landscape geometry and the time evolution of the neural tangent kernel},
  author={Fort, Stanislav and Dziugaite, Gintare Karolina and Paul, Mansheej and Kharaghani, Sepideh and Roy, Daniel M and Ganguli, Surya},
  journal={Advances in Neural Information Processing Systems},
  volume={33},
  pages={5850--5861},
  year={2020}
}

@article{lee2019wide,
  title={Wide neural networks of any depth evolve as linear models under gradient descent},
  author={Lee, Jaehoon and Xiao, Lechao and Schoenholz, Samuel and Bahri, Yasaman and Novak, Roman and Sohl-Dickstein, Jascha and Pennington, Jeffrey},
  journal={Advances in neural information processing systems},
  volume={32},
  year={2019}
}

@article{jacot2018neural,
  title={Neural tangent kernel: Convergence and generalization in neural networks},
  author={Jacot, Arthur and Gabriel, Franck and Hongler, Cl{\'e}ment},
  journal={Advances in neural information processing systems},
  volume={31},
  year={2018}
}

@article{zhang2016understanding,
  title={Understanding deep learning requires rethinking generalization},
  author={Zhang, Chiyuan and Bengio, Samy and Hardt, Moritz and Recht, Benjamin and Vinyals, Oriol},
  journal={arXiv preprint arXiv:1611.03530},
  year={2016}
}

@InProceedings{belkin2018understand,
  title = 	 {To Understand Deep Learning We Need to Understand Kernel Learning},
  author =       {Belkin, Mikhail and Ma, Siyuan and Mandal, Soumik},
  booktitle = 	 {Proceedings of the 35th International Conference on Machine Learning},
  pages = 	 {541--549},
  year = 	 {2018},
  editor = 	 {Dy, Jennifer and Krause, Andreas},
  volume = 	 {80},
  series = 	 {Proceedings of Machine Learning Research},
  month = 	 {10--15 Jul},
  publisher =    {PMLR}
}

@article{lecun2015deep,
  title={Deep learning},
  author={LeCun, Yann and Bengio, Yoshua and Hinton, Geoffrey},
  journal={nature},
  volume={521},
  number={7553},
  pages={436--444},
  year={2015},
  publisher={Nature Publishing Group UK London}
}

@article{schmidhuber2015deep,
  title={Deep learning in neural networks: An overview},
  author={Schmidhuber, J{\"u}rgen},
  journal={Neural networks},
  volume={61},
  pages={85--117},
  year={2015},
  publisher={Elsevier}
}

@inproceedings{ribeiro2016should,
  title={"{W}hy should {I} trust you?" Explaining the predictions of any classifier},
  author={Ribeiro, Marco Tulio and Singh, Sameer and Guestrin, Carlos},
  booktitle={Proceedings of the 22nd ACM SIGKDD international conference on knowledge discovery and data mining},
  pages={1135--1144},
  year={2016}
}

@inproceedings{hardt2016train,
  title={Train faster, generalize better: Stability of stochastic gradient descent},
  author={Hardt, Moritz and Recht, Ben and Singer, Yoram},
  booktitle={International conference on machine learning},
  pages={1225--1234},
  year={2016},
  organization={PMLR}
}

@article{d8d62392-9a37-31e7-ad3b-37a6f6ee8ef6,
 ISSN = {00034851},
 URL = {http://www.jstor.org/stable/2236626},
 author = {Herbert Robbins and Sutton Monro},
 journal = {The Annals of Mathematical Statistics},
 number = {3},
 pages = {400--407},
 publisher = {Institute of Mathematical Statistics},
 title = {A Stochastic Approximation Method},
 urldate = {2025-03-26},
 volume = {22},
 year = {1951}
}

@article{schmidt2020nonparametric,
  title={Nonparametric regression using deep neural networks with {ReLU} activation function},
  author={Schmidt-Hieber, Johannes},
  journal = {The Annals of Statistics},
  volume = {48},
  number = {4},
  doi = {10.1214/19-AOS1875},
  year={2020}
}

@article{suzuki2021deep,
  title={Deep learning is adaptive to intrinsic dimensionality of model smoothness in anisotropic Besov space},
  author={Suzuki, Taiji and Nitanda, Atsushi},
  journal={Advances in Neural Information Processing Systems},
  volume={34},
  pages={3609--3621},
  year={2021}
}

@InProceedings{pmlr-v40-Frostig15,
  title = 	 {Competing with the Empirical Risk Minimizer in a Single Pass},
  author = 	 {Frostig, Roy and Ge, Rong and Kakade, Sham M. and Sidford, Aaron},
  booktitle = 	 {Proceedings of The 28th Conference on Learning Theory},
  pages = 	 {728--763},
  year = 	 {2015},
  editor = 	 {Grünwald, Peter and Hazan, Elad and Kale, Satyen},
  volume = 	 {40},
  series = 	 {Proceedings of Machine Learning Research},
  address = 	 {Paris, France},
  month = 	 {03--06 Jul},
  publisher =    {PMLR},
  url = 	 {https://proceedings.mlr.press/v40/Frostig15.html}
}

@article{doi:10.1137/070704277,
  title={Robust stochastic approximation approach to stochastic programming},
  author={Nemirovski, Arkadi and Juditsky, Anatoli and Lan, Guanghui and Shapiro, Alexander},
  journal={SIAM Journal on optimization},
  volume={19},
  number={4},
  pages={1574--1609},
  year={2009},
  publisher={SIAM}
}

@InProceedings{10.1007/11776420_37,
author="Hazan, Elad
and Kalai, Adam
and Kale, Satyen
and Agarwal, Amit",
editor="Lugosi, G{\'a}bor
and Simon, Hans Ulrich",
title="Logarithmic Regret Algorithms for Online Convex Optimization",
booktitle="Learning Theory",
year="2006",
publisher="Springer Berlin Heidelberg",
address="Berlin, Heidelberg",
pages="499--513",
isbn="978-3-540-35296-9"
}

@article{hazan2014beyond,
  title={Beyond the regret minimization barrier: optimal algorithms for stochastic strongly-convex optimization},
  author={Hazan, Elad and Kale, Satyen},
  journal={The Journal of Machine Learning Research},
  volume={15},
  number={1},
  pages={2489--2512},
  year={2014},
  publisher={JMLR. org}
}

@inproceedings{rakhlin2012making,
  title={Making gradient descent optimal for strongly convex stochastic optimization},
  author={Rakhlin, Alexander and Shamir, Ohad and Sridharan, Karthik},
  booktitle={29th International Conference on Machine Learning, ICML 2012},
  pages={449--456},
  year={2012}
}

@inproceedings{sutskever2013importance,
  title={On the importance of initialization and momentum in deep learning},
  author={Sutskever, Ilya and Martens, James and Dahl, George and Hinton, Geoffrey},
  booktitle={International conference on machine learning},
  pages={1139--1147},
  year={2013},
  organization={PMLR}
}

@article{defazio2020understanding,
  title={Understanding the role of momentum in non-convex optimization: Practical insights from a {L}yapunov analysis},
  author={Defazio, Aaron},
  journal={arXiv preprint arXiv:2010.00406},
  year={2020}
}

@article{leclerc2020two,
  title={The two regimes of deep network training},
  author={Leclerc, Guillaume and Madry, Aleksander},
  journal={arXiv preprint arXiv:2002.10376},
  year={2020}
}

@article{duchi2011adaptive,
  title={Adaptive subgradient methods for online learning and stochastic optimization.},
  author={Duchi, John and Hazan, Elad and Singer, Yoram},
  journal={Journal of machine learning research},
  volume={12},
  number={7},
  year={2011}
}

@article{tieleman2012lecture,
  title={Lecture 6.5-rmsprop, coursera: Neural networks for machine learning},
  author={Tieleman, Tijmen and Hinton, Geoffrey},
  journal={University of Toronto, Technical Report},
  volume={6},
  pages={2012},
  year={2012}
}

@article{kingma2014adam,
  title={Adam: A method for stochastic optimization},
  author={Kingma, Diederik P and Ba, Jimmy},
  journal={arXiv preprint arXiv:1412.6980},
  year={2014}
}

@article{harutyunyan2019multitask,
  title={Multitask learning and benchmarking with clinical time series data},
  author={Harutyunyan, Hrayr and Khachatrian, Hrant and Kale, David C and Ver Steeg, Greg and Galstyan, Aram},
  journal={Scientific data},
  volume={6},
  number={1},
  pages={96},
  year={2019},
  publisher={Nature Publishing Group UK London}
}

@inproceedings{xu2015show,
  title={Show, attend and tell: Neural image caption generation with visual attention},
  author={Xu, Kelvin and Ba, Jimmy and Kiros, Ryan and Cho, Kyunghyun and Courville, Aaron and Salakhudinov, Ruslan and Zemel, Rich and Bengio, Yoshua},
  booktitle={International conference on machine learning},
  pages={2048--2057},
  year={2015},
  organization={PMLR}
}

@article{oh2017value,
  title={Value prediction network},
  author={Oh, Junhyuk and Singh, Satinder and Lee, Honglak},
  journal={Advances in neural information processing systems},
  volume={30},
  year={2017}
}

@article{malladi2022sdes,
  title={On the {SDE}s and scaling rules for adaptive gradient algorithms},
  author={Malladi, Sadhika and Lyu, Kaifeng and Panigrahi, Abhishek and Arora, Sanjeev},
  journal={Advances in Neural Information Processing Systems},
  volume={35},
  pages={7697--7711},
  year={2022}
}

@article{wang2023marginal,
  title={The marginal value of momentum for small learning rate {SGD}},
  author={Wang, Runzhe and Malladi, Sadhika and Wang, Tianhao and Lyu, Kaifeng and Li, Zhiyuan},
  journal={arXiv preprint arXiv:2307.15196},
  year={2023}
}

@article{novak2019neural,
  title={Neural tangents: Fast and easy infinite neural networks in python},
  author={Novak, Roman and Xiao, Lechao and Hron, Jiri and Lee, Jaehoon and Alemi, Alexander A and Sohl-Dickstein, Jascha and Schoenholz, Samuel S},
  journal={arXiv preprint arXiv:1912.02803},
  year={2019}
}

@article{bengio2013representation,
  title={Representation learning: A review and new perspectives},
  author={Bengio, Yoshua and Courville, Aaron and Vincent, Pascal},
  journal={IEEE transactions on pattern analysis and machine intelligence},
  volume={35},
  number={8},
  pages={1798--1828},
  year={2013},
  publisher={IEEE}
}

@article{lecun1998gradient,
  title={Gradient-based learning applied to document recognition},
  author={LeCun, Yann and Bottou, L{\'e}on and Bengio, Yoshua and Haffner, Patrick},
  journal={Proceedings of the IEEE},
  volume={86},
  number={11},
  pages={2278--2324},
  year={1998},
  publisher={Ieee}
}

@inproceedings{chaudhari2018stochastic,
  title={Stochastic gradient descent performs variational inference, converges to limit cycles for deep networks},
  author={Chaudhari, Pratik and Soatto, Stefano},
  booktitle={2018 Information Theory and Applications Workshop (ITA)},
  pages={1--10},
  year={2018},
  organization={IEEE}
}

@article{orvieto2019continuous,
  title={Continuous-time models for stochastic optimization algorithms},
  author={Orvieto, Antonio and Lucchi, Aurelien},
  journal={Advances in Neural Information Processing Systems},
  volume={32},
  year={2019}
}

@article{vaswani2017attention,
  title={Attention is all you need},
  author={Vaswani, Ashish and Shazeer, Noam and Parmar, Niki and Uszkoreit, Jakob and Jones, Llion and Gomez, Aidan N and Kaiser, {\L}ukasz and Polosukhin, Illia},
  journal={Advances in neural information processing systems},
  volume={30},
  year={2017}
}

@article{bartlett2002rademacher,
  title={Rademacher and {G}aussian complexities: Risk bounds and structural results},
  author={Bartlett, Peter L and Mendelson, Shahar},
  journal={Journal of Machine Learning Research},
  volume={3},
  number={Nov},
  pages={463--482},
  year={2002}
}

@article{arora2019exact,
  title={On exact computation with an infinitely wide neural net},
  author={Arora, Sanjeev and Du, Simon S and Hu, Wei and Li, Zhiyuan and Salakhutdinov, Russ R and Wang, Ruosong},
  journal={Advances in neural information processing systems},
  volume={32},
  year={2019}
}

@book{kloeden2013numerical,
  title={Numerical Solution of Stochastic Differential Equations},
  author={Kloeden, P.E. and Platen, E.},
  isbn={9783662126165},
  lccn={92015916},
  series={Stochastic Modelling and Applied Probability},
  url={https://books.google.com.hk/books?id=r9r6CAAAQBAJ},
  year={2013},
  publisher={Springer Berlin Heidelberg}
}

@article{feng2019uniform,
  title={Uniform-in-time weak error analysis for stochastic gradient descent algorithms via diffusion approximation},
  author={Feng, Yuanyuan and Gao, Tingran and Li, Lei and Liu, Jian-Guo and Lu, Yulong},
  journal={arXiv preprint arXiv:1902.00635},
  year={2019}
}

@article{aiudi2025local,
  title={Local kernel renormalization as a mechanism for feature learning in overparametrized convolutional neural networks},
  author={Aiudi, R and Pacelli, R and Baglioni, P and Vezzani, A and Burioni, R and Rotondo, P},
  journal={Nature Communications},
  volume={16},
  number={1},
  pages={568},
  year={2025},
  publisher={Nature Publishing Group UK London}
}

@article{courtois2023can,
  title={Can neural networks extrapolate? Discussion of a theorem by Pedro Domingos},
  author={Courtois, Adrien and Morel, Jean-Michel and Arias, Pablo},
  journal={Revista de la Real Academia de Ciencias Exactas, F{\'\i}sicas y Naturales. Serie A. Matem{\'a}ticas},
  volume={117},
  number={2},
  pages={79},
  year={2023},
  publisher={Springer}
}

@article{hu2017diffusion,
  title={On the diffusion approximation of nonconvex stochastic gradient descent},
  author={Hu, Wenqing and Li, Chris Junchi and Li, Lei and Liu, Jian-Guo},
  journal={arXiv preprint arXiv:1705.07562},
  year={2017}
}

@article{song2020score,
  title={Score-based generative modeling through stochastic differential equations},
  author={Song, Yang and Sohl-Dickstein, Jascha and Kingma, Diederik P and Kumar, Abhishek and Ermon, Stefano and Poole, Ben},
  journal={arXiv preprint arXiv:2011.13456},
  year={2020}
}

@article{vapnik1991necessary,
  title={The necessary and sufficient conditions for consistency in the empirical risk minimization method},
  author={Vapnik, Vladimir and Chervonenkis, Alexey},
  journal={Pattern Recognition and Image Analysis},
  volume={1},
  number={3},
  pages={283--305},
  year={1991}
}

@inproceedings{oymak2019overparameterized,
  title={Overparameterized nonlinear learning: Gradient descent takes the shortest path?},
  author={Oymak, Samet and Soltanolkotabi, Mahdi},
  booktitle={International Conference on Machine Learning},
  pages={4951--4960},
  year={2019},
  organization={PMLR}
}

@article{goodfellow2020generative,
  title={Generative adversarial networks},
  author={Goodfellow, Ian and Pouget-Abadie, Jean and Mirza, Mehdi and Xu, Bing and Warde-Farley, David and Ozair, Sherjil and Courville, Aaron and Bengio, Yoshua},
  journal={Communications of the ACM},
  volume={63},
  number={11},
  pages={139--144},
  year={2020},
  publisher={ACM New York, NY, USA}
}

@article{arjovsky2017towards,
  title={Towards principled methods for training generative adversarial networks},
  author={Arjovsky, Martin and Bottou, L{\'e}on},
  journal={arXiv preprint arXiv:1701.04862},
  year={2017}
}

@book{berlinet2011reproducing,
  title={Reproducing {K}ernel {H}ilbert {S}paces in probability and statistics},
  author={Berlinet, Alain and Thomas-Agnan, Christine},
  year={2011},
  publisher={Springer Science \& Business Media}
}

@article{tieleman2012rmsprop,
  title={Rmsprop: Divide the gradient by a running average of its recent magnitude. coursera: Neural networks for machine learning},
  author={Tieleman, Tijmen and Hinton, Geoffrey},
  journal={COURSERA Neural Networks Mach. Learn},
  volume={17},
  pages={6},
  year={2012}
}

@article{compagnoni2024adaptive,
  title={Adaptive methods through the lens of {SDE}s: Theoretical insights on the role of noise},
  author={Compagnoni, Enea Monzio and Liu, Tianlin and Islamov, Rustem and Proske, Frank Norbert and Orvieto, Antonio and Lucchi, Aurelien},
  journal={arXiv preprint arXiv:2411.15958},
  year={2024}
}

@inproceedings{mnih2016asynchronous,
  title={Asynchronous methods for deep reinforcement learning},
  author={Mnih, Volodymyr and Badia, Adria Puigdomenech and Mirza, Mehdi and Graves, Alex and Lillicrap, Timothy and Harley, Tim and Silver, David and Kavukcuoglu, Koray},
  booktitle={International conference on machine learning},
  pages={1928--1937},
  year={2016},
  organization={PmLR}
}

@inproceedings{yaz2018unusual,
  title={The unusual effectiveness of averaging in {GAN} training},
  author={Yaz, Yasin and Foo, Chuan-Sheng and Winkler, Stefan and Yap, Kim-Hui and Piliouras, Georgios and Chandrasekhar, Vijay and others},
  booktitle={International Conference on Learning Representations},
  year={2018}
}

@article{che2016mode,
  title={Mode regularized generative adversarial networks},
  author={Che, Tong and Li, Yanran and Jacob, Athul Paul and Bengio, Yoshua and Li, Wenjie},
  journal={arXiv preprint arXiv:1612.02136},
  year={2016}
}

@article{cybenko1989approximation,
  title={Approximation by superpositions of a {S}igmoidal function},
  author={Cybenko, George},
  journal={Mathematics of control, signals and systems},
  volume={2},
  number={4},
  pages={303--314},
  year={1989},
  publisher={Springer}
}

@article{hornik1989multilayer,
  title={Multilayer feedforward networks are universal approximators},
  author={Hornik, Kurt and Stinchcombe, Maxwell and White, Halbert},
  journal={Neural networks},
  volume={2},
  number={5},
  pages={359--366},
  year={1989},
  publisher={Elsevier}
}

@inproceedings{cohen2016expressive,
  title={On the expressive power of deep learning: A tensor analysis},
  author={Cohen, Nadav and Sharir, Or and Shashua, Amnon},
  booktitle={Conference on learning theory},
  pages={698--728},
  year={2016},
  organization={PMLR}
}

@article{hanin2017approximating,
  title={Approximating continuous functions by {R}elu nets of minimal width},
  author={Hanin, Boris and Sellke, Mark},
  journal={arXiv preprint arXiv:1710.11278},
  year={2017}
}

@article{44f6fc04-5e81-3369-9330-afe10d4a9993,
 ISSN = {00905364, 21688966},
 URL = {https://www.jstor.org/stable/26931513},
 author = {Tengyuan Liang and Alexander Rakhlin},
 journal = {The Annals of Statistics},
 number = {3},
 pages = {pp. 1329--1347},
 publisher = {Institute of Mathematical Statistics},
 title = {JUST INTERPOLATE: KERNEL “RIDGELESS” REGRESSION CAN GENERALIZE},
 urldate = {2025-09-04},
 volume = {48},
 year = {2020}
}

@inproceedings{belkin2019does,
  title={Does data interpolation contradict statistical optimality?},
  author={Belkin, Mikhail and Rakhlin, Alexander and Tsybakov, Alexandre B},
  booktitle={The 22nd international conference on artificial intelligence and statistics},
  pages={1611--1619},
  year={2019},
  organization={PMLR}
}

@article{belkin2021fit,
  title={Fit without fear: Remarkable mathematical phenomena of deep learning through the prism of interpolation},
  author={Belkin, Mikhail},
  journal={Acta Numerica},
  volume={30},
  pages={203--248},
  year={2021},
  publisher={Cambridge University Press}
}

@article{bartlett2020benign,
  title={Benign overfitting in linear regression},
  author={Bartlett, Peter L and Long, Philip M and Lugosi, G{\'a}bor and Tsigler, Alexander},
  journal={Proceedings of the National Academy of Sciences},
  volume={117},
  number={48},
  pages={30063--30070},
  year={2020},
  publisher={National Academy of Sciences}
}

@inproceedings{frei2022benign,
  title={Benign overfitting without linearity: Neural network classifiers trained by gradient descent for noisy linear data},
  author={Frei, Spencer and Chatterji, Niladri S and Bartlett, Peter},
  booktitle={Conference on Learning Theory},
  pages={2668--2703},
  year={2022},
  organization={PMLR}
}

@article{mallinar2022benign,
  title={Benign, tempered, or catastrophic: Toward a refined taxonomy of overfitting},
  author={Mallinar, Neil and Simon, James and Abedsoltan, Amirhesam and Pandit, Parthe and Belkin, Misha and Nakkiran, Preetum},
  journal={Advances in Neural Information Processing Systems},
  volume={35},
  pages={1182--1195},
  year={2022}
}

@article{montanari2022interpolation,
  title={The interpolation phase transition in neural networks: Memorization and generalization under lazy training},
  author={Montanari, Andrea and Zhong, Yiqiao},
  journal={The Annals of Statistics},
  volume={50},
  number={5},
  pages={2816--2847},
  year={2022},
  publisher={Institute of Mathematical Statistics}
}

@inproceedings{bordelon2020spectrum,
  title={Spectrum dependent learning curves in kernel regression and wide neural networks},
  author={Bordelon, Blake and Canatar, Abdulkadir and Pehlevan, Cengiz},
  booktitle={International Conference on Machine Learning},
  pages={1024--1034},
  year={2020},
  organization={PMLR}
}
\end{document}